\documentclass[10pt]{article}
\usepackage[margin=1in]{geometry}
\usepackage[T1]{fontenc}
\usepackage[utf8]{inputenc}
\usepackage{amsmath,amssymb,amsfonts,amsthm,mathtools}
\usepackage{bm,bbm,mathrsfs}

\usepackage{microtype}
\usepackage{xcolor}
\usepackage{graphicx}
\usepackage{booktabs,tabularx,array,makecell,multirow}
\usepackage{enumitem}
\usepackage{algorithm,algorithmic}
\usepackage{caption,subcaption}
\usepackage{wrapfig}
\usepackage{adjustbox}
\usepackage{nicefrac}
\usepackage{siunitx}
\usepackage{tikz}
\usetikzlibrary{positioning,arrows.meta,calc,matrix}
\usepackage{pgfplots}
\usepackage{pgfplotstable}
\pgfplotsset{compat=1.18}

\usepackage{hyperref}
\hypersetup{colorlinks=true,linkcolor=blue!50!black,citecolor=blue!50!black,urlcolor=blue!50!black}
\usepackage{aliascnt}
\usepackage[nameinlink,capitalise,noabbrev]{cleveref}
\usepackage{titletoc}

\setlist[itemize]{topsep=3pt,itemsep=2pt,parsep=0pt,leftmargin=1.5em}
\setlist[enumerate]{topsep=3pt,itemsep=2pt,parsep=0pt,leftmargin=1.7em}

\newcolumntype{P}[1]{>{\raggedright\arraybackslash}p{#1}}
\newcolumntype{Y}{>{\raggedright\arraybackslash}X}

\makeatletter
\@ifundefined{theorem}{%
  \theoremstyle{plain}
  \newtheorem{theorem}{Theorem}[section]
}{}

\@ifundefined{proposition}{%
  \newaliascnt{proposition}{theorem}
  \newtheorem{proposition}[proposition]{Proposition}
  \aliascntresetthe{proposition}
}{}
\@ifundefined{lemma}{%
  \newaliascnt{lemma}{theorem}
  \newtheorem{lemma}[lemma]{Lemma}
  \aliascntresetthe{lemma}
}{}
\@ifundefined{corollary}{%
  \newaliascnt{corollary}{theorem}
  \newtheorem{corollary}[corollary]{Corollary}
  \aliascntresetthe{corollary}
}{}
\@ifundefined{assumption}{%
  \theoremstyle{definition}
  \newaliascnt{assumption}{theorem}
  \newtheorem{assumption}[assumption]{Assumption}
  \aliascntresetthe{assumption}
}{}
\@ifundefined{definition}{%
  \theoremstyle{definition}
  \newaliascnt{definition}{theorem}
  
  \aliascntresetthe{definition}
}{}
\@ifundefined{example}{%
  \theoremstyle{definition}
  \newaliascnt{example}{theorem}
  
  \aliascntresetthe{example}
}{}
\@ifundefined{remark}{%
  \theoremstyle{remark}
  \newaliascnt{remark}{theorem}
  
  \aliascntresetthe{remark}
}{}
\@ifundefined{condition}{%
  \theoremstyle{definition}
  \newaliascnt{condition}{theorem}
  
  \aliascntresetthe{condition}
}{}
\makeatother

\crefname{theorem}{Theorem}{Theorems}
\Crefname{theorem}{Theorem}{Theorems}
\crefname{proposition}{Proposition}{Propositions}
\Crefname{proposition}{Proposition}{Propositions}
\crefname{lemma}{Lemma}{Lemmas}
\Crefname{lemma}{Lemma}{Lemmas}
\crefname{corollary}{Corollary}{Corollaries}
\Crefname{corollary}{Corollary}{Corollaries}
\crefname{assumption}{Assumption}{Assumptions}
\Crefname{assumption}{Assumption}{Assumptions}
\crefname{definition}{Definition}{Definitions}
\Crefname{definition}{Definition}{Definitions}
\crefname{example}{Example}{Examples}
\Crefname{example}{Example}{Examples}
\crefname{remark}{Remark}{Remarks}
\Crefname{remark}{Remark}{Remarks}
\crefname{condition}{Condition}{Conditions}
\Crefname{condition}{Condition}{Conditions}

\providecommand{\R}{\mathbb{R}}

\providecommand{\E}{\mathbb{E}}

\providecommand{\A}{\mathcal{A}}

\providecommand{\X}{\mathcal{X}}

\makeatletter
\@ifundefined{ceil}{}{}
\@ifundefined{floor}{}{}
\@ifundefined{paren}{}{}
\@ifundefined{bracket}{}{}
\@ifundefined{set}{}{}
\@ifundefined{abs}{}{}
\@ifundefined{norm}{}{}
\@ifundefined{ip}{\DeclarePairedDelimiterX{\ip}[2]{\langle}{\rangle}{#1,\,#2}}{}
\@ifundefined{gendivx}{\DeclarePairedDelimiterX{\gendivx}[2]{(}{)}{#1\;\delimsize\|\;#2}}{}
\makeatother

\usepackage[round]{natbib}
\usepackage{dsfont}
\usepackage{longtable}

\title{A Sharp Finite-Iteration Theory for Asynchronous Categorical Distributional Temporal-Difference Learning}
\author{Ege C.~Kaya \qquad Abolfazl Hashemi\\Purdue University}
\date{}

\begin{document}

\maketitle

\begin{abstract}
We study finite-iteration behavior of asynchronous categorical distributional temporal-difference methods, covering scalar categorical TD (CTD) in the Cram\'er geometry and multivariate signed-categorical TD (MTD) in the maximum mean discrepancy (MMD) geometry. We establish high-probability guarantees under i.i.d.\ sampling and under a continuing Markovian trajectory. For CTD, the Cram\'er sample complexity to the true return law has leading term $\tilde O(\rho_{\min}^{-1}(1-\gamma)^{-2}\varepsilon^{-2})$ in the i.i.d.\ and $\tilde O(\mu_{\min}^{-1}(1-\gamma)^{-2}\varepsilon^{-2}+\mu_{\min}^{-1}t_{\mathrm{mix}})$ in the Markovian setting, without using variance reduction or data dropping. For MTD, the MMD sample complexity has leading terms $\tilde O(\rho_{\min}^{-1}(1-\gamma)^{-1-c}\varepsilon^{-2})$ and $\tilde O(\mu_{\min}^{-1}(1-\gamma)^{-1-c}\varepsilon^{-2}+\mu_{\min}^{-1}t_{\mathrm{mix}})$, where $c\in(0,2)$ is the homogeneity exponent of the MMD kernel, and they reduce to the CTD rates when $c=1$. These are, to our knowledge, the first finite-iteration rates for MTD. For undiscounted fixed-horizon policy evaluation, the same analysis applies to fixed-horizon versions of CTD and MTD under i.i.d.\ and episodic sampling, and the rates match the discounted ones with the effective horizon $(1-\gamma)^{-1}$ replaced by the horizon $H$ in the leading term. Matching minimax lower bounds show that the leading terms, and the implied $1$-Wasserstein rates, are optimal under i.i.d., Markovian, and episodic sampling. Together, these results provide a unified non-asymptotic analysis of asynchronous categorical distributional TD across scalar, multivariate, discounted, and fixed-horizon settings, with rates that are minimax optimal in their leading terms.
\end{abstract}

\section{Introduction}
Categorical representations are one of the two central approximation families in distributional reinforcement learning (RL) \citep{bellemare2017distributional,rowland2018analysis,dabney2018distributional,dabney2018implicit,yang2019fully,lyle2019comparative,rowland2019statistics,bellemare2023distributional}. In the scalar case they underlie the categorical methods initiated by C51 \citep{bellemare2017distributional}, while in the multivariate case they support signed-measure constructions for vector-valued returns \citep{wiltzer2024foundations}. These methods replace infinite-dimensional return distributions with finite-dimensional representations while preserving the geometry that governs projected distributional policy evaluation \citep{howard1972risk,jaquette1973markov,jaquette1976utility,sobel1982variance,morimura2010nonparametric}. For these categorical methods, the asymptotic picture is by now well understood \citep{gurvits1994incremental, dayan1994td, tsitsiklis1994asynchronous, jaakkola1993convergence, bertsekas1996neuro, littman1996generalized, rowland2018analysis,bellemare2023distributional}. The scalar categorical temporal-difference (CTD) method admits a projected Bellman contraction in the Cram\'er metric together with asymptotic guarantees \citep{rowland2018analysis,bellemare2023distributional}. The multivariate signed-categorical temporal-difference (MTD) method admits the analogous maximum mean discrepancy (MMD)-based contraction and theory \citep{wiltzer2024foundations}. What remains less well understood is the sharp finite-iteration behavior of these TD methods. Standard TD learning is usually not a synchronous procedure \citep{sutton1988learning,sutton1998reinforcement,tsitsiklis1996analysis,borkar1998asynchronous}. In practice, each update modifies the sampled state only, and states are typically part of a Markovian trajectory in online RL. Finite-iteration bounds for CTD and MTD are therefore most informative when they cover both asynchronous updates and Markovian sampling.

Our analysis has three parts. In the discounted setting, the CTD and MTD iterates stay in an invariant categorical domain, and we unroll the asynchronous stochastic approximation (SA) recursion around its averaged operator. A Hilbert-space Freedman inequality controls the noise, and the total variance controls its accumulated variance. Under Markovian sampling, we use the reverse-time analysis of \citet{peng2026sharp}, which reads the trajectory backward and controls the noise with a local Poisson equation of the reverse chain. In the fixed-horizon setting, we use the same unrolling, but the contraction comes from the horizon structure, since an error at horizon $h$ is removed once it has been passed down $h$ layers. Under episodic sampling, the resets give an explicit Poisson solution. Finally, we use two-point constructions to show that no estimator improves the leading terms.

\noindent\textbf{Related work.} Several recent results address neighboring finite-iteration questions \citep{qu2020finite,bhandari2018finite,srikant2019finite,patil2023finite,dalal2018finite,peng2024statistical,boeck2022speedy, zhang2023estimation, rowland2024near, peng2025finite, kaya2026quotientcategorical}. With a generative model, \citet{peng2024statistical} give high-probability rates for synchronous distributional TD and CTD, and \citet{rowland2024near} obtain near-minimax rates with a direct estimator. The analysis of \citet{peng2025statisticalefficiencydistributionaltemporal} covers one-state CTD with independent samples and Markovian variants that use data dropping or variance reduction. Along one Markov trajectory, \citet{peng2026online} give asymptotic inference for averaged nonparametric distributional TD, and \citet{peng2026sharp} obtains sharp finite-iteration rates for scalar TD, with a matching lower bound. For episodic sampling, we use a Poisson-equation decomposition as in \citet{haque2024stochastic}, and our fixed-horizon methods are the categorical counterparts of the fixed-horizon TD of \citet{de2020fixed}. A finite-iteration analysis of standard CTD along a Markovian trajectory without data dropping or variance reduction, and corresponding guarantees for signed-categorical MTD and undiscounted fixed-horizon learning remain open. Appendix~\ref{app:related} discusses the related work in more detail.

\noindent\textbf{Contributions.}
\begin{enumerate}[leftmargin=*,labelsep=0.2em, align=left, itemsep=2pt]
    \item We prove high-probability finite-iteration bounds for the discounted asynchronous CTD and MTD recursions. For CTD, this gives the i.i.d.\ rate $\tilde O(\rho_{\min}^{-1}(1-\gamma)^{-2}\varepsilon^{-2})$, which is tight in $(1-\gamma)$ and matches the minimax $1$-Wasserstein rate.
    \item We extend the reverse-time analysis of scalar TD by \citet{peng2026sharp} to the operator-valued CTD and MTD recursions, using an energy bound through the discounted total variance, so that the Markovian rate has the same dependence on $(1-\gamma)$ as the i.i.d.\ rate. Without data dropping or variance reduction, CTD attains the minimax $1$-Wasserstein rate $\tilde O(\mu_{\min}^{-1}(1-\gamma)^{-3}\varepsilon^{-2}+\mu_{\min}^{-1}t_{\mathrm{mix}})$ of scalar TD.
    \item For fixed-horizon policy evaluation, we establish high-probability guarantees under i.i.d.\ and episodic sampling, with coverage entering as the inverse minimum occupancy and rates matching the discounted ones with $(1-\gamma)^{-1}$ replaced by $H$.
    \item We prove minimax lower bounds under i.i.d., Markovian, and episodic sampling, which show that the leading terms of all our rates, and the additive hitting time of the Markovian rates, are optimal up to logarithmic and kernel-dependent factors.
\end{enumerate}

\section{Discounted categorical policy evaluation}\label{sec:discounted}
We consider a discounted Markov decision process \citep{sutton1998reinforcement} $(\X,\A,P,R,\gamma)$ with finite state and action spaces $\X$, $\A$, deterministic reward function $R(s,a)$, transition kernel $P(\cdot\mid s,a)$, and discount factor $\gamma\in(0,1)$. We assume that $R:\X \times \A \to [0, 1]$ in the scalar, and $R:\X \times \A \to[0, 1]^q, q\ge2$ in the multivariate case. A policy $\pi$ is fixed throughout. At iteration $k$, let $S_k$ denote the queried state, $A_k\sim\pi(\cdot\mid S_k)$, write $R_k:=R(S_k,A_k)$, and let $S'_k\sim P(\cdot\mid S_k,A_k)$ denote its sampled successor. In the i.i.d.\ regime, the tuples $(S_k,A_k,R_k,S'_k)_{k\ge0}$ are i.i.d.\ with query-state marginal $S_k\sim\rho$, and $S'_k$ is independent of the next queried state $S_{k+1}$. In the Markovian regime, $S'_k=S_{k+1}$, so $(S_k)_{k\ge0}$ is the on-policy state chain with transition matrix $Q=P^\pi$, where $P^\pi(s,s'):=\sum_{a}\pi(a\mid s)P(s'\mid s,a)$. We assume this chain is irreducible, with unique stationary distribution $\mu_{\X}$. Periodicity is allowed. The i.i.d.\ regime models a replay buffer or simulator, and the Markovian regime represents an on-policy trajectory.

Throughout, the subscript $\star\in\{\mathrm C,\mathrm M\}$ refers to CTD and MTD, respectively. For each method, we work with a block-supremum contraction metric $\ell_{\star,\infty}$, which admits a statewise isometric embedding $I_\star$ to a product space with contraction norm $\lVert U\rVert_{2,\infty}:=\max_s\lVert U(s)\rVert_2$, and a statewise projection $\Pi_\star^\Theta$ onto the space of $\Theta$-supported, state-indexed representations $\mathcal F^\X_{\star,\Theta}$. We can then compose the distributional Bellman operator $T^\pi$, whose fixed point is the true return law $\eta^\pi$, to obtain the embedded, projected operator
\begin{equation}
\mathcal O_\star := I_\star \circ \Pi_\star^\Theta T^\pi \circ I_\star^{-1}.
\end{equation}
This operator admits a one-step sampled Bellman target $\widehat T_\star(U_k;s,(R_k, S'_k))$ computed from a current estimate $U_k$ and a random sample $(R_k, S'_k)$, satisfying
\begin{equation}
\E \left[ \widehat T_\star(U_k; S_k, (R_k, S'_k)) \mid U_k, S_k =s\right] = (\mathcal O_\star U_k)(s),
\end{equation}
and a recursion that takes the form
\begin{equation}\label{eq:recursion}
U_{k+1} = U_k + \alpha_k P_{S_k} (\widehat T_\star(U_k; S_k, (R_k, S'_k)) - U_k(S_k)),
\end{equation}
where $P_s$ is the coordinate projector onto block $s$.

Our discounted analysis unrolls each categorical update around the averaged asynchronous operator, whose powers are deterministic and contract in $\lVert \cdot \rVert_{2,\infty}$. Squared norms of the unrolled noise are propagated through a nonnegative matrix, which keeps the coverage dependence at the inverse minimum state mass, and the resulting martingales are controlled by a Freedman inequality in Hilbert space \citep{pinelis1994optimum}, as in \citet{peng2025statisticalefficiencydistributionaltemporal}. For CTD, the accumulated noise is measured by a discounted total variance, which we bound through a martingale of conditional distribution functions. Under Markovian sampling, we instead use the reverse-time analysis of \citet{peng2026sharp}, where the final error is written as a sum of noise terms multiplied by random propagation maps, and reading the trajectory backward from the final update makes these weights predictable. The common ingredients are collected in Appendix~\ref{app:common}, the method-specific verifications in Appendices~\ref{app:ctd} and~\ref{app:mtd}, and the discounted proofs in Appendix~\ref{app:sharp}.

\subsection{Discounted CTD}\label{sec:ctd-new2}
For each state $s\in\X$, fix an ordered support $\Theta(s)=\{\theta_1(s)<\cdots<\theta_d(s)\}\subset\mathbb R$. $\mathcal F^{\X}_{\mathrm C,\Theta}$ is the class of state-indexed categorical laws supported on these statewise grids. The contraction metric $\ell_{\mathrm C, \infty}$ is the supremum Cram\'er metric, the embedding $I_{\mathrm C}$ is the standard cumulative-mass isometry $I_{\mathrm C, s}$ \citep{rowland2018analysis, bellemare2023distributional} applied to all states, and the statewise projection $\Pi_{\mathrm C}^{\Theta}$ is the usual linear-interpolation categorical projection $\Pi_{\mathrm C}^{\Theta(s)}$ applied to all states. Mass below $\theta_1(s)$ or above $\theta_d(s)$ goes to $\theta_1(s)$ or $\theta_d(s)$, respectively. For a sampled transition $(S_k,A_k,R_k,S'_k)$, the sampled Bellman target is
\begin{equation}\label{eq:ctd-target}
\widehat T_{\mathrm C}(U_k;S_k,(R_k,S'_k)) := I_{\mathrm C,S_k}\Bigl(\Pi_{\mathrm C}^{\Theta(S_k)}\bigl((f_{R_k,\gamma})_\# I_{\mathrm C,S'_k}^{-1}(U_k(S'_k))\bigr)\Bigr).
\end{equation}
where $f_{r,\gamma}(z)=r+\gamma z$ and $\#$ is the measure-pushforward operator. Appendix~\ref{app:ctd} verifies the contraction, domain, and unbiasedness ingredients.

\subsection{Discounted MTD}\label{sec:mtd-new2}
Let $q \ge 2$. For each state $s\in\X$, fix a support of $d$ distinct points $\Theta(s)=\{\theta_1(s),\cdots,\theta_d(s)\}\subset\mathbb R^q$. $\mathcal F^{\X}_{\mathrm M,\Theta}$ is the class of state-indexed, mass-$1$ signed categorical laws supported on these points. The contraction metric $\ell_{\mathrm M, \infty}$ is the supremum MMD metric with a continuous characteristic kernel $\kappa$ induced by a shift-invariant, $c$-homogeneous semimetric of strong negative type with $c\in(0,2)$ \citep{smola2007hilbert,gretton2012kernel,muandet2017kernel,wiltzer2024foundations}.

The embedding $I_{\mathrm M}$ is the per-state embedding $I_{\mathrm M, s}(\eta(s)):=K_s^{1/2}p(s)$ applied to all states, $K_s:=\bigl(\kappa(\theta_i(s),\theta_j(s))\bigr)_{i,j=1}^d$ is the Gram matrix on the support $\Theta(s)$ \citep{wiltzer2024foundations}. The statewise projection $\Pi_{\mathrm M}^{\Theta}$ is
\begin{equation}
\Pi^{\Theta(s)}_{\mathrm M}\nu := \arg\inf_{p\in\mathbb R_1^d} \mathrm{MMD}_{\kappa}\Bigl(\sum_{i=1}^d p_i\delta_{\theta_i(s)},\nu\Bigr),
\end{equation}
applied to all states, where $\mathbb R_1^d$ is the affine subspace of $\mathbb{R}^d$ with unit total mass. The projection is well-defined for finitely supported mass-$1$ signed measures, as shown by \citet{wiltzer2024foundations}. Given a sampled transition $(S_k,A_k,R_k,S'_k)$, the sampled Bellman target $\widehat T_{\mathrm M}$ is defined by \eqref{eq:ctd-target} with subscripts $\mathrm C$ replaced by $\mathrm M$.

\subsection{Finite-iteration guarantees}\label{sec:sharp-main}
Let $\eta^*_\star$ denote the unique fixed point of the projected operator $\Pi_\star^{\Theta} T^\pi$. From any $\eta_0\in\mathcal F_{\star,\Theta}^{\X}$, let $U_0=I_\star(\eta_0)$, let $U_k$ be generated by \eqref{eq:recursion} with constant step size $\alpha_k\equiv\alpha$, and write $e_{\star,k}:=\ell_{\star,\infty}(I_\star^{-1}(U_k),\eta^*_\star)$. Set $\beta_{\mathrm C}:=\sqrt\gamma$, $\beta_{\mathrm M}:=\gamma^{c/2}$, and $B_{\mathrm C}:=\max_s\sqrt{\theta_d(s)-\theta_1(s)}$, and let $B_{\mathrm M}$ be the largest MMD between $\eta^*_{\mathrm M}(s)$ and a sampled Bellman target at $s$ computed from $\eta^*_{\mathrm M}$ (Proposition~\ref{prop:mtd-perturbation}). For both methods, $B_\star$ bounds the noise of the sampled target at the fixed point. Let $\nu:=\rho$ under i.i.d.\ sampling and $\nu:=\mu_{\X}$ under Markovian sampling, with $\nu_{\min}:=\min_s\nu(s)>0$ ($\rho_{\min}$ or $\mu_{\min}$). Under Markovian sampling the chain enters only through $\mu_{\min}$ and the largest expected hitting time $h_Q$ of the time-reversed chain (Appendix~\ref{app:sharp-rev}): $h_Q\le2t_{\mathrm{mix}}\lceil\log_2(2/\mu_{\min})\rceil/\mu_{\min}$ for aperiodic, and $h_Q<\infty$ for periodic chains, where $t_{\mathrm{mix}}$ is the mixing time. The noise enters through the discounted total variance $\mathrm{TV}_{\star,\max}$ of the sampled target at the fixed point, defined in Appendix~\ref{app:sharp} and bounded below.

\begin{theorem}[Discounted CTD and MTD]\label{thm:sharp}
There is a universal constant $C_F\ge1$ such that the following hold for $\star\in\{\mathrm C,\mathrm M\}$ and $\ell\ge1$.
\begin{enumerate}[leftmargin=*,labelsep=0.3em,itemsep=1pt]
\item Under i.i.d.\ sampling, assume $\alpha\le(1-\beta_\star^2)/(2^{11}C_F\ell)$ and let
\begin{equation}\label{eq:main-N}
N_\star:=\sqrt{6C_F\alpha\,\mathrm{TV}_{\star,\max}\ell}+2C_F\alpha B_\star\ell,\qquad L:=\Bigl\lceil\frac{\log8}{\alpha\rho_{\min}(1-\beta_\star)}\Bigr\rceil.
\end{equation}
For all $K\in\mathbb N$ and $D_0\ge\max\{e_{\star,0},2N_\star\}$, w.p. at least $1-16KL\lvert\X\rvert(\log_2L+3)e^{-\ell}$,
\begin{equation}\label{eq:main-bound}
e_{\star,KL}\le2N_\star+2^{-K}D_0.
\end{equation}
\item Under Markovian sampling, assume $\alpha\le1-\beta_\star$, let $\varepsilon_0,\delta\in(0,1]$, $\mathcal E_\star:=4\,\mathrm{TV}_{\star,\max}+3B_\star^2$, and
\begin{equation}\label{eq:main-N-mk}
\begin{gathered}
N_\star^{\mathrm{mk}}:=\ell\sqrt{2C_F\alpha\mathcal E_\star}+2C_F\alpha B_\star\ell+\frac{2\alpha B_\star}{1-\beta_\star},\\
L_{\mathrm{mk}}:=\Bigl\lceil8\Bigl(\frac{1}{\alpha\mu_{\min}(1-\beta_\star)}+h_Q\Bigr)\Bigr\rceil,\qquad K_{\mathrm{mk}}:=\Bigl\lceil\log_2\frac{2\lvert\X\rvert}{\delta\varepsilon_0\mu_{\min}}\Bigr\rceil.
\end{gathered}
\end{equation}
For every $T\ge K_{\mathrm{mk}}L_{\mathrm{mk}}$ and any law of $S_0$, w.p. at least $1-\delta/2-4\lvert\X\rvert e^{-\ell}/\mu_{\min}$,
\begin{equation}\label{eq:main-bound-mk}
e_{\star,T}\le\varepsilon_0e_{\star,0}+N_\star^{\mathrm{mk}}.
\end{equation}
\end{enumerate}
\end{theorem}
\begin{proof}[Proof sketch]
Under i.i.d.\ sampling the embedded error $W_k:=U_k-I_\star(\eta^*_\star)$ satisfies $W_{k+1}=\bar AW_k+\alpha\zeta_k$, where $\bar A$ is the averaged asynchronous operator and $\zeta_k$ is a martingale difference. Over epochs of length $L$, $\bar A^L$ shrinks the error by a factor $8$, and Freedman's inequality bounds the unrolled noise statewise, with variance controlled by the propagation matrix $I-\alpha D_\nu(I-\beta_\star^2P^\pi)$, with $D_\nu:=\operatorname{diag}(\nu)$. Under Markovian sampling we follow the reverse-time analysis of \citet{peng2026sharp}, reading the trajectory backward from time $T$, and the total variance enters through the energy bound $\mathcal E_\star$. Appendix~\ref{app:sharp} gives the full proof.
\end{proof}

\noindent\textbf{Interpretation.} Both parts bound the error by an initial error that is attenuated geometrically plus a noise floor. Under i.i.d.\ sampling, each epoch of $L$ updates halves the distance of the error bound to $2N_\star$, and $N_\star^2=O(\alpha\,\mathrm{TV}_{\star,\max}\ell)$ up to lower-order terms. Under Markovian sampling, $T$ enters through $\varepsilon_0$: Taking $\varepsilon_0:=2^{-\lfloor T/L_{\mathrm{mk}}\rfloor}\cdot2\lvert\X\rvert/(\delta\mu_{\min})$ gives $K_{\mathrm{mk}}L_{\mathrm{mk}}\le T$, so, with the same probability and whenever $\varepsilon_0\le1$,
\begin{equation}\label{eq:main-mk-T}
e_{\star,T}\le\frac{2\lvert\X\rvert}{\delta\mu_{\min}}2^{-\lfloor T/L_{\mathrm{mk}}\rfloor}e_{\star,0}+N_\star^{\mathrm{mk}}.
\end{equation}
The initialization error is thus forgotten on the time scale $L_{\mathrm{mk}}=O(1/(\alpha\mu_{\min}(1-\beta_\star))+h_Q)$, and $(N_\star^{\mathrm{mk}})^2=O(\alpha\mathcal E_\star\ell^2)$, up to lower-order terms. Neither floor decreases with $T$ because of the constant step size. The corollaries below choose $\alpha$ to control the floor and $T$ to control the remaining initialization error. Since the iterates converge to $\eta^*_\star$ rather than to $\eta^\pi$, they then add the projection bias $\ell_{\star,\infty}(\eta^*_\star,\eta^\pi)$.

For CTD, assume $\theta_1(s)=0$ and $\theta_d(s)=(1-\gamma)^{-1}$ for all $s$, so that $B_{\mathrm C}=(1-\gamma)^{-1/2}$, and let $\Delta_{\max}$ denote the largest gap between consecutive atoms. Lemma~\ref{lem:ctd-tv} shows $\mathrm{TV}_{\mathrm C,\max}\le\bar\tau/(1-\gamma)$ with $\bar\tau:=2+8\Delta_{\max}/(1-\gamma)$.

\begin{corollary}[CTD sample complexity]\label{cor:sharp-ctd}
Under the assumptions above, let $\varepsilon\in(0,1]$ and $\delta\in(0,1)$, and let the grids be uniform with $d-1\ge4\varepsilon^{-2}(1-\gamma)^{-2}$ atoms, so that $\bar\tau\le4$. There is a constant step size such that the error to the true return law satisfies $\ell_{\mathrm C,\infty}(\eta_T,\eta^\pi)\le\varepsilon$ w.p. at least $1-\delta$ and $\mathbb E[\ell_{\mathrm C,\infty}(\eta_T,\eta^\pi)^2]\le2\varepsilon^2$ after $T=T_{\mathrm{iid}}$ or $T=T_{\mathrm{mk}}$ updates, where $\eta_T:=I_{\mathrm C}^{-1}(U_T)$ and $\tilde O$ hides factors logarithmic in $\lvert\X\rvert$, $\delta^{-1}$, $\varepsilon^{-1}$, $\nu_{\min}^{-1}$, and $(1-\gamma)^{-1}$:
\begin{equation}\label{eq:ctd-rates}
T_{\mathrm{iid}}=\tilde O\bigl(\rho_{\min}^{-1}\bar\tau(1-\gamma)^{-2}\varepsilon^{-2}\bigr),\qquad
T_{\mathrm{mk}}=\tilde O\bigl(\mu_{\min}^{-1}\bigl[\bar\tau(1-\gamma)^{-2}\varepsilon^{-2}+(1-\gamma)^{-5/2}\varepsilon^{-1}\bigr]+h_Q\bigr).
\end{equation}
The step size is of order $(1-\gamma)\varepsilon^2/\ell$ under i.i.d.\ sampling and $\min\{(1-\gamma)\varepsilon^2/\ell^2,(1-\gamma)^{3/2}\varepsilon\}$ under Markovian sampling, with $\ell$ logarithmic in the same quantities. Since the supremum $1$-Wasserstein distance satisfies $\bar W_1\le(1-\gamma)^{-1/2}\ell_{\mathrm C,\infty}$ for laws on $[0,(1-\gamma)^{-1}]$, the same holds for $\bar W_1(\eta_T,\eta^\pi)$ after replacing $\varepsilon$ by $(1-\gamma)^{1/2}\varepsilon$, with $d-1\ge4\varepsilon^{-2}(1-\gamma)^{-3}$ atoms.
\end{corollary}
For MTD, let $d_\kappa(x,y):=\mathrm{MMD}_\kappa(\delta_x,\delta_y)^2$, $\bar\kappa:=\sup_{x,y\in[0,1]^q}d_\kappa(x,y)$, and let $h_\Theta$ be the largest $d_\kappa$-distance from a point of $[0,(1-\gamma)^{-1}]^q$, the range of the returns, to its nearest atom. Lemma~\ref{lem:mtd-tv} gives $\mathrm{TV}_{\mathrm M,\max}\le\bar\tau_{\mathrm M}(1-\gamma)^{-c}$ with $\bar\tau_{\mathrm M}:=\bar\kappa+2h_\Theta(1-\gamma)^c/(1-\gamma^c)^2$, and the explicit bound $B_{\mathrm M}\le\bar B_{\mathrm M}:=\sqrt{\bar\kappa}(1-\gamma)^{-c/2}+2\sqrt{h_\Theta/(1-\gamma^c)}$.

\begin{corollary}[MTD sample complexity]\label{cor:sharp-mtd}
Let $\varepsilon\in(0,\bar B_{\mathrm M}]$ and $\delta\in(0,1)$, and let the atoms satisfy $h_\Theta\le\varepsilon^2(1-\gamma^c)/4$. There is a constant step size such that the error to the true return law satisfies $\ell_{\mathrm M,\infty}(\eta_T,\eta^\pi)\le\varepsilon$ w.p. at least $1-\delta$ and $\mathbb E[\ell_{\mathrm M,\infty}(\eta_T,\eta^\pi)^2]\le2\varepsilon^2$ after $T=T_{\mathrm{iid}}$ or $T=T_{\mathrm{mk}}$ updates, where $\eta_T:=I_{\mathrm M}^{-1}(U_T)$ and $\tilde O$ also hides constants depending on $c$ and logarithms of $\bar\kappa$ and $e_{\mathrm M,0}/\varepsilon$:
\begin{equation}\label{eq:mtd-rates}
\begin{gathered}
T_{\mathrm{iid}}=\tilde O\bigl(\rho_{\min}^{-1}\bigl[\bar\tau_{\mathrm M}(1-\gamma)^{-1-c}\varepsilon^{-2}+(1-\gamma)^{-2}\bigr]\bigr),\\
T_{\mathrm{mk}}=\tilde O\bigl(\mu_{\min}^{-1}\bigl[\bar\tau_{\mathrm M}(1-\gamma)^{-1-c}\varepsilon^{-2}+\bar B_{\mathrm M}(1-\gamma)^{-2}\varepsilon^{-1}\bigr]+h_Q\bigr).
\end{gathered}
\end{equation}
The step size is of order $\min\{(1-\gamma)/\ell,(1-\gamma)^c\varepsilon^2/(\bar\tau_{\mathrm M}\ell)\}$ under i.i.d.\ sampling and $\min\{(1-\gamma)^c\varepsilon^2/(\bar\tau_{\mathrm M}\ell^2),(1-\gamma)\varepsilon/\bar B_{\mathrm M}\}$ under Markovian sampling.
\end{corollary}
The condition on $h_\Theta$ gives $\bar\tau_{\mathrm M}\le\bar\kappa+\varepsilon^2(1-\gamma)^c/(2(1-\gamma^c))$, which is at most $2\bar\kappa$ when $c\ge1$ and $\varepsilon^2\le2\bar\kappa$. Since $\bar\kappa$ depends only on the kernel and $q$, for $c=1$ and $\varepsilon^2\le\bar\kappa$ these are the CTD rates, up to factors depending on $\bar\kappa$. A coordinate grid of spacing $\Delta$ gives $h_\Theta\le\bar\kappa\Delta^c$, and the number of atoms the condition requires grows exponentially in $q$, which affects the cost of each update but not the sample complexity.

Table~\ref{tab:discounted-comparison} compares the closest tabular results, omitting offline, function-approximation, and asymptotic analyses. For small $\varepsilon$, the i.i.d.\ Cram\'er rate is $\tilde O(\rho_{\min}^{-1}(1-\gamma)^{-2}\varepsilon^{-2})$. In $\bar W_1$, it becomes $\tilde O(\rho_{\min}^{-1}(1-\gamma)^{-3}\varepsilon^{-2})$, the rate of \citet{peng2025statisticalefficiencydistributionaltemporal} for one-state CTD and the order of their minimax lower bound. Section~\ref{sec:lower} shows that no estimator improves the leading terms of these rates, under either sampling model, beyond constant and logarithmic factors. Our Cram\'er rate improves the $(1-\gamma)^{-5/2}$ stated by \citet{peng2025statisticalefficiencydistributionaltemporal}. The improvement comes from Lemma~\ref{lem:ctd-tv}: summing the worst-case one-step Cram\'er noise variance $B_{\mathrm C}^2=(1-\gamma)^{-1}$ over the effective horizon gives $O((1-\gamma)^{-2})$, but the variances cannot be that large along a whole trajectory, and their discounted sum is $O((1-\gamma)^{-1})$, as in the total-variance argument of \citet{gheshlaghi2013minimax}. Under Markovian sampling, no term that depends on $\varepsilon$ carries a mixing factor. In $\bar W_1$, with the grid condition above and $\varepsilon\le1$, the Markovian rate is $\tilde O(\mu_{\min}^{-1}(1-\gamma)^{-3}\varepsilon^{-2}+h_Q)$. With standard CTD, which needs neither data dropping nor variance reduction, this improves by $t_{\mathrm{mix}}$ on data dropping and removes the factor $(1-\gamma)^{-1}$ from the mixing term of VR-CTD, and it also covers periodic chains. This matches the rate of \citet{peng2026sharp} for scalar TD. Since the error of the mean return is at most $\bar W_1$, the lower bound of that work applies, and the rate is minimax optimal up to logarithms over chains with given $\mu_{\min}$ and mixing-time budget at least $16(1-\gamma)^{-1}$.

As in \citet{rowland2018analysis}, the CTD projection bias is at most $\sqrt{\Delta_{\max}/(1-\gamma)}$ (Lemma~\ref{lem:ctd-tv}), which the grid condition of Corollary~\ref{cor:sharp-ctd} makes at most $\varepsilon/2$. In $\bar W_1$, the condition $d-1\ge4\varepsilon^{-2}(1-\gamma)^{-3}$ is that of \citet{peng2025statisticalefficiencydistributionaltemporal}. The number of atoms affects the sample complexity only through $\bar\tau$, which decreases as the grid is refined, and affects the cost of each update linearly.
\begin{table*}[t]
\centering
\caption{Discounted tabular distributional policy-evaluation guarantees. The last column counts environment transitions, after converting synchronous sweeps and per-state sample counts, and $\tilde O$ hides logarithmic factors. All guarantees hold with high probability. Under uniform i.i.d.\ queries $\rho_{\min}^{-1}=\lvert\X\rvert$, the analogue of the factor $\lvert\X\rvert$ (or $\lvert\X\rvert\lvert\A\rvert$ for state-action sampling) in the generative-model rows. Bounds to the true return law assume fine grids.}
\label{tab:discounted-comparison}
\footnotesize
\renewcommand{\arraystretch}{1.25}
\begin{tabularx}{\linewidth}{@{}>{\raggedright\arraybackslash}p{0.13\linewidth}>{\raggedright\arraybackslash}p{0.17\linewidth}>{\raggedright\arraybackslash}p{0.17\linewidth}>{\raggedright\arraybackslash}X@{}}
\toprule
Work & Method and sampling & Error criterion & Transition complexity \\
\midrule
\citet{boeck2022speedy} & speedy categorical policy evaluation, synchronous generative model & Cram\'er to projected fixed point & $\tilde O(\lvert\X\rvert\lvert\A\rvert\varepsilon^{-2}(1-\gamma)^{-3})$ \\
\citet{zhang2023estimation} & plug-in estimator, generative model & supremum $1$-Wasserstein to true return law & $\tilde O(\lvert\X\rvert\lvert\A\rvert\varepsilon^{-2}(1-\gamma)^{-4})$ \\
\citet{rowland2024near} & DCFP, generative model & supremum $1$-Wasserstein to true return law & $\tilde O(\lvert\X\rvert\varepsilon^{-2}(1-\gamma)^{-3})$ \\
\citet{peng2025statisticalefficiencydistributionaltemporal} & CTD, one-state i.i.d.\ queries & Cram\'er to true return law & $\tilde O(\rho_{\min}^{-1}\varepsilon^{-2}(1-\gamma)^{-5/2})$ \\
\citet{peng2025statisticalefficiencydistributionaltemporal} & CTD with data dropping, aperiodic Markov chain & supremum $1$-Wasserstein to true return law & $\tilde O(t_{\mathrm{mix}}\mu_{\min}^{-1}\varepsilon^{-2}(1-\gamma)^{-3})$ \\
\citet{peng2025statisticalefficiencydistributionaltemporal} & VR-CTD, aperiodic Markov chain & supremum $1$-Wasserstein to true return law & $\tilde O(\mu_{\min}^{-1}\varepsilon^{-2}(1-\gamma)^{-3}+\mu_{\min}^{-1}t_{\mathrm{mix}}(1-\gamma)^{-1})$ \\
This work$^\dagger$ & CTD, one-state i.i.d.\ queries & supremum Cram\'er and $1$-Wasserstein to true return law & \makecell[tl]{Cram\'er: $\tilde O(\rho_{\min}^{-1}\varepsilon^{-2}(1-\gamma)^{-2})$\\[2pt]$\bar W_1$: $\tilde O(\rho_{\min}^{-1}\varepsilon^{-2}(1-\gamma)^{-3})$} \\
This work$^\dagger$ & CTD, aperiodic Markov chain$^\ddagger$ & supremum Cram\'er and $1$-Wasserstein to true return law & \makecell[tl]{Cram\'er: $\medmuskip=3mu\tilde O(\mu_{\min}^{-1}[\varepsilon^{-2}(1-\gamma)^{-2}+\varepsilon^{-1}(1-\gamma)^{-5/2}+t_{\mathrm{mix}}])$\\[2pt]$\bar W_1$: $\tilde O(\mu_{\min}^{-1}[\varepsilon^{-2}(1-\gamma)^{-3}+t_{\mathrm{mix}}])$} \\
This work$^\S$ & MTD, one-state i.i.d.\ queries & supremum MMD to true return law & $\tilde O(\rho_{\min}^{-1}[\varepsilon^{-2}(1-\gamma)^{-1-c}+(1-\gamma)^{-2}])$ \\
This work$^\S$ & MTD, aperiodic Markov chain$^\ddagger$ & supremum MMD to true return law & \makecell[tl]{$\tilde O(\mu_{\min}^{-1}[\varepsilon^{-2}(1-\gamma)^{-1-c}+\varepsilon^{-1}(1-\gamma)^{-2-c/2}$\\$\phantom{\tilde O(\mu_{\min}^{-1}[}+(1-\gamma)^{-2}+t_{\mathrm{mix}}])$} \\
\bottomrule
\multicolumn{4}{@{}p{\linewidth}@{}}{$^\dagger$Corollary~\ref{cor:sharp-ctd}, with uniform grids on $[0,(1-\gamma)^{-1}]$ of $d-1\ge4\varepsilon^{-2}(1-\gamma)^{-2}$ atoms, and $d-1\ge4\varepsilon^{-2}(1-\gamma)^{-3}$ in $\bar W_1$, the condition of \citet{peng2025statisticalefficiencydistributionaltemporal}, and $\varepsilon\le1$. $^\ddagger$Periodic chains are also covered, with $\mu_{\min}^{-1}t_{\mathrm{mix}}$ replaced by the reverse hitting time $h_Q$ (Lemma~\ref{lem:rev-hit}). $^\S$Corollary~\ref{cor:sharp-mtd}, with $h_\Theta\le\varepsilon^2(1-\gamma^c)/4$, omitting factors that depend only on $c$ and the kernel.}
\end{tabularx}
\end{table*}

\section{Undiscounted fixed-horizon categorical policy evaluation}\label{sec:undisc}

We now turn to the undiscounted fixed-horizon setting, where the return is accumulated over a fixed number $h$ of future steps \citep{sutton1998reinforcement,de2020fixed}. Fix a horizon $H \in \mathbb N$ and a stationary policy $\pi(\cdot \mid s)$. For each $h \in \{0,1, \dots,H\}$ and $s \in \X$, let $\eta^h(s)$ denote the $h$-step return-distribution object under $\pi$, with terminal layer $\eta^0(s)=\delta_0$.

We consider two sampling models. In the i.i.d.\ regime the tuples are i.i.d.\ with $S_k\sim\rho$, as in Section~\ref{sec:discounted}. In the episodic regime, the data come in episodes of length $H$, as in finite-horizon RL \citep{dann2015sample,azar2017minimax,jin2018q}. Each episode starts from a reset distribution $\nu_0$, independently of earlier episodes, and then follows $\pi$. Within an episode, $S_{k+1}=S'_k$. At the last step of an episode, the successor $S'_k$ is recorded before the reset, so the next start state $S_{k+1}\sim\nu_0$ is not used as a bootstrap state, and $\eta^h$ is the $h$-step return of the non-terminating process, as in the continuing setting of \citet{de2020fixed}. Since how often a state is updated depends on the within-episode phase, we define the phase distributions $\rho_t(s) := \Pr(S_t = s)$ for $0 \le t \le H-1$ and $s \in \X$, and the average episode occupancy
\begin{equation}
\bar\rho_H(s):=\frac{1}{H}\sum_{t=0}^{H-1}\rho_t(s),\qquad \rho_{H,\min}:=\min_{s\in\X}\bar\rho_H(s)>0.
\end{equation}
The phase-state pair $(k\bmod H,S_k)$ is a Markov chain with stationary distribution $\rho_t(s)/H$, which regenerates at every reset.

For $h \ge 1$, define the undiscounted fixed-horizon distributional Bellman operator by
\begin{equation}
(T_H^\pi \eta)^h(s) := \sum_{a\in\A}\pi(a\mid s)\sum_{s'\in\X}P(s'\mid s,a)(f_{R(s,a),1})_\#\eta^{h-1}(s').
\end{equation}
Thus, the $h$th horizon bootstraps from the $(h-1)$st horizon, as in the formulation of \citet{de2020fixed}.

For both CTD and MTD, we use exactly the same local categorical projections and local statewise isometric embeddings as in Section~\ref{sec:discounted}, now applied separately at each horizon and then stacked over $h=1,\dots,H$. The embedded iterate at $s$ is $U(s):=(U^1(s),\ldots,U^H(s))$, and $U:=(U^h(s))_{(h,s)\in\mathcal S_H}$ with $\mathcal S_H:=\{1,\ldots,H\}\times\X$.

The discount factor that makes $\mathcal O_\star$ contract in the discounted setting is absent here. Instead, layer $h$ of $T_H^\pi\eta$ depends only on layer $h-1$ of $\eta$, so the error at horizon $h$ is driven only by the error at horizon $h-1$, and errors are removed once they have been passed down to the fixed terminal layer. Let $\Pi_H^\Theta$ and $I_H$ denote the horizonwise projection and horizonwise embedding, and define $\mathcal O_H:=I_H\circ \Pi_H^\Theta T_H^\pi \circ I_H^{-1}$. The projected fixed point $\eta^*_H$ is determined by backward induction, $\eta^{*,0}_H=\delta_0$ and $\eta^{*,h}_H=(\Pi_H^\Theta T_H^\pi\eta^*_H)^h$ for $h=1,\dots,H$. We measure errors in the metric
\begin{equation}
\ell_{H,\infty}(\eta,\eta') := \max_{(h,s)\in\mathcal S_H}\ell\bigl(\eta^h(s),{\eta'}^h(s)\bigr),
\end{equation}
into which $I_H$ is an isometric embedding for the norm $\lVert U\rVert_{2,\infty}:=\max_{(h,s)\in\mathcal S_H}\lVert U^h(s)\rVert_2$.

For each sampled transition $(S_k,A_k,R_k,S'_k)$, where $S'_k\sim P(\cdot\mid S_k,A_k)$ is the recorded successor, let $\widehat T_H(U_k;S_k,(R_k,S'_k)):=(\widehat T_H^h(U_k;S_k,(R_k,S'_k)))_{h=1}^H$ denote the stacked one-step sampled Bellman target, where the horizonwise components are defined in the CTD and MTD subsections below. For every $s \in \X$,
\begin{equation}
\mathbb{E}\bigl[\widehat T_H\bigl(U_k; S_k, (R_k,S'_k)\bigr)\mid U_k,\ S_k = s \bigr] = (\mathcal O_H U_k)(s).
\end{equation}
Thus, one transition from $s$ updates every horizon at $s$. All horizons share the observed reward, and horizon $h$ bootstraps from the current $(h-1)$-step estimate at the next state. If $P_s^H$ denotes the coordinate projector onto the full horizon stack at state $s$, the online fixed-horizon recursion takes the form
\begin{equation}\label{eq:undisc-recur}
U_{k+1}=U_k+\alpha_kP_{S_k}^H\Bigl(\widehat T_H(U_k;S_k,(R_k,S'_k))-U_k(S_k)\Bigr).
\end{equation}

\subsection{Undiscounted fixed-horizon CTD}

For CTD, each state-horizon pair $(h,s)$ carries an ordered, possibly horizon-dependent scalar support $\Theta_h(s)=\{\theta_{h,1}(s)<\cdots<\theta_{h,d}(s)\}\subset\mathbb R$. The statewise Cram\'er metric $\ell_{H, \mathrm C, \infty}$, cumulative-mass embedding $I_{H, \mathrm C, h, s}$, and linear-interpolation projection $\Pi^{\Theta_h(s)}_{H, \mathrm C}$ are defined exactly as in Section~\ref{sec:ctd-new2}, only now applied separately at each horizon and stacked over $h=1,\dots,H$ \citep{rowland2018analysis,bellemare2023distributional}. Given a sampled transition $(S_k,A_k,R_k,S'_k)$, the sampled Bellman target is computed at every horizon $h=1,\dots,H$ by
\begin{equation}\label{eq:fh-ctd-target}
\widehat T_{H,\mathrm C}^h(U_k;S_k,(R_k,S'_k)) := I_{H,\mathrm C,h,S_k}\Bigl(\Pi_{H,\mathrm C}^{\Theta_h(S_k)}\bigl((f_{R_k,1})_\# I_{H,\mathrm C,h-1,S'_k}^{-1}(U_k^{h-1}(S'_k))\bigr)\Bigr),
\end{equation}
with the terminal convention $I_{H,\mathrm C,0,s}^{-1}(U_k^0(s)):=\delta_0$ for every $s$ and $k$. Stacking these horizonwise targets gives $\widehat T_{H,\mathrm C}(U_k;S_k,(R_k,S'_k)):=(\widehat T_{H,\mathrm C}^h(U_k;S_k,(R_k,S'_k)))_{h=1}^H$.

\begin{figure*}[t]
\centering
\includegraphics[width=0.96\linewidth]{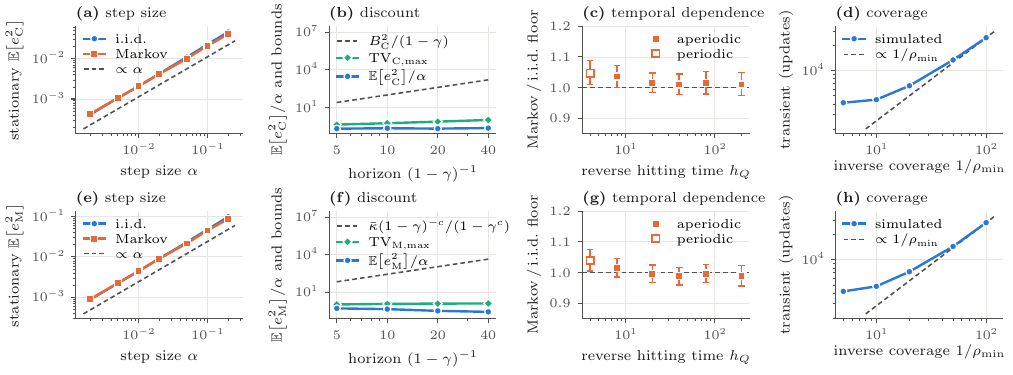}
\caption{Constant-step CTD (top) and MTD (bottom), stationary squared supremum Cram\'er and MMD errors with $95\%$ bootstrap intervals. (a, e) Noise floor against the step size. (b, f) Floor divided by $\alpha$ against the horizon, with the exact total variance and the worst-case bound. (c, g) Markovian over i.i.d.\ floor against the reverse hitting time (hollow: periodic chain). (d, h) Transient length against inverse coverage.}
\label{fig:sweeps}
\end{figure*}

\subsection{Undiscounted fixed-horizon MTD}
For MTD, each state-horizon pair $(h,s)$ carries a possibly horizon-dependent support of $d$ distinct points $\Theta_h(s)=\{\theta_{h,1}(s),\dots,\theta_{h,d}(s)\}\subset\mathbb R^q$. The statewise MMD metric $\ell_{H, \mathrm M, \infty}$ with a characteristic kernel $\kappa$ induced by a shift-invariant, $c$-homogeneous semimetric of strong negative type, the signed-categorical projection $\Pi^{\Theta_h(s)}_{H, \mathrm M}$, and the Gram-matrix embedding $I_{H, \mathrm M, h, s}$ are defined exactly as in Section~\ref{sec:mtd-new2}, now applied separately at each horizon and stacked over $h=1,\dots,H$. Given a sampled transition $(S_k,A_k,R_k,S'_k)$, the horizonwise targets $\widehat T_{H,\mathrm M}^h$ are defined by \eqref{eq:fh-ctd-target} with subscripts $\mathrm C$ replaced by $\mathrm M$, again with the terminal convention $I_{H,\mathrm M,0,s}^{-1}(U_k^0(s)):=\delta_0$ for every $s$ and $k$, and stacking them over $h=1,\dots,H$ gives $\widehat T_{H,\mathrm M}$.

\subsection{Finite-iteration guarantees}
Let $\eta^*_{H,\star}$ denote the projected fixed point of $\Pi_{H,\star}^\Theta T_H^\pi$, determined by backward induction as above. From any $\eta_0\in\mathcal F_{H,\star,\Theta}^{\X}$, let $U_0=I_{H,\star}(\eta_0)$, let $U_k$ be generated by \eqref{eq:undisc-recur} with $\widehat T_H=\widehat T_{H,\star}$ and constant step size $\alpha_k\equiv\alpha$, and write $e^H_{\star,k}:=\ell_{H,\star,\infty}(I_{H,\star}^{-1}(U_k),\eta^*_{H,\star})$. Let $B_{H,\mathrm C}:=\max_{(h,s)\in\mathcal S_H}\sqrt{\theta_{h,d}(s)-\theta_{h,1}(s)}$, and $B_{H,\mathrm M}$ be the largest MMD between $\eta^{*,h}_{H,\mathrm M}(s)$ and a sampled Bellman target at $(h,s)$ computed from $\eta^*_{H,\mathrm M}$ (Lemma~\ref{lem:fh-mtd-tv}). For both methods, $B_{H,\star}$ bounds the noise of the sampled target at the fixed point. Let $\nu:=\rho$ under i.i.d.\ sampling and $\nu:=\bar\rho_H$ under episodic sampling, with $\nu_{\min}:=\min_s\nu(s)>0$ ($\rho_{\min}$ or $\rho_{H,\min}$). The noise enters through the fixed-horizon total variance $\mathrm{TV}^H_{\star,\max}$ of the sampled target at the fixed point, the largest sum of the one-step noise variances along $h$ steps of the chain, defined in Appendix~\ref{app:fh-sharp} and bounded below.

\begin{theorem}[Fixed-horizon CTD and MTD]\label{thm:undisc}
With $C_F$ from Theorem~\ref{thm:sharp}, the following hold for $\star\in\{\mathrm C,\mathrm M\}$ and $\ell\ge1$.
\begin{enumerate}[leftmargin=*,labelsep=0.3em,itemsep=1pt]
\item Under i.i.d.\ sampling, let $\alpha\le1/(2^{11}C_FH\ell)$ and
\begin{equation}\label{eq:fh-N}
N^H_\star:=\sqrt{6C_F\alpha\,\mathrm{TV}^H_{\star,\max}\ell}+2C_F\alpha B_{H,\star}\ell,\qquad L_H:=\Bigl\lceil\frac{2(H+8)}{\alpha\nu_{\min}}\Bigr\rceil.
\end{equation}
For all $K\in\mathbb N$ and $D_0\ge\max\{e^H_{\star,0},2N^H_\star\}$, w.p. at least $1-16KL_HH\lvert\X\rvert(\log_2L_H+3)e^{-\ell}$,
\begin{equation}\label{eq:fh-bound}
e^H_{\star,KL_H}\le2N^H_\star+2^{-K}D_0.
\end{equation}
\item Under episodic sampling, assume in addition $\alpha\le2^{-21}/(C_FH^2\ell)$. Then \eqref{eq:fh-bound} holds with the same probability, with $N^H_\star$ replaced by $N^H_\star+20\alpha H^2B_{H,\star}$.
\end{enumerate}
\end{theorem}
\begin{proof}[Proof sketch]
Under i.i.d.\ sampling, the embedded error $W_k:=U_k-I_{H,\star}(\eta^*_{H,\star})$ satisfies $W_{k+1}=\bar AW_k+\alpha\zeta_k$ as in Theorem~\ref{thm:sharp}. The linear part of the sampled target maps horizon $h-1$ at the successor to horizon $h$ with norm at most one, so $\bar A$ is $I-\alpha D_\nu$ plus a nilpotent part. An error at horizon $h$ disappears after being passed down $h$ times, and each passage happens at rate at least $\alpha\nu_{\min}$, so after $L_H$ steps a binomial tail shrinks the error by a factor of $8$. Squared norms propagate through $I-\alpha D_\nu(I-N_H)$, where $N_H$ shifts horizon $h-1$ at the successor to horizon $h$, and the accumulated variance is $\alpha(I-N_H)^{-1}\sigma^2_\star$. Under episodic sampling the Poisson equation of the phase-state chain has an explicit solution of size at most $H$. The rest follows the i.i.d.\ proof of Theorem~\ref{thm:sharp}, with the forward Poisson argument of Theorem~\ref{thm:sharp-app} under episodic sampling (Appendix~\ref{app:fh-sharp}).
\end{proof}
For CTD, assume $\theta_{h,1}(s)=0$ and $\theta_{h,d}(s)=h$ for all $(h,s)$, so that $B_{H,\mathrm C}=\sqrt H$, and let $\Delta_{\max}$ denote the largest gap between consecutive atoms. Lemma~\ref{lem:fh-ctd-tv} shows $\mathrm{TV}^H_{\mathrm C,\max}\le\bar\tau_HH$ with $\bar\tau_H:=2+2H\Delta_{\max}$.

\begin{corollary}[Fixed-horizon CTD sample complexity]\label{cor:undisc}
Under the assumptions above, let $\varepsilon\in(0,1]$ and $\delta\in(0,1)$, and let the grids be uniform with $d-1\ge4H^2\varepsilon^{-2}$ atoms, so that $\bar\tau_H\le4$. There is a constant step size such that the error to the true $h$-step return laws $\eta^\pi_H:=(\eta^{\pi,h})_{h\le H}$ satisfies $\ell_{H,\mathrm C,\infty}(\eta_T,\eta^\pi_H)\le\varepsilon$ w.p. at least $1-\delta$ and $\mathbb E[\ell_{H,\mathrm C,\infty}(\eta_T,\eta^\pi_H)^2]\le2\varepsilon^2$ after $T=T_{\mathrm{iid}}$ or $T=T_{\mathrm{ep}}$ updates, where $\eta_T:=I_{H,\mathrm C}^{-1}(U_T)$ and $\tilde O$ hides factors logarithmic in $\lvert\X\rvert$, $\delta^{-1}$, $\varepsilon^{-1}$, $\nu_{\min}^{-1}$, and $H$:
\begin{equation}\label{eq:fh-rates}
T_{\mathrm{iid}}=\tilde O\bigl(\rho_{\min}^{-1}\bar\tau_HH^2\varepsilon^{-2}\bigr),\qquad
T_{\mathrm{ep}}=\tilde O\bigl(\rho_{H,\min}^{-1}\bigl[\bar\tau_HH^2\varepsilon^{-2}+H^{7/2}\varepsilon^{-1}+H^3\bigr]\bigr).
\end{equation}
The step size is of order $\alpha_{\mathrm{iid}}:=\varepsilon^2/(H\ell)$ under i.i.d.\ sampling and $\min\{\alpha_{\mathrm{iid}},1/(H^2\ell),\varepsilon/H^{5/2}\}$ under episodic sampling. Since the supremum $1$-Wasserstein distance satisfies $\bar W_1\le H^{1/2}\ell_{H,\mathrm C,\infty}$ for laws on $[0,H]$, the same holds for $\bar W_1(\eta_T,\eta^\pi_H)$ after replacing $\varepsilon$ by $H^{-1/2}\varepsilon$, with $d-1\ge4H^3\varepsilon^{-2}$ atoms.
\end{corollary}
For MTD, let $h_{\Theta,H}$ be the largest $d_\kappa$-distance from a point of $[0,h]^q$ to its nearest atom in $\Theta_h(s)$, over all $(h,s)$. Lemma~\ref{lem:fh-mtd-tv} gives $\mathrm{TV}^H_{\mathrm M,\max}\le\bar\tau_{H,\mathrm M}H^c$ with $\bar\tau_{H,\mathrm M}:=\bar\kappa+H^{2-c}h_{\Theta,H}$, and the explicit bound $B_{H,\mathrm M}\le\bar B_{H,\mathrm M}:=\sqrt{\bar\kappa}H^{c/2}+2\sqrt{Hh_{\Theta,H}}$.

\begin{corollary}[Fixed-horizon MTD sample complexity]\label{cor:undisc-mtd}
Let $\varepsilon\in(0,\bar B_{H,\mathrm M}]$ and $\delta\in(0,1)$, and let the atoms satisfy $h_{\Theta,H}\le\varepsilon^2/(4H)$. There is a constant step size such that the error to the true $h$-step return laws satisfies $\ell_{H,\mathrm M,\infty}(\eta_T,\eta^\pi_H)\le\varepsilon$ w.p. at least $1-\delta$ and $\mathbb E[\ell_{H,\mathrm M,\infty}(\eta_T,\eta^\pi_H)^2]\le2\varepsilon^2$ after $T=T_{\mathrm{iid}}$ or $T=T_{\mathrm{ep}}$ updates, where $\eta_T:=I_{H,\mathrm M}^{-1}(U_T)$ and $\tilde O$ hides constants depending on $c$ and logarithms of $\bar\kappa$ and $e^H_{\mathrm M,0}/\varepsilon$:
\begin{equation}\label{eq:fh-mtd-rates}
\begin{gathered}
T_{\mathrm{iid}}=\tilde O\bigl(\rho_{\min}^{-1}\bigl[\bar\tau_{H,\mathrm M}H^{1+c}\varepsilon^{-2}+H^2\bigr]\bigr),\\
T_{\mathrm{ep}}=\tilde O\bigl(\rho_{H,\min}^{-1}\bigl[\bar\tau_{H,\mathrm M}H^{1+c}\varepsilon^{-2}+\bar B_{H,\mathrm M}H^3\varepsilon^{-1}+H^3\bigr]\bigr).
\end{gathered}
\end{equation}
The step size is of order $\alpha_{\mathrm{iid}}:=\min\{1/(H\ell),\allowbreak\varepsilon^2/(\bar\tau_{H,\mathrm M}H^c\ell)\}$ under i.i.d.\ sampling and $\min\{\alpha_{\mathrm{iid}},\allowbreak1/(H^2\ell),\allowbreak\varepsilon/(H^2\bar B_{H,\mathrm M})\}$ under episodic sampling.
\end{corollary}

\noindent\textbf{Horizon dependence.} When $(1-\gamma)^{-1}$ is replaced by $H$, the discounted i.i.d.\ rates yield the fixed-horizon ones. Under episodic sampling the leading terms agree, and a Poisson constant $2H$ enters the lower-order terms: the phase-state chain regenerates at every reset, so trajectory dependence never extends beyond one episode, independent of the mixing of $P^\pi$. The coverage enters as $\nu_{\min}^{-1}$, and the leading CTD term $\bar\tau_HH^2\varepsilon^{-2}$ carries no additional factor of $H$ under episodic sampling. As in the discounted case, the gain over the worst-case one-step variance $B_{H,\mathrm C}^2=H$ summed along $H$ steps comes from the total variance: the one-step Cram\'er variances along $h$ steps add up to at most $2h+2h^2\Delta_{\max}\le4h$ for the grids of Corollary~\ref{cor:undisc}. In $\bar W_1$, the leading term is $\tilde O(\nu_{\min}^{-1}H^3\varepsilon^{-2})$, the fixed-horizon counterpart of the $(1-\gamma)^{-3}$ rate of the discounted setting.

\section{Minimax lower bounds}\label{sec:lower}
We now show that no estimator improves the leading terms of Corollaries~\ref{cor:sharp-ctd}--\ref{cor:undisc-mtd}. An estimator is any function of the observed transitions and of independent randomness, so CTD and MTD with any step size are included. We write $\bar H:=(1-\gamma)^{-1}$ with $\gamma\ge1/2$ in the discounted setting and $\bar H:=H\ge2$ in the fixed-horizon setting, and we bound the error at a single state, at horizon $H$ in the fixed-horizon setting, which is never larger than the supremum metrics of the corollaries. Let $D:=\bar H^{1/2}$ for the Cram\'er distance, $D:=\bar H$ for $W_1$, and $D:=(k_c\kappa_1)^{1/2}\bar H^{c/2}$ for the MMD, where $\kappa_1:=d_\kappa(e_1,0)$ for the first standard basis vector $e_1$ of $\mathbb R^q$ and $k_c\in(0,1]$ depends only on $c$. Up to a factor depending on $c$, $D$ is the largest distance between two laws on $[0,\bar H]$ (Lemma~\ref{lem:lb-sep}).

\begin{theorem}[Lower bounds]\label{thm:lower}
Let $\delta\in(0,1/4)$, $\underline\nu\in(0,1/8]$, $\varepsilon\le D/96$, and, under Markovian sampling, $\hat h\ge\bar H/\underline\nu$. Consider the MDPs and policies with $\nu_{\min}\ge\underline\nu$, and also $h_Q\le \hat h$ under Markovian sampling. For every estimator computed from $T$ transitions, there exist such an MDP and policy that make the error exceed $\varepsilon$ w.p.\ more than $\delta$ whenever
\begin{equation}\label{eq:lb-T}
T<\frac{\bar HD^2\log(1/(4\delta))}{27648\,\underline\nu\,\varepsilon^2},
\end{equation}
or, under Markovian sampling, whenever $T\le \hat h/2$.
\end{theorem}
\begin{proof}[Proof sketch]
We use Le Cam's two-point method \citep{tsybakov2009introduction} on MDPs with two states $x$ and $z$, one action, reward $1$ at $x$ ($e_1$ for MTD) and $0$ at $z$, where the chain stays at $x$ w.p.\ $p$, else moves to $z$, and moves from $z$ to $x$ w.p.\ $r$. For \eqref{eq:lb-T}, the two MDPs have $p=1-1/\bar H$ and $p=1-1/\bar H+\Delta_p$, with $\Delta_p$ of order $\varepsilon/(\bar HD)$, a perturbation of a self-loop probability as in the value-estimation lower bound of \citet{gheshlaghi2013minimax}. This separates their return laws at $x$ by more than $2\varepsilon$, by the Cauchy--Schwarz inequality for Cram\'er and a Fourier representation of the kernel for the MMD. Each transition from $x$ carries Kullback--Leibler divergence $O(\bar H\Delta_p^2)$, the sampling law or $r$ makes these transitions a fraction $\underline\nu$ of the data, and the Bretagnolle--Huber inequality yields the result. For $T\le \hat h/2$, the two MDPs differ only in the reward at $x$ and have $r=1/\hat h$, so a trajectory started at $z$ misses $x$ for $\hat h/2$ steps w.p.\ at least $1/2$ \citep{peng2026sharp}. Appendix~\ref{app:lower} gives the full proof.
\end{proof}

Substituting the three values of $D$ into \eqref{eq:lb-T}, every estimator needs $\Omega(\nu_{\min}^{-1}\bar H^2\varepsilon^{-2})$ transitions in Cram\'er, $\Omega(\nu_{\min}^{-1}\bar H^3\varepsilon^{-2})$ in $\bar W_1$, and $\Omega(\nu_{\min}^{-1}\kappa_1\bar H^{1+c}\varepsilon^{-2})$ in MMD, up to $\log(1/\delta)$ and, in MMD, factors depending on $c$. Under the grid conditions, these match the leading terms of Corollaries~\ref{cor:sharp-ctd}--\ref{cor:undisc-mtd} under all three sampling models, up to logarithmic factors and, for MTD, the factor $\bar\kappa/(k_c\kappa_1)$, which depends only on $c$, $q$, and the kernel (Appendix~\ref{app:lower}). The condition $T\le \hat h/2$ matches the additive $h_Q$ of the Markovian rates, so over chains with $\mu_{\min}\ge\underline\nu$ and $h_Q\le \hat h$ the Markovian CTD Cram\'er rate is minimax optimal up to logarithms for $\varepsilon\le\min\{(1-\gamma)^{1/2},(1-\gamma)^{-1/2}/96\}$. The condition $\varepsilon\le D/96$ is mild, since any two return laws are within distance $O(D)$ (Lemma~\ref{lem:lb-sep}).

\section{Experiments}\label{sec:experiments}
Figure~\ref{fig:sweeps} tests Theorem~\ref{thm:sharp} for constant-step CTD and MTD on a five-state ring. The noise floor, the stationary squared error to the projected fixed point, is linear in $\alpha$ (a, e) and does not grow with temporal dependence (c, g), and coverage changes only the transient length (d, h). In (b, f), the floor is of order $\alpha\,\mathrm{TV}_{\star,\max}$. Without the total variance, it can only be bounded by $\alpha$ times the worst-case one-step variance summed over the horizon, $(1-\gamma)^{-2}$ for CTD, which yields a $(1-\gamma)^{-3}$ Cram\'er rate. This sum blows up with the horizon, while the floor divided by $\alpha$ stays nearly constant, as the theory predicts. Further experimental results are given in Appendix~\ref{app:more-exp}.
\section{Conclusion}
We established finite-iteration guarantees for asynchronous categorical distributional temporal-difference learning in the scalar Cram\'er and multivariate signed-categorical MMD settings, together with matching minimax lower bounds. The scope of the present theory is tabular finite-state policy evaluation. Extending the unrolled analysis to linear function approximation and to control is left for future work.

\bibliographystyle{apalike}
\bibliography{refs}

\clearpage
\appendix
\thispagestyle{empty}

\section*{Appendix Table of Contents}
\addcontentsline{toc}{section}{Appendix}
\markboth{Appendix}{Appendix}
\startcontents[appendix]
\printcontents[appendix]{l}{1}{\setcounter{tocdepth}{3}}

\newpage

\section{Notation}
\label{app:notation}

{\renewcommand{\arraystretch}{1.35}
\begin{longtable}{@{}p{0.25\linewidth}p{0.71\linewidth}@{}}
\caption{Notation used in the paper.}\label{tab:notation}\\
\toprule
Symbol & Meaning \\
\endfirsthead
\toprule
Symbol & Meaning \\
\endhead
\bottomrule
\endfoot
\bottomrule
\endlastfoot
\midrule
\multicolumn{2}{@{}l}{\textbf{Markov decision process}} \\
\midrule
$\X$, $\A$ & finite state and action spaces \\
$P(\cdot\mid s,a)$ & transition kernel \\
$\pi$ & fixed policy \\
$P^\pi$ & state transition matrix, $P^\pi(s,s')=\sum_a\pi(a\mid s)P(s'\mid s,a)$ \\
$R(s,a)$ & deterministic reward, in $[0,1]$ for CTD and in $[0,1]^q$ for MTD \\
$q$ & reward dimension of MTD, $q\ge2$ \\
$\gamma$ & discount factor in $(0,1)$ \\
$T^\pi$ & discounted distributional Bellman operator \\
$\eta^\pi$ & true return law, the unique fixed point of $T^\pi$ \\
$f_{r,\gamma}$, $\#$ & map $z\mapsto r+\gamma z$, and measure pushforward operation\\
\midrule
\multicolumn{2}{@{}l}{\textbf{Sampling}} \\
\midrule
$(S_k,A_k,R_k,S'_k)$ & queried state, action, reward, and sampled successor at iteration $k$ \\
$\rho$ & query-state distribution in the i.i.d.\ regime \\
$\rho_{\min}$ & $\min_s\rho(s)$ \\
$Q$ & transition matrix of the Markovian state trajectory, $Q=P^\pi$ \\
$\mu_{\X}$ & stationary distribution of $Q$ \\
$\mu_{\min}$ & $\min_s\mu_{\X}(s)$ \\
$t_{\mathrm{mix}}$ & mixing time of an aperiodic chain, $\max_s\mathrm{TV}(Q^t(s,\cdot),\mu_{\X})\le(1/4)^{\lfloor t/t_{\mathrm{mix}}\rfloor}$ for all $t\ge0$ \\
$\hat Q$ & time reversal of $Q$, $\hat Q(x,u)=\mu_{\X}(u)Q(u,x)/\mu_{\X}(x)$ \\
$h_Q$ & reverse hitting time, $\max_{i,x}\mathbb E_x[\tau_i]$ for the chain with kernel $\hat Q$ (Lemma~\ref{lem:rev-hit}) \\
$Z_Q$ & fundamental matrix $(I-Q+\bm 1\mu_{\X}^{\mathsf T})^{-1}$ \\
$C_Q^{\mathrm{Pois}}$ & maximum absolute row sum of $Z_Q$ \\
$C_Q^\star$ & $\max\{C_Q^{\mathrm{Pois}},C_{\hat Q}^{\mathrm{Pois}}\}$ \\
$\nu$ & sampling distribution: $\rho$, $\mu_{\X}$, or $\bar\rho_H$ \\
$\nu_{\min}$ & $\min_s\nu(s)$ \\
$D_\nu$ & $\operatorname{diag}(\nu)$ \\
$\mathcal F_k$ & filtration of the discounted analysis (Appendix~\ref{app:sharp}) \\
$\mathcal H_k$ & history through the $k$th queried state (Appendix~\ref{app:common}) \\
$\mathcal G_n$ & reverse-time filtration (Appendix~\ref{app:sharp-rev}) \\
\midrule
\multicolumn{2}{@{}l}{\textbf{Step sizes, accuracy, and confidence}} \\
\midrule
$\alpha_k$, $\alpha$ & step size, and constant step size \\
$\alpha_{\mathrm{iid}}$ & order of the i.i.d.\ step size in Corollaries~\ref{cor:undisc} and~\ref{cor:undisc-mtd} \\
$\varepsilon$ & target accuracy \\
$\delta$ & failure probability \\
$\ell$ & confidence parameter \\
$K$, $L$ & number of epochs, and epoch length \\
$T$ & number of updates \\
\midrule
\multicolumn{2}{@{}l}{\textbf{Categorical representations}} \\
\midrule
$\Theta(s)$, $d$ & support of $d$ atoms at state $s$, ordered in $\mathbb R$ for CTD and distinct in $\mathbb R^q$ for MTD \\
$\mathcal F_{\mathrm C,\Theta}^{\X}$ & state-indexed categorical laws on the supports \\
$\mathcal F_{\mathrm M,\Theta}^{\X}$ & state-indexed mass-$1$ signed categorical laws on the supports \\
$\Pi_{\mathrm C}^{\Theta}$ & statewise Cram\'er projection, with clipping \\
$\Pi_{\mathrm M}^{\Theta}$ & statewise signed MMD projection \\
$I_{\mathrm C}$ & cumulative-mass embedding \\
$I_{\mathrm M}$ & Gram-matrix embedding \\
$K_s$ & Gram matrix at state $s$ \\
$\Delta_{\max}$ & largest gap between consecutive atoms \\
\midrule
\multicolumn{2}{@{}l}{\textbf{Kernels and metrics}} \\
\midrule
$\kappa$ & characteristic kernel of the MMD \\
$c$ & homogeneity exponent of the shift-invariant semimetric inducing $\kappa$, in $(0,2)$ \\
$d_\kappa$ & $d_\kappa(x,y)=\mathrm{MMD}_\kappa(\delta_x,\delta_y)^2$ \\
$\bar\kappa$ & supremum of $d_\kappa$ over $[0,1]^q$ \\
$\ell_{\mathrm C}$ & Cram\'er distance \\
$\ell_{\mathrm C,\infty}$ & supremum Cram\'er metric over states \\
$\ell_{\mathrm M,\infty}$ & supremum MMD metric over states \\
$\bar W_1$ & supremum over states of the $1$-Wasserstein distance, with $\bar W_1\le(1-\gamma)^{-1/2}\ell_{\mathrm C,\infty}$ for laws on $[0,(1-\gamma)^{-1}]$ \\
$\lVert\cdot\rVert_{2,\infty}$ & block-supremum norm \\
\midrule
\multicolumn{2}{@{}l}{\textbf{Recursions and operators}} \\
\midrule
$\star$ & method index, $\star\in\{\mathrm C,\mathrm M\}$, dropped in the appendices when the method is generic \\
$\mathcal O_\star$ & embedded projected operator \\
$\widehat T_\star$ & sampled Bellman target \\
$\eta^*_\star$ & projected fixed point, the unique fixed point of $\Pi_\star^\Theta T^\pi$ \\
$U^*$ & embedding of the projected fixed point \\
$U_k$ & embedded iterate \\
$W_k$ & centered iterate, $U_k-U^*$ \\
$e_{\star,k}$ & error $\ell_{\star,\infty}(I_\star^{-1}(U_k),\eta^*_\star)=\lVert W_k\rVert_{2,\infty}$ \\
$P_s$ & coordinate projector onto block $s$ \\
$\beta$ & contraction modulus of a generic method \\
$\beta_{\mathrm C}$, $\beta_{\mathrm M}$ & $\sqrt\gamma$ and $\gamma^{c/2}$ \\
$\bar A$ & averaged asynchronous operator, $I-\alpha D_\nu(I-\mathcal P)$ \\
$\mathcal P$ & linear part of $\mathcal O_\star$ \\
$\mathcal V$ & propagation matrix of squared block norms, $I-\alpha D_\nu(I-\beta^2P^\pi)$, used for the noise variance (Lemma~\ref{lem:propagation}) \\
$\mathcal R$ & propagation matrix of block norms, $I-\alpha D_\nu(I-\beta P^\pi)$, used for terms bounded in norm rather than in variance (Lemma~\ref{lem:propagation}) \\
\midrule
\multicolumn{2}{@{}l}{\textbf{Noise and total variance}} \\
\midrule
$\xi$ & fixed-point noise, $\widehat T_\star(U^*;y)-U^*(s)$ \\
$b$ & almost-sure bound on $\xi$, written $B_\star$ in Section~\ref{sec:discounted} \\
$\sigma^2$ & conditional second moment of $\xi$ \\
$B_{\mathrm C}$ & $\max_s\sqrt{\theta_d(s)-\theta_1(s)}$ \\
$B_{\mathrm M}$ & largest MMD between $\eta^*_{\mathrm M}(s)$ and a sampled target computed from $\eta^*_{\mathrm M}$ (Proposition~\ref{prop:mtd-perturbation}) \\
$\bar B_{\mathrm M}$ & explicit bound on $B_{\mathrm M}$, $\sqrt{\bar\kappa}(1-\gamma)^{-c/2}+2\sqrt{h_\Theta/(1-\gamma^c)}$ (Lemma~\ref{lem:mtd-tv}) \\
$\mathrm{TV}$ & discounted total variance, $(I-\beta^2P^\pi)^{-1}\sigma^2$ \\
$\mathrm{TV}_{\max}$ & maximum of $\mathrm{TV}$ over states, written $\mathrm{TV}_{\star,\max}$ in Section~\ref{sec:discounted} \\
$\bar\tau$ & normalized total variance of CTD, $\bar\tau:=2+8\Delta_{\max}/(1-\gamma)$, so that $\mathrm{TV}_{\mathrm C,\max}\le\bar\tau/(1-\gamma)$ (Lemma~\ref{lem:ctd-tv}). The grid conditions make it at most $4$ \\
$h_\Theta$ & largest $d_\kappa$-distance from a possible return to its nearest atom \\
$\bar\tau_{\mathrm M}$ & normalized total variance of MTD, $\bar\tau_{\mathrm M}:=\bar\kappa+2h_\Theta(1-\gamma)^c/(1-\gamma^c)^2$, so that $\mathrm{TV}_{\mathrm M,\max}\le\bar\tau_{\mathrm M}(1-\gamma)^{-c}$ (Lemma~\ref{lem:mtd-tv}). The first term comes from the kernel and the second from the grid \\
\midrule
\multicolumn{2}{@{}l}{\textbf{Constants of the discounted guarantees}} \\
\midrule
$C_F$ & universal Freedman constant \\
$N_\star$ & noise level of Theorem~\ref{thm:sharp} under i.i.d.\ sampling \\
$N_\star^{\mathrm{mk}}$ & noise level of Theorem~\ref{thm:sharp} under Markovian sampling \\
$\mathcal E_\star$ & energy constant $4\,\mathrm{TV}_{\star,\max}+3B_\star^2$, written $\mathcal E$ in Appendix~\ref{app:sharp-rev} \\
$L_{\mathrm{mk}}$, $K_{\mathrm{mk}}$ & block length and number of blocks of Theorem~\ref{thm:sharp} \\
$X_n$ & reversed state $S_{T-n}$ (Appendix~\ref{app:sharp-rev}) \\
$\hat p_n$, $\Upsilon_n$ & scalar and operator weights of the reverse-time analysis \\
$\Psi_i$ & anchored Poisson solution of the reverse-time analysis \\
\midrule
\multicolumn{2}{@{}l}{\textbf{Fixed-horizon setting}} \\
\midrule
$H$ & horizon \\
$h$ & horizon index \\
$\mathcal S_H$ & set of horizon-state pairs, $\{1,\dots,H\}\times\X$ \\
$T_H^\pi$ & fixed-horizon distributional Bellman operator \\
$\eta^\pi_H$ & true $h$-step return laws, $(\eta^{\pi,h})_{h\le H}$ \\
$\eta^*_{H,\star}$ & fixed-horizon projected fixed point of $\Pi_{H,\star}^\Theta T_H^\pi$, determined by backward induction \\
$\ell_{H,\star,\infty}$ & supremum metric over horizon-state pairs \\
$e^H_{\star,k}$ & error of the $k$th iterate in $\ell_{H,\star,\infty}$ \\
$\Theta_h(s)$ & horizon-dependent support \\
$B_{H,\mathrm C}$ & $\max_{h,s}\sqrt{\theta_{h,d}(s)-\theta_{h,1}(s)}$ \\
$B_{H,\mathrm M}$ & largest norm of the fixed-point noise of fixed-horizon MTD \\
$\nu_0$ & reset distribution \\
$\phi_k$ & within-episode phase, $k\bmod H$ \\
$Z_k$ & phase-state pair, $(k\bmod H,S_k)$ \\
$Q_H$ & transition matrix of $Z_k$ \\
$\rho_t$ & phase distributions \\
$\bar\rho_H$ & average episode occupancy \\
$\rho_{H,\min}$ & minimum of $\bar\rho_H$ \\
$N_H$ & horizon shift \\
$\mathcal R_H$ & propagation matrix $I-\alpha D_\nu(I-N_H)$, for both block norms and squared block norms \\
$\mathrm{TV}_H$ & fixed-horizon total variance, $(I-N_H)^{-1}\sigma^2$ \\
$\mathrm{TV}_{H,\max}$ & maximum of $\mathrm{TV}_H$, written $\mathrm{TV}^H_{\star,\max}$ in Section~\ref{sec:undisc} \\
$N^H_\star$ & noise level of Theorem~\ref{thm:undisc} \\
$L_H$ & epoch length of Theorem~\ref{thm:undisc} \\
$C_H^\star$ & episodic Poisson constant, $2H$ \\
$\bar\tau_H$ & normalized fixed-horizon total variance of CTD, $\bar\tau_H:=2+2H\Delta_{\max}$, so that $\mathrm{TV}^H_{\mathrm C,\max}\le\bar\tau_HH$ (Lemma~\ref{lem:fh-ctd-tv}). The grid conditions make it at most $4$ \\
$h_{\Theta,H}$ & counterpart of $h_\Theta$ over $[0,h]^q$ \\
$\bar\tau_{H,\mathrm M}$ & normalized fixed-horizon total variance of MTD, $\bar\tau_{H,\mathrm M}:=\bar\kappa+H^{2-c}h_{\Theta,H}$, so that $\mathrm{TV}^H_{\mathrm M,\max}\le\bar\tau_{H,\mathrm M}H^c$ (Lemma~\ref{lem:fh-mtd-tv}). The first term comes from the kernel and the second from the grid \\
$\bar B_{H,\mathrm M}$ & explicit bound on $B_{H,\mathrm M}$, $\sqrt{\bar\kappa}H^{c/2}+2\sqrt{Hh_{\Theta,H}}$ (Lemma~\ref{lem:fh-mtd-tv}) \\
\midrule
\multicolumn{2}{@{}l}{\textbf{Lower bounds}} \\
\midrule
$\bar H$ & effective horizon, $(1-\gamma)^{-1}$ or $H$ \\
$\underline\nu$ & coverage level of Theorem~\ref{thm:lower} \\
$\hat h$ & bound on the reverse hitting time $h_Q$ in Theorem~\ref{thm:lower} \\
$D$ & metric scale of Theorem~\ref{thm:lower} \\
$\kappa_1$ & $d_\kappa(e_1,0)$ for the first standard basis vector $e_1$ \\
$k_c$ & constant depending only on $c$ (Appendix~\ref{app:lower}) \\
\end{longtable}
}

\clearpage

\section{Common finite-iteration ingredients in the block-supremum geometry}
\label{app:common}

We collect the abstract ingredients \citep{robbins1951stochastic, ljung2003analysis, borkar2008stochastic, kushner2003stochastic} used in the proofs of Theorems~\ref{thm:sharp} and~\ref{thm:undisc}. Throughout this appendix, let
\begin{equation}
V := \prod_{s\in \X} E(s)
\end{equation}
be a finite product of Euclidean state blocks. For $U\in V$, define
\begin{equation}
\lVert U \rVert_{2,\infty} := \max_{s\in \X}\lVert U(s) \rVert_2.
\end{equation}

For each state $s\in \X$, let $P_s:V\to V$ denote the coordinate projector onto block $s$.

The iterates of the categorical recursions lie in a strict subset of the ambient product space $V$. For CTD, they lie in the embedded product of probability simplices, and for MTD, in the embedded affine space of mass-$1$ signed categorical measures. The bounds of Appendices~\ref{app:sharp} and~\ref{app:fh-sharp} hold only on these categorical domains. Each domain is convex, and every sampled update maps it into itself, so the following lemma shows that every iterate stays in its domain. It therefore suffices to verify the structural assumptions on the domain rather than on all of $V$.
\begin{lemma}[Invariant-domain restriction]
\label{lem:invariant-domain}
Let $\mathcal D\subseteq V$ be nonempty, closed, and convex. Suppose that $U_0\in\mathcal D$ and that a full-sample map $\widehat H$ satisfies
\begin{equation}
\widehat H(U;y)\in\mathcal D \qquad\text{for every }U\in\mathcal D\text{ and every admissible sample }y.
\end{equation}
If $0\le\alpha_k\le1$ and
\begin{equation}
U_{k+1}=(1-\alpha_k)U_k+\alpha_k\widehat H(U_k;Y_k),
\end{equation}
then $U_k\in\mathcal D$ for every $k$.
\end{lemma}

\begin{proof}
The claim follows by induction because $U_{k+1}$ is a convex combination of two points in $\mathcal D$.
\end{proof}

\subsection{Discounted setting}

Let $\mathcal D\subseteq V$ be nonempty, closed, and convex. Let $\mathcal O:\mathcal D\to\mathcal D$ be a contraction in $\lVert \cdot \rVert_{2,\infty}$ with modulus $\beta\in(0,1)$ and fixed point $U^*\in\mathcal D$. Let $\widehat T(U;s,\chi)$ be a sampled target, where $\chi=(a,r,s')$ collects the action, reward, and successor state, satisfying
\begin{equation}
\mathbb E\left[ \widehat T(U_k;S_k,\chi_k) \mid U_k, S_k=s\right] = (\mathcal O U_k)(s),
\end{equation}
and consider the asynchronous recursion
\begin{equation}
U_{k+1} = U_k + \alpha_k P_{S_k}\left(\widehat T(U_k;S_k,\chi_k)-U_k(S_k)\right).
\end{equation}
Define the full-sample map
\begin{equation}
\widehat H(U;s,\chi):=U+P_s\bigl(\widehat T(U;s,\chi)-U(s)\bigr),
\end{equation}
and assume that $U_0\in\mathcal D$ and $\widehat H(U;s,\chi)\in\mathcal D$ for every $U\in\mathcal D$ and every admissible $(s,\chi)$. By Lemma~\ref{lem:invariant-domain}, every iterate then belongs to $\mathcal D$ when $\alpha_k\le1$. We write $W_k:=U_k-U^*$ and $\mathcal H_k:=\sigma(U_0,S_0,\chi_0,\dots,S_{k-1},\chi_{k-1},S_k)$ for the history through $U_k$ and $S_k$ before the sample at time $k$ is revealed.

\subsection{Poisson equation of a finite Markov chain}

Let $(S_k)_{k\ge0}$ be an irreducible Markov chain on $\X$ with transition matrix $Q$ and stationary distribution $\mu_{\X}$. Aperiodicity is not needed in this subsection. Define
\begin{equation}
\mu_{\min}:=\min_{s\in\X}\mu_{\X}(s)>0,\qquad
Z_Q:=\bigl(I-Q+\bm 1\mu_{\X}^{\mathsf T}\bigr)^{-1},\qquad
C_Q^{\mathrm{Pois}}:=\max_{s\in\X}\sum_{x\in\X}\lvert (Z_Q)_{sx}\rvert.
\end{equation}
The inverse exists for every finite irreducible chain, including periodic chains. The quantity $C_Q^{\mathrm{Pois}}$ is an explicit Poisson condition number that records the temporal dependence of the trajectory.

Let
\begin{equation}
\hat Q(x,u):=\frac{\mu_{\X}(u)Q(u,x)}{\mu_{\X}(x)}
\end{equation}
be the time reversal of $Q$. This is an irreducible transition matrix with the same stationary distribution, and $\hat Q=Q$ when the chain is reversible. Writing $D_\mu:=\operatorname{diag}(\mu_{\X})$,
\begin{equation}
D_\mu^{-1}\bigl(I-Q+\bm1\mu_{\X}^{\mathsf T}\bigr)^{\mathsf T}D_\mu=I-\hat Q+\bm1\mu_{\X}^{\mathsf T},
\qquad\text{hence}\qquad
Z_{\hat Q}(x,u)=\frac{\mu_{\X}(u)}{\mu_{\X}(x)}Z_Q(u,x).
\end{equation}
We write
\begin{equation}
C_Q^\star:=\max\{C_Q^{\mathrm{Pois}},C_{\hat Q}^{\mathrm{Pois}}\},\qquad\text{where}\quad C_{\hat Q}^{\mathrm{Pois}}=\max_{x\in\X}\sum_{u\in\X}\frac{\mu_{\X}(u)}{\mu_{\X}(x)}\lvert Z_Q(u,x)\rvert.
\end{equation}
Since $Z_Q\bm1=\bm1$, every row of $Z_Q$ has absolute sum at least $1$, therefore $C_Q^\star\ge1$.

\begin{lemma}
\label{lem:poisson-facts}
Let
\begin{equation}
z(u,x):=Z_Q(u,x)-\mu_{\X}(x),\qquad v_x(u):=\sum_{u'}Q(u,u')z(u',x)^2-\Bigl(\sum_{u'}Q(u,u')z(u',x)\Bigr)^2.
\end{equation}
Then, the following hold.
\begin{enumerate}
\item For every $x\in\X$,
\begin{equation}
\mu_{\X}^{\mathsf T}Z_Q=\mu_{\X}^{\mathsf T},\qquad\mu_{\X}^{\mathsf T}z(\cdot,x)=0,\qquad z(\cdot,x)-Qz(\cdot,x)=\bm 1_x-\mu_{\X}(x)\bm1.
\end{equation}
\item For every $g:\X\to\mathbb R$, the function $\varphi:=Z_Q(g-\mu_{\X}(g)\bm1)$ solves the Poisson equation
\begin{equation}
\varphi-Q\varphi=g-\mu_{\X}(g)\bm1,\qquad\text{and}\qquad\lVert\varphi\rVert_\infty\le C_Q^{\mathrm{Pois}}\lVert g-\mu_{\X}(g)\bm1\rVert_\infty.
\end{equation}
\item For all $u,x\in\X$,
\begin{equation}
\lvert z(u,x)\rvert\le2C_Q^\star,\qquad\sum_u\mu_{\X}(u)\lvert z(u,x)\rvert\le2C_Q^\star\mu_{\X}(x).
\end{equation}
\item For every $x\in\X$,
\begin{equation}
\sum_u\mu_{\X}(u)v_x(u)=2\mu_{\X}(x)z(x,x)-\mu_{\X}(x)\bigl(1-\mu_{\X}(x)\bigr)\le4C_Q^\star\mu_{\X}(x),
\end{equation}
and
\begin{equation}
\sum_u\mu_{\X}(u)z(u,x)^2\le4(C_Q^\star)^2\mu_{\X}(x).
\end{equation}
\end{enumerate}
\end{lemma}

\begin{proof}
For the first statement, multiplying $(I-Q+\bm1\mu_{\X}^{\mathsf T})Z_Q=I$ on the left by $\mu_{\X}^{\mathsf T}$ and using $\mu_{\X}^{\mathsf T}Q=\mu_{\X}^{\mathsf T}$ gives $\mu_{\X}^{\mathsf T}Z_Q=\mu_{\X}^{\mathsf T}$, and then $\mu_{\X}^{\mathsf T}z(\cdot,x)=\mu_{\X}(x)-\mu_{\X}(x)=0$. The same identity gives
\begin{equation}
(I-Q)Z_Q=I-\bm1\mu_{\X}^{\mathsf T}Z_Q=I-\bm1\mu_{\X}^{\mathsf T},
\end{equation}
and $(I-Q)\bm1=0$ gives the Poisson equation for $z(\cdot,x)$.

For the second statement, since $\mu_{\X}^{\mathsf T}\varphi=\mu_{\X}^{\mathsf T}(g-\mu_{\X}(g)\bm1)=0$,
\begin{equation}
(I-Q)\varphi=\bigl(I-Q+\bm1\mu_{\X}^{\mathsf T}\bigr)\varphi=g-\mu_{\X}(g)\bm1,
\end{equation}
and the norm bound is the definition of $C_Q^{\mathrm{Pois}}$.

For the third statement, the first bound is $\lvert z(u,x)\rvert\le\lvert Z_Q(u,x)\rvert+\mu_{\X}(x)\le C_Q^{\mathrm{Pois}}+1\le2C_Q^\star$. For the second, the expression of $Z_{\hat Q}$ above gives
\begin{equation}
\sum_u\mu_{\X}(u)\lvert z(u,x)\rvert\le\mu_{\X}(x)\Bigl(\sum_u\lvert Z_{\hat Q}(x,u)\rvert+1\Bigr)\le2C_Q^\star\mu_{\X}(x).
\end{equation}

For the fourth statement, write $f:=z(\cdot,x)$ and $g:=\bm1_x-\mu_{\X}(x)\bm1$, so that $Qf=f-g$ by the first statement. Stationarity gives $\mu_{\X}(Q(f^2))=\mu_{\X}(f^2)$, hence
\begin{equation}
\sum_u\mu_{\X}(u)v_x(u)=\mu_{\X}(f^2)-\mu_{\X}\bigl((f-g)^2\bigr)=2\mu_{\X}(fg)-\mu_{\X}(g^2).
\end{equation}
Here
\begin{equation}
\mu_{\X}(fg)=\mu_{\X}(x)f(x)-\mu_{\X}(x)\mu_{\X}(f)=\mu_{\X}(x)z(x,x),\qquad\mu_{\X}(g^2)=\mu_{\X}(x)(1-\mu_{\X}(x)),
\end{equation}
which gives the identity, and the third statement gives the bound. Finally,
\begin{equation}
\sum_u\mu_{\X}(u)z(u,x)^2\le\max_u\lvert z(u,x)\rvert\sum_u\mu_{\X}(u)\lvert z(u,x)\rvert\le4(C_Q^\star)^2\mu_{\X}(x).
\end{equation}
\end{proof}

When the chain is aperiodic, its mixing time $t_{\mathrm{mix}}\in\mathbb N$ is any integer such that
\begin{equation}\label{eq:tmix}
\max_{s\in\X}\mathrm{TV}\bigl(Q^t(s,\cdot),\mu_{\X}\bigr)\le(1/4)^{\lfloor t/t_{\mathrm{mix}}\rfloor}\qquad\text{for all }t\ge0,
\end{equation}
as in \citet{peng2025statisticalefficiencydistributionaltemporal}. A periodic irreducible chain has no finite mixing time, while $C_Q^\star$ is still finite.

\section{CTD details}
\label{app:ctd}

We verify the hypotheses of Appendix~\ref{app:common} for the discounted scalar categorical temporal-difference recursion, and Appendix~\ref{app:sharp} uses these results to prove Theorem~\ref{thm:sharp}.

\subsection{Statewise Cram\'er isometry and contraction}

For each state $s\in \X$, let
\begin{equation}
\Theta(s)=\{\theta_1(s)<\cdots<\theta_d(s)\}\subset\mathbb R
\end{equation}
and let $\mathcal F^{\X}_{\mathrm C,\Theta}$ denote the class of state-indexed categorical laws supported on these statewise grids. For
\begin{equation}
\eta(s)=\sum_{i=1}^d p_i(s)\delta_{\theta_i(s)},
\end{equation}
define the cumulative masses and the grid separations
\begin{equation}
c^\eta_j(s) := \sum_{i=1}^j p_i(s),
\qquad
\Delta_j(s) := \theta_{j+1}(s)-\theta_j(s),
\qquad
j=1,\dots,d-1.
\end{equation}
The statewise Cram\'er embedding is
\begin{equation}
I_{\mathrm C,s}(\eta(s)) := \bigl(\sqrt{\Delta_1(s)}\,c^\eta_1(s), \dots, \sqrt{\Delta_{d-1}(s)}\,c^\eta_{d-1}(s),0\bigr)^\top \in\mathbb R^d.
\end{equation}

\begin{proposition}[Statewise Cram\'er isometry]
\label{prop:ctd-isometry}
For every state $s\in \X$ and all $\mu,\nu\in \mathcal F_{\mathrm C,\Theta(s)}$,
\begin{equation}
\ell_{\mathrm C}(\mu,\nu)^2 = \sum_{j=1}^{d-1}\Delta_j(s)\bigl(c^\mu_j(s)-c^\nu_j(s)\bigr)^2 = \lVert I_{\mathrm C,s}(\mu)-I_{\mathrm C,s}(\nu) \rVert_2^2.
\end{equation}
Consequently, with the product embedding
\begin{equation}
I_{\mathrm C}(\eta):=(I_{\mathrm C,s}(\eta(s)))_{s\in \X},
\end{equation}
we have
\begin{equation}
\lVert I_{\mathrm C}(\eta)-I_{\mathrm C}(\eta') \rVert_{2,\infty} = \ell_{\mathrm C,\infty}(\eta,\eta').
\end{equation}
\end{proposition}

\begin{proof}
For laws supported on the ordered grid $\Theta(s)$, the cumulative distribution functions are piecewise
constant on the intervals
\begin{equation}
[\theta_j(s),\theta_{j+1}(s)), \qquad j=1,\dots,d-1,
\end{equation}
with values $c^\mu_j(s)$ and $c^\nu_j(s)$ respectively. Therefore,
\begin{equation}
\ell_{\mathrm C}(\mu,\nu)^2 = \int_{\mathbb R}\bigl(F_\mu(z)-F_\nu(z)\bigr)^2\,\mathrm dz = \sum_{j=1}^{d-1}\Delta_j(s)\bigl(c^\mu_j(s)-c^\nu_j(s)\bigr)^2,
\end{equation}
which is exactly $\lVert I_{\mathrm C,s}(\mu)-I_{\mathrm C,s}(\nu) \rVert_2^2$. Taking the maximum over $s\in \X$ gives the product-space identity.
\end{proof}

Recall from Section~\ref{sec:ctd-new2} that
\begin{equation}
\mathcal O_{\mathrm C} := I_{\mathrm C}\circ \Pi_{\mathrm C}^{\Theta}T^\pi \circ I_{\mathrm C}^{-1}.
\end{equation}

Define the embedded categorical domain
\begin{equation}
\mathcal D_{\mathrm C}:=I_{\mathrm C}(\mathcal F^{\X}_{\mathrm C,\Theta})\subseteq V.
\end{equation}

\begin{proposition}[Invariance of the CTD domain]
\label{prop:ctd-domain}
The set $\mathcal D_{\mathrm C}$ is compact and convex, $\mathcal O_{\mathrm C}(\mathcal D_{\mathrm C})\subseteq\mathcal D_{\mathrm C}$, and the full-sample replacement map
\begin{equation}
\widehat H_{\mathrm C}(U;s,(r,s')):=U+P_s\bigl(\widehat T_{\mathrm C}(U;s,(r,s'))-U(s)\bigr)
\end{equation}
satisfies
\begin{equation}
\widehat H_{\mathrm C}(U;s,(r,s'))\in\mathcal D_{\mathrm C}
\end{equation}
for every $U\in\mathcal D_{\mathrm C}$ and every admissible $(s,r,s')$. Consequently, if $U_0\in\mathcal D_{\mathrm C}$ and $0\le\alpha_k\le1$, then every CTD iterate belongs to $\mathcal D_{\mathrm C}$.
\end{proposition}

\begin{proof}
The class $\mathcal F^{\X}_{\mathrm C,\Theta}$ is a finite product of probability simplices, and the cumulative-mass embedding is linear on each simplex, its image $\mathcal D_{\mathrm C}$ therefore being compact and convex. The projected Bellman operator and every sampled projected target are categorical laws on the prescribed support. Hence $\mathcal O_{\mathrm C}$ maps $\mathcal D_{\mathrm C}$ into itself, and replacing one state block by a sampled embedded target gives $\widehat H_{\mathrm C}(U;s,(r,s'))\in\mathcal D_{\mathrm C}$. The final claim follows from Lemma~\ref{lem:invariant-domain}.
\end{proof}

For a finitely supported probability law $\nu$ on $\mathbb R$, the projection $\Pi_{\mathrm C}^{\Theta(s)}\nu$ moves the mass of $\nu$ at $z\in[\theta_j(s),\theta_{j+1}(s)]$ to $\theta_j(s)$ and $\theta_{j+1}(s)$ with weights $(\theta_{j+1}(s)-z)/\Delta_j(s)$ and $(z-\theta_j(s))/\Delta_j(s)$, and moves the mass below $\theta_1(s)$ to $\theta_1(s)$ and the mass above $\theta_d(s)$ to $\theta_d(s)$. We write $\ell_{\mathrm C}$ for the Cram\'er distance between finitely supported laws on $\mathbb R$.

\begin{lemma}[Cram\'er projection and pushforward]
\label{lem:ctd-projection}
Let $\Theta=\{\theta_1<\cdots<\theta_d\}\subset\mathbb R$ be any grid and let $\mu,\nu$ be finitely supported probability laws on $\mathbb R$. Then
\begin{equation}
\ell_{\mathrm C}(\Pi_{\mathrm C}^{\Theta}\mu,\Pi_{\mathrm C}^{\Theta}\nu)\le\ell_{\mathrm C}(\mu,\nu).
\end{equation}
Moreover, for $r\in\mathbb R$ and any $\gamma'>0$ (we use $\gamma'=\gamma$ and $\gamma'=1$),
\begin{equation}
\ell_{\mathrm C}\bigl((f_{r,\gamma'})_\#\mu,(f_{r,\gamma'})_\#\nu\bigr)=\sqrt{\gamma'}\,\ell_{\mathrm C}(\mu,\nu),
\end{equation}
and for finitely supported probability laws $\mu_i,\nu_i$ and mixture weights $w_i\ge0$ with $\sum_iw_i=1$,
\begin{equation}
\ell_{\mathrm C}\Bigl(\sum_iw_i\mu_i,\sum_iw_i\nu_i\Bigr)^2\le\sum_iw_i\,\ell_{\mathrm C}(\mu_i,\nu_i)^2.
\end{equation}
\end{lemma}

\begin{proof}
Let $\Delta_j:=\theta_{j+1}-\theta_j$. Checking point masses in each of the three cases above and using linearity in the law shows that the cumulative mass of $\Pi_{\mathrm C}^{\Theta}\mu$ at $\theta_j$ equals the cell average $\Delta_j^{-1}\int_{\theta_j}^{\theta_{j+1}}F_\mu(z)\,\mathrm dz$. The distribution function of $\Pi_{\mathrm C}^{\Theta}\mu$ is $0$ below $\theta_1$, $1$ from $\theta_d$ on, and equal to this average on $[\theta_j,\theta_{j+1})$. Jensen's inequality on each cell gives
\begin{equation}
\ell_{\mathrm C}(\Pi_{\mathrm C}^{\Theta}\mu,\Pi_{\mathrm C}^{\Theta}\nu)^2=\sum_{j=1}^{d-1}\Delta_j\Bigl(\frac1{\Delta_j}\int_{\theta_j}^{\theta_{j+1}}(F_\mu-F_\nu)\,\mathrm dz\Bigr)^2\le\int_{\mathbb R}(F_\mu-F_\nu)^2\,\mathrm dz=\ell_{\mathrm C}(\mu,\nu)^2.
\end{equation}
The distribution function of $(f_{r,\gamma'})_\#\mu$ is $z\mapsto F_\mu((z-r)/\gamma')$, and the substitution $z=r+\gamma'y$ gives the scaling identity. Finally, the distribution function of a mixture is the mixture of the distribution functions, and the mixture inequality follows from convexity of $x\mapsto x^2$.
\end{proof}

The lemma does not require $\mu$ and $\nu$ to be supported on $\Theta$. We apply it to $(f_{r,\gamma})_\#\eta(s')$, which is supported on $r+\gamma\Theta(s')$ rather than on $\Theta(s)$.

\begin{proposition}[Embedded CTD contraction]
\label{prop:ctd-contraction-app}
For all $U,U'\in I_{\mathrm C}(\mathcal F^{\X}_{\mathrm C,\Theta})$,
\begin{equation}
\lVert \mathcal O_{\mathrm C}U-\mathcal O_{\mathrm C}U'\rVert_{2,\infty} \le \sqrt{\gamma}\,\lVert U-U'\rVert_{2,\infty}.
\end{equation}
\end{proposition}

\begin{proof}
Let $\eta:=I_{\mathrm C}^{-1}(U)$ and $\eta':=I_{\mathrm C}^{-1}(U')$. For each $s$,
\begin{equation}
(T^\pi\eta)(s)=\sum_{a,s'}\pi(a\mid s)P(s'\mid s,a)\,(f_{R(s,a),\gamma})_\#\eta(s').
\end{equation}
Lemma~\ref{lem:ctd-projection} applied to the projection, the mixture, and the pushforward gives
\begin{equation}
\ell_{\mathrm C}\bigl((\Pi_{\mathrm C}^{\Theta}T^\pi\eta)(s),(\Pi_{\mathrm C}^{\Theta}T^\pi\eta')(s)\bigr)\le\sqrt\gamma\max_{s'}\ell_{\mathrm C}\bigl(\eta(s'),\eta'(s')\bigr),
\end{equation}
which is the contraction of $\Pi_{\mathrm C}^{\Theta}T^\pi$ in the supremum Cram\'er metric \citep{rowland2018analysis,bellemare2023distributional}. With Proposition~\ref{prop:ctd-isometry},
\begin{equation}
\lVert \mathcal O_{\mathrm C}U-\mathcal O_{\mathrm C}U'\rVert_{2,\infty} = \ell_{\mathrm C,\infty}(\Pi_{\mathrm C}^{\Theta}T^\pi\eta,\Pi_{\mathrm C}^{\Theta}T^\pi\eta') \le \sqrt{\gamma}\,\ell_{\mathrm C,\infty}(\eta,\eta') = \sqrt{\gamma}\,\lVert U-U'\rVert_{2,\infty}.
\end{equation}
\end{proof}

\subsection{Unbiased sampled target and bounded blocks}

Recall the sampled CTD target from Section~\ref{sec:ctd-new2}:
\begin{equation}
\widehat T_{\mathrm C}(U_k;S_k,(R_k,S'_k)) := I_{\mathrm C,S_k} \Bigl(\Pi_{\mathrm C}^{\Theta(S_k)}\bigl((f_{R_k,\gamma})_\# I_{\mathrm C,S'_k}^{-1}(U_k(S'_k))\bigr)\Bigr),
\end{equation}
where $f_{r,\gamma}(z)=r+\gamma z$.

\begin{proposition}[Conditional unbiasedness of the sampled CTD target]
\label{prop:ctd-unbiased}
For every $U\in I_{\mathrm C}(\mathcal F^{\X}_{\mathrm C,\Theta})$ and every state $s\in \X$,
\begin{equation}
\mathbb E\left[\widehat T_{\mathrm C}(U;s,(R_k,S'_k))\mid U_k=U, S_k=s\right] = (\mathcal O_{\mathrm C}U)(s).
\end{equation}
\end{proposition}

\begin{proof}
Condition on $U_k=U$ and $S_k=s$. By definition,
\begin{equation}
\widehat T_{\mathrm C}(U;s,(R_k,S'_k)) = I_{\mathrm C,s}\Bigl(\Pi_{\mathrm C}^{\Theta(s)}\bigl((f_{R_k,\gamma})_\# I_{\mathrm C,S'_k}^{-1}(U(S'_k))\bigr)\Bigr).
\end{equation}
Taking conditional expectation with respect to the one-step transition and reward law under $\pi$ gives exactly the Bellman expectation at state $s$. Applying the deterministic statewise projection and then the embedding $I_{\mathrm C,s}$ yields
\begin{equation}
\mathbb E\left[\widehat T_{\mathrm C}(U;s,(R_k,S'_k))
\mid U_k=U, S_k=s\right] = I_{\mathrm C,s}\bigl((\Pi_{\mathrm C}^{\Theta}T^\pi I_{\mathrm C}^{-1}U)(s)\bigr) = (\mathcal O_{\mathrm C}U)(s).
\end{equation}
Here, the expectation commutes with the categorical projection and the cumulative-mass embedding because both maps are linear in the input measure. The same computation holds conditionally on $\mathcal H_k$ in the Markovian regime, because $(A_k,S'_k)$ given $\mathcal H_k$ has law $\pi(\cdot\mid S_k)P(\cdot\mid S_k,\cdot)$.
\end{proof}

\begin{proposition}[Uniform bound on embedded categorical blocks]
\label{prop:ctd-block-bound}
For every state $s\in \X$ and every embedded categorical law $U(s)\in I_{\mathrm C,s}(\mathcal F_{\mathrm C,\Theta(s)})$,
\begin{equation}
\lVert U(s) \rVert_2^2 \le \theta_d(s)-\theta_1(s).
\end{equation}
\end{proposition}

\begin{proof}
Write
\begin{equation}
U(s)=I_{\mathrm C,s}(\eta(s)) = \bigl( \sqrt{\Delta_1(s)}\,c^\eta_1(s), \dots, \sqrt{\Delta_{d-1}(s)}\,c^\eta_{d-1}(s),0\bigr)^\top.
\end{equation}
Because $\eta(s)$ is a probability law, every cumulative mass $c^\eta_j(s)$ lies in $[0,1]$. Hence
\begin{equation}
\lVert U(s) \rVert_2^2 = \sum_{j=1}^{d-1}\Delta_j(s)\bigl(c^\eta_j(s)\bigr)^2 \le \sum_{j=1}^{d-1}\Delta_j(s) = \theta_d(s)-\theta_1(s).
\end{equation}
\end{proof}

\section{MTD details}
\label{app:mtd}

We now verify the abstract hypotheses of Appendix~\ref{app:common} for the discounted multivariate signed-categorical recursion.

\subsection{Statewise MMD isometry and contraction}

For each state $s\in \X$, let
\begin{equation}
\Theta(s)=\{\theta_1(s),\dots,\theta_d(s)\}\subset\mathbb R^q, \qquad \mathbb R_1^d := \Bigl\{ p\in\mathbb R^d : \sum_{i=1}^d p_i=1 \Bigr\}.
\end{equation}
Let $\mathcal M(\R^q)$ denote the space of finite signed Borel measures on $\R^q$. Define
\begin{equation}
\mathcal F^{\X}_{\mathrm M,\Theta} := \Bigl\{ \eta : \X\to \mathcal M(\mathbb R^q) : \eta(s)=\sum_{i=1}^d p_i(s)\delta_{\theta_i(s)}, \ p(s)\in \mathbb R_1^d \text{ for all }s\in \X \Bigr\}.
\end{equation}

Let $\kappa:\mathbb R^q\times\mathbb R^q\to\mathbb R$ be the characteristic kernel from Section~\ref{sec:mtd-new2}. For each state $s$, let
\begin{equation}
K_s := \bigl( \kappa(\theta_i(s),\theta_j(s)) \bigr)_{i,j=1}^d
\end{equation}
be the Gram matrix on $\Theta(s)$. For
\begin{equation}
\eta(s)=\sum_{i=1}^d p_i(s)\delta_{\theta_i(s)},
\end{equation}
define
\begin{equation}
I_{\mathrm M,s}(\eta(s)) := K_s^{1/2}p(s)\in\mathbb R^d.
\end{equation}

\begin{proposition}[Statewise MMD isometry]
\label{prop:mtd-isometry}
For every state $s\in \X$ and all signed-categorical laws $\mu,\nu$ supported on $\Theta(s)$,
\begin{equation}
\mathrm{MMD}_\kappa(\mu,\nu)^2 = \lVert I_{\mathrm M,s}(\mu)-I_{\mathrm M,s}(\nu) \rVert_2^2.
\end{equation}
Consequently, with
\begin{equation}
I_{\mathrm M}(\eta):=(I_{\mathrm M,s}(\eta(s)))_{s\in \X},
\end{equation}
we have
\begin{equation}
\lVert I_{\mathrm M}(\eta)-I_{\mathrm M}(\eta') \rVert_{2,\infty} = \ell_{\mathrm M,\infty}(\eta,\eta').
\end{equation}
\end{proposition}

\begin{proof}
Write
\begin{equation}
\mu=\sum_{i=1}^d p_i\delta_{\theta_i(s)}, \qquad \nu=\sum_{i=1}^d q_i\delta_{\theta_i(s)}.
\end{equation}
Then
\begin{equation}
\mu-\nu=\sum_{i=1}^d (p_i-q_i)\delta_{\theta_i(s)}.
\end{equation}
By the definition of MMD,
\begin{equation}
\mathrm{MMD}_\kappa(\mu,\nu)^2 = \sum_{i=1}^d\sum_{j=1}^d (p_i-q_i)(p_j-q_j)\kappa(\theta_i(s),\theta_j(s)) = (p-q)^\top K_s(p-q).
\end{equation}
Since $K_s$ is positive semidefinite,
\begin{equation}
(p-q)^\top K_s(p-q) = \lVert K_s^{1/2}(p-q) \rVert_2^2 = \lVert I_{\mathrm M,s}(\mu)-I_{\mathrm M,s}(\nu) \rVert_2^2.
\end{equation}
Taking the maximum over $s\in \X$ gives the product isometry.
\end{proof}

Because the points $\theta_1(s),\dots,\theta_d(s)$ are distinct, each $\eta(s)\in\mathcal F_{\mathrm M,\Theta(s)}$ determines its weight vector $p(s)$, so $I_{\mathrm M,s}$ is well defined. Since $\kappa$ is characteristic, $\mathrm{MMD}_\kappa$ is a metric on finitely supported mass-$1$ signed measures \citep{wiltzer2024foundations}, and Proposition~\ref{prop:mtd-isometry} then shows that $I_{\mathrm M,s}(\mu)=I_{\mathrm M,s}(\nu)$ implies $\mu=\nu$. Hence $I_{\mathrm M}$ is injective and $I_{\mathrm M}^{-1}$ is well defined on its image, even when $K_s$ is singular.

Recall from Section~\ref{sec:mtd-new2} that
\begin{equation}
\mathcal O_{\mathrm M} := I_{\mathrm M}\circ \Pi_{\mathrm M}^{\Theta}T^\pi \circ I_{\mathrm M}^{-1}.
\end{equation}

Define the embedded signed-categorical domain
\begin{equation}
\mathcal D_{\mathrm M}:=I_{\mathrm M}(\mathcal F^{\X}_{\mathrm M,\Theta})\subseteq V.
\end{equation}

\begin{proposition}[Invariance of the MTD domain]
\label{prop:mtd-domain}
The set $\mathcal D_{\mathrm M}$ is a closed affine subset of $V$, $\mathcal O_{\mathrm M}(\mathcal D_{\mathrm M})\subseteq\mathcal D_{\mathrm M}$, and the full-sample replacement map
\begin{equation}
\widehat H_{\mathrm M}(U;s,(r,s')):=U+P_s\bigl(\widehat T_{\mathrm M}(U;s,(r,s'))-U(s)\bigr)
\end{equation}
satisfies
\begin{equation}
\widehat H_{\mathrm M}(U;s,(r,s'))\in\mathcal D_{\mathrm M}
\end{equation}
for every $U\in\mathcal D_{\mathrm M}$ and every admissible $(s,r,s')$. Consequently, if $U_0\in\mathcal D_{\mathrm M}$ and $0\le\alpha_k\le1$, then every MTD iterate belongs to $\mathcal D_{\mathrm M}$.
\end{proposition}

\begin{proof}
At each state, the coefficient vectors of mass-$1$ signed categorical laws form the affine hyperplane $\mathbb R_1^d$. The image of this hyperplane under the linear Gram-matrix embedding is affine and closed, so the product of these images over the states, $\mathcal D_{\mathrm M}$, is closed, affine, and therefore convex. The signed-categorical projection has mass $1$, both for the expected Bellman target and for each sampled Bellman target. Thus, $\mathcal O_{\mathrm M}$ and every full-sample replacement preserve $\mathcal D_{\mathrm M}$. Lemma~\ref{lem:invariant-domain} gives the final claim.
\end{proof}

\begin{proposition}[Embedded MTD contraction]
\label{prop:mtd-contraction-app}
For all $U,U'\in I_{\mathrm M}(\mathcal F^{\X}_{\mathrm M,\Theta})$,
\begin{equation}
\lVert \mathcal O_{\mathrm M}U-\mathcal O_{\mathrm M}U'\rVert_{2,\infty} \le \gamma^{c/2}\lVert U-U'\rVert_{2,\infty}.
\end{equation}
\end{proposition}

\begin{proof}
Let $\eta:=I_{\mathrm M}^{-1}(U)$ and $\eta':=I_{\mathrm M}^{-1}(U')$. Let
\begin{equation}
d_\kappa(x,y):=\mathrm{MMD}_\kappa(\delta_x,\delta_y)^2=\kappa(x,x)+\kappa(y,y)-2\kappa(x,y).
\end{equation}
For mass-$1$ signed measures, $\mu-\nu$ has total mass zero, so expanding $d_\kappa$ gives
\begin{equation}
\mathrm{MMD}_\kappa(\mu,\nu)^2=-\frac12\iint d_\kappa\,\mathrm d(\mu-\nu)\,\mathrm d(\mu-\nu).
\end{equation}
Since $\kappa$ is induced by a shift-invariant, $c$-homogeneous semimetric, $d_\kappa$ is that semimetric up to a positive factor. Hence
\begin{equation}
d_\kappa(r+\gamma x,r+\gamma y)=\gamma^cd_\kappa(x,y),\qquad\text{and so}\qquad\mathrm{MMD}_\kappa\bigl((f_{r,\gamma})_\#\mu,(f_{r,\gamma})_\#\nu\bigr)=\gamma^{c/2}\,\mathrm{MMD}_\kappa(\mu,\nu).
\end{equation}
The MMD is the RKHS norm of the difference of mean embeddings, so it is convex under mixtures. Combining these two facts with the nonexpansiveness in Proposition~\ref{prop:mtd-proj-affine} below gives the contraction of $\Pi_{\mathrm M}^{\Theta}T^\pi$ in the supremum MMD metric \citep{wiltzer2024foundations}. With Proposition~\ref{prop:mtd-isometry},
\begin{equation}
\lVert \mathcal O_{\mathrm M}U-\mathcal O_{\mathrm M}U'\rVert_{2,\infty} = \ell_{\mathrm M,\infty}(\Pi_{\mathrm M}^{\Theta}T^\pi\eta,\Pi_{\mathrm M}^{\Theta}T^\pi\eta') \le \gamma^{c/2}\ell_{\mathrm M,\infty}(\eta,\eta') = \gamma^{c/2}\lVert U-U'\rVert_{2,\infty}.
\end{equation}
\end{proof}

\subsection{Unbiased sampled target and fixed-point noise}

Recall the sampled MTD target from Section~\ref{sec:mtd-new2}:
\begin{equation}
\widehat T_{\mathrm M}(U_k;S_k,(R_k,S'_k)) := I_{\mathrm M,S_k} \Bigl( \Pi_{\mathrm M}^{\Theta(S_k)} \bigl( (f_{R_k,\gamma})_\# I_{\mathrm M,S'_k}^{-1}(U_k(S'_k)) \bigr) \Bigr).
\end{equation}

\begin{proposition}[Embedded MTD projection as an affine Euclidean projector]
\label{prop:mtd-proj-affine}
For each state $s \in \X$, let
\begin{equation}
\mathcal A_s := K_s^{1/2}\mathbb R_1^d \subset \mathbb R^d .
\end{equation}
For every finitely supported mass-$1$ signed measure $\nu$ on $\mathbb R^q$, define
\begin{equation}
b_s(\nu) := \left( \int \kappa(\theta_i(s),y)\,d\nu(y) \right)_{i=1}^d \in \mathbb R^d
\end{equation}
and
\begin{equation}
\zeta_s(\nu) := K_s^{\dagger/2} b_s(\nu),
\end{equation}
where $K_s^\dagger$ is the Moore-Penrose pseudoinverse of $K_s$. Then $b_s(\nu) \in \operatorname{range}(K_s)$ and
\begin{equation}
I_{\mathrm M,s}\bigl(\Pi_{\mathrm M}^{\Theta(s)}\nu\bigr) = \operatorname{proj}_{\mathcal A_s}\bigl(\zeta_s(\nu)\bigr),
\end{equation}
where $\operatorname{proj}_{\mathcal A_s}$ is the Euclidean orthogonal projection onto $\mathcal A_s$. Consequently, the map
\begin{equation}
J_s(\nu) := I_{\mathrm M,s}\bigl(\Pi_{\mathrm M}^{\Theta(s)}\nu\bigr)
\end{equation}
is affine on the affine space of finitely supported mass-$1$ signed measures, and $\lVert J_s(\mu)-J_s(\nu)\rVert_2\le\mathrm{MMD}_\kappa(\mu,\nu)$ for all such $\mu,\nu$.
\end{proposition}

\begin{proof}
Let $\mathcal H_\kappa$ be the RKHS of $\kappa$, let $\varphi : \mathbb R^q \to \mathcal H_\kappa$ denote the feature map, and define the linear operator
\begin{equation}
\Phi_s : \mathbb R^d \to \mathcal H_\kappa, \qquad \Phi_s p := \sum_{i=1}^d p_i \varphi(\theta_i(s)).
\end{equation}
Then
\begin{equation}
K_s = \Phi_s^*\Phi_s .
\end{equation}
For a finitely supported mass-$1$ signed measure $\nu$, define its kernel mean embedding by
\begin{equation}
m_\nu := \int \varphi(y)\,d\nu(y) \in \mathcal H_\kappa .
\end{equation}
Let $\mathcal P_s$ denote the orthogonal projection in $\mathcal H_\kappa$ onto $\operatorname{range}(\Phi_s)$. Since $\Phi_s^*(I-\mathcal P_s)=0$, we have
\begin{equation}
b_s(\nu) = \Phi_s^* m_\nu = \Phi_s^* \mathcal P_s m_\nu \in \operatorname{range}(\Phi_s^*\Phi_s) = \operatorname{range}(K_s).
\end{equation}

Now fix $p \in \mathbb R_1^d$ and write
\begin{equation}
\mu_p := \sum_{i=1}^d p_i \delta_{\theta_i(s)} .
\end{equation}
By the reproducing-kernel representation of MMD,
\begin{equation}
\mathrm{MMD}_\kappa(\mu_p,\nu)^2 = \lVert \Phi_s p - m_\nu \rVert_{\mathcal H_\kappa}^2 = \lVert \Phi_s p - \mathcal P_s m_\nu \rVert_{\mathcal H_\kappa}^2 + \lVert (I-\mathcal P_s)m_\nu \rVert_{\mathcal H_\kappa}^2 .
\end{equation}

We next identify $\mathcal P_s m_\nu$ in coefficient form. Since $K_s=\Phi_s^*\Phi_s$, the standard Moore-Penrose identity gives
\begin{equation}
\Phi_s K_s^\dagger \Phi_s^* = \mathcal P_s .
\end{equation}
Therefore,
\begin{equation}
\Phi_s\bigl(K_s^\dagger b_s(\nu)\bigr) = \Phi_s K_s^\dagger \Phi_s^* m_\nu = \mathcal P_s m_\nu .
\end{equation}
Substituting this into the previous display yields
\begin{equation}
\lVert \Phi_s p - \mathcal P_s m_\nu \rVert_{\mathcal H_\kappa} = \lVert \Phi_s(p-K_s^\dagger b_s(\nu)) \rVert_{\mathcal H_\kappa}.
\end{equation}
Since $\lVert \Phi_s x \rVert_{\mathcal H_\kappa}^2 = x^\top K_s x = \lVert K_s^{1/2}x \rVert_2^2$ for every $x \in \mathbb R^d$, we obtain
\begin{equation}
\lVert \Phi_s p - \mathcal P_s m_\nu \rVert_{\mathcal H_\kappa} = \lVert K_s^{1/2}(p-K_s^\dagger b_s(\nu)) \rVert_2 = \lVert K_s^{1/2}p - K_s^{1/2}K_s^\dagger b_s(\nu) \rVert_2 .
\end{equation}
Because $K_s$ is symmetric positive semidefinite, spectral calculus gives
\begin{equation}
K_s^{1/2}K_s^\dagger = K_s^{\dagger/2}.
\end{equation}
Hence
\begin{equation}
\lVert \Phi_s p - \mathcal P_s m_\nu \rVert_{\mathcal H_\kappa} = \lVert K_s^{1/2}p - K_s^{\dagger/2} b_s(\nu) \rVert_2 = \lVert K_s^{1/2}p - \zeta_s(\nu) \rVert_2 .
\end{equation}
Therefore
\begin{equation}
\mathrm{MMD}_\kappa(\mu_p,\nu)^2 = \lVert K_s^{1/2}p - \zeta_s(\nu) \rVert_2^2 + c_s(\nu),
\end{equation}
where
\begin{equation}
c_s(\nu) := \lVert (I-\mathcal P_s)m_\nu \rVert_{\mathcal H_\kappa}^2
\end{equation}
does not depend on $p$. Minimizing over $p \in \mathbb R_1^d$ is therefore equivalent to Euclidean projection of $\zeta_s(\nu)$ onto
\begin{equation}
\mathcal A_s = K_s^{1/2}\mathbb R_1^d .
\end{equation}
Since $I_{\mathrm M,s}\bigl(\sum_{i=1}^d p_i\delta_{\theta_i(s)}\bigr)=K_s^{1/2}p$, it follows that
\begin{equation}
I_{\mathrm M,s}\bigl(\Pi_{\mathrm M}^{\Theta(s)}\nu\bigr) = \operatorname{proj}_{\mathcal A_s}\bigl(\zeta_s(\nu)\bigr).
\end{equation}

Finally, $\nu \mapsto b_s(\nu)$ is linear, hence so is $\nu \mapsto \zeta_s(\nu)=K_s^{\dagger/2}b_s(\nu)$. Since orthogonal projection onto an affine subspace is an affine map, the composition
\begin{equation}
J_s(\nu) = \operatorname{proj}_{\mathcal A_s}\bigl(\zeta_s(\nu)\bigr)
\end{equation}
is affine on the affine space of finitely supported mass-$1$ signed measures. For the last claim, $\lVert K_s^{\dagger/2}\Phi_s^*x\rVert_2^2=\langle x,\Phi_sK_s^\dagger\Phi_s^*x\rangle_{\mathcal H_\kappa}=\lVert\mathcal P_sx\rVert_{\mathcal H_\kappa}^2$ for every $x\in\mathcal H_\kappa$. With $x=m_\mu-m_\nu$ this gives $\lVert\zeta_s(\mu)-\zeta_s(\nu)\rVert_2\le\lVert m_\mu-m_\nu\rVert_{\mathcal H_\kappa}=\mathrm{MMD}_\kappa(\mu,\nu)$, and Euclidean projection onto $\mathcal A_s$ is $1$-Lipschitz.
\end{proof}

\begin{proposition}[Conditional unbiasedness of the sampled MTD target]
\label{prop:mtd-unbiased}
For every admissible $U\in I_{\mathrm M}(\mathcal F^{\X}_{\mathrm M,\Theta})$ and every state $s\in \X$,
\begin{equation}
\mathbb E\left[ \widehat T_{\mathrm M}(U;s,(R_k,S'_k)) \,\mid\, U_k=U,\ S_k=s \right] = (\mathcal O_{\mathrm M}U)(s).
\end{equation}
\end{proposition}

\begin{proof}
Let $\eta:=I_{\mathrm M}^{-1}(U)$ and condition on $U_k=U$ and $S_k=s$. By Proposition~\ref{prop:mtd-proj-affine}, the statewise map
\begin{equation}
J_s(\nu):=I_{\mathrm M,s}\bigl(\Pi_{\mathrm M}^{\Theta(s)}\nu\bigr)
\end{equation}
is affine on finitely supported mass-$1$ signed measures. Hence
\begin{equation}
\begin{aligned}
\mathbb E\left[ \widehat T_{\mathrm M}(U;s,(R_k,S'_k)) \,\mid\, U_k=U,\ S_k=s \right] &= \mathbb E\left[ J_s\bigl((f_{R_k,\gamma})_\#\eta(S'_k)\bigr) \,\mid\, U_k=U,\ S_k=s \right] \\ 
&= J_s\left( \mathbb E\left[ (f_{R_k,\gamma})_\#\eta(S'_k) \,\mid\, U_k=U,\ S_k=s \right] \right)\\ 
&= J_s\bigl((T^\pi\eta)(s)\bigr)\\ 
&= I_{\mathrm M,s}\bigl((\Pi_{\mathrm M}^{\Theta}T^\pi\eta)(s)\bigr)\\ 
&= (\mathcal O_{\mathrm M}U)(s).
\end{aligned}
\end{equation}
The same computation holds conditionally on $\mathcal H_k$ in the Markovian regime, since $(A_k,S'_k)$ given $\mathcal H_k$ has law $\pi(\cdot\mid S_k)P(\cdot\mid S_k,\cdot)$.
\end{proof}

\begin{proposition}[MTD fixed-point noise]
\label{prop:mtd-perturbation}
Let $\eta^*$ be the unique fixed point of $\Pi_{\mathrm M}^{\Theta}T^\pi$ and $U^*:=I_{\mathrm M}(\eta^*)$. The constant
\begin{equation}
B_{\mathrm M} := \max_{s,s'\in \X} \sup_{r\in[0,1]^q} \left\lVert \widehat T_{\mathrm M}(U^*;s,(r,s'))-U^*(s) \right\rVert_2
\end{equation}
is finite, and it bounds the norm of the fixed-point noise $\widehat T_{\mathrm M}(U^*;y)-U^*(s)$ for every sample $y=(s,a,s')$.
\end{proposition}

\begin{proof}
Fix $(s,s')$ and write $\widehat T_{\mathrm M}(U^*;s,(r,s'))=J_s((f_{r,\gamma})_\#\eta^*(s'))$, with $J_s$ as in Proposition~\ref{prop:mtd-proj-affine}. Since $J_s$ is $1$-Lipschitz from $\mathrm{MMD}_\kappa$ and $\kappa$ is shift invariant, coupling the two translates of $\eta^*(s')$ atom by atom gives
\begin{equation}
\bigl\lVert\widehat T_{\mathrm M}(U^*;s,(r,s'))-\widehat T_{\mathrm M}(U^*;s,(r',s'))\bigr\rVert_2\le\sum_i\lvert p_i(s')\rvert\,\mathrm{MMD}_\kappa(\delta_r,\delta_{r'}),
\end{equation}
where $p(s')$ is the weight vector of $\eta^*(s')$. The right-hand side tends to $0$ as $r'\to r$ because $\kappa$ is continuous, so the supremum over the compact set $[0,1]^q$ is finite. The second claim holds because $R(s,a)\in[0,1]^q$.
\end{proof}

\section{Analysis of the discounted recursions}
\label{app:sharp}

This appendix proves Theorem~\ref{thm:sharp} and Corollaries~\ref{cor:sharp-ctd} and~\ref{cor:sharp-mtd}. Under i.i.d.\ sampling, we unroll the recursion around the averaged asynchronous operator, whose powers are deterministic, and propagate statewise squared norms through the nonnegative matrix $\mathcal V$. Noise injected at state $u$ reaches state $s$ after $k$ steps with weight $(\mathcal V^k)_{s,u}$, and summing these weights against the sampling frequencies gives $\alpha^{-1}(I-\beta^2P^\pi)^{-1}$ (Lemma~\ref{lem:propagation}). As a result, the coverage enters only through $\nu_{\min}^{-1}$, and the accumulated noise is bounded by the discounted total variance of Lemmas~\ref{lem:ctd-tv} and~\ref{lem:mtd-tv}. Theorem~\ref{thm:sharp-app} also covers a Markov chain through the Poisson equation of Appendix~\ref{app:common}, which Appendix~\ref{app:fh-sharp} uses for the episodic chain. Under Markovian sampling, Theorem~\ref{thm:sharp} is proved by the reverse-time analysis of Appendix~\ref{app:sharp-rev}, which follows \citet{peng2026sharp} and reads the trajectory backward.

\subsection{Structural assumptions}

Let $V=\prod_{s\in\X}E(s)$ be as in Appendix~\ref{app:common}, let $\mathcal D\subseteq V$ be the invariant domain of the method, and let $U^*\in\mathcal D$ be the embedded fixed point. Write $E_0(s)$ for the linear subspace of $E(s)$ spanned by differences of points of the $s$th block of $\mathcal D$, and $E_0:=\prod_sE_0(s)$. A sample is $y=(s,a,s')$, with reward $R(s,a)$. Let $P^\pi(s,s'):=\sum_a\pi(a\mid s)P(s'\mid s,a)$.

\begin{assumption}[Structure of the sampled target]
\label{ass:sharp}
There exists $\beta\in(0,1)$ such that the following hold.
\begin{enumerate}
\item For every sample $y=(s,a,s')$ there is a linear map $L_y:E_0(s')\to E_0(s)$ with operator norm at most $\beta$ such that
\begin{equation}
\widehat T(U;y)-\widehat T(U';y)=L_y\bigl(U(s')-U'(s')\bigr)\qquad\text{for all }U,U'\in\mathcal D.
\end{equation}
\item The fixed-point noise $\xi(y):=\widehat T(U^*;y)-U^*(s)$ satisfies $\mathbb E[\xi(Y)\mid S=s]=0$ when $A\sim\pi(\cdot\mid s)$ and $S'\sim P(\cdot\mid s,A)$, and $\lVert\xi(y)\rVert_2\le b$ for every sample.
\end{enumerate}
We write $\sigma^2(s):=\mathbb E[\lVert\xi(Y)\rVert_2^2\mid S=s]$ and define the discounted total variance
\begin{equation}
\mathrm{TV}:=(I-\beta^2P^\pi)^{-1}\sigma^2=\sum_{k\ge0}\beta^{2k}(P^\pi)^k\sigma^2\in\mathbb R^{\X},\qquad \mathrm{TV}_{\max}:=\max_{s\in\X}\mathrm{TV}(s).
\end{equation}
\end{assumption}
Note that Assumption~\ref{ass:sharp} is not a restriction on the problem: Lemmas~\ref{lem:ctd-tv} and~\ref{lem:mtd-tv} verify it for CTD and MTD with no further conditions.

Define the linear map $\mathcal P:E_0\to E_0$ by
\begin{equation}
(\mathcal PY)(s):=\mathbb E\bigl[L_YY(S')\mid S=s\bigr].
\end{equation}
Taking conditional expectations in the first item gives $\mathcal OU-\mathcal OU'=\mathcal P(U-U')$ for $U,U'\in\mathcal D$, and Jensen's inequality then yields
\begin{equation}\label{eq:sq-mixture}
\lVert(\mathcal PY)(s)\rVert_2^2\le\beta^2\sum_{s'\in\X}P^\pi(s,s')\lVert Y(s')\rVert_2^2\qquad\text{for all }Y\in E_0,\ s\in\X.
\end{equation}

\noindent\textbf{Sampling and recursion.} In the i.i.d.\ regime, the samples $Y_t=(S_t,A_t,S'_t)$ are i.i.d.\ with $S_t\sim\rho$, and we set
\begin{equation}
\nu:=\rho,\qquad\mathcal F_t:=\sigma(Y_0,\dots,Y_{t-1}).
\end{equation}
In the Markovian regime $S'_t=S_{t+1}$, the chain $(S_t)$ has transition matrix $Q=P^\pi$ and stationary distribution $\mu_{\X}$, and we set
\begin{equation}
\nu:=\mu_{\X},\qquad\mathcal F_t:=\sigma(S_0,Y_0,\dots,Y_{t-1}),
\end{equation}
which contains $S_t$. In both regimes $U_0\in\mathcal D$ is deterministic, $\alpha\in(0,1]$ is a constant step size, and
\begin{equation}
U_{t+1}=U_t+\alpha P_{S_t}\bigl(\widehat T(U_t;Y_t)-U_t(S_t)\bigr),\qquad W_t:=U_t-U^*\in E_0.
\end{equation}
We write $\nu_{\min}:=\min_s\nu(s)$ and $D_\nu:=\operatorname{diag}(\nu)$.

\subsection{Deterministic propagation}

Define the averaged operator and the squared-norm propagation matrix
\begin{equation}
\bar A:=I-\alpha D_\nu(I-\mathcal P)\ \text{ on }E_0,\qquad \mathcal V:=I-\alpha D_\nu(I-\beta^2P^\pi)\in\mathbb R^{\X\times\X}.
\end{equation}
For $Y\in E_0$ write $q_Y(u):=\lVert Y(u)\rVert_2^2$ and $r_Y(u):=\lVert Y(u)\rVert_2$, and let
\begin{equation}
\mathcal R:=I-\alpha D_\nu(I-\beta P^\pi)\in\mathbb R^{\X\times\X}.
\end{equation}

\begin{lemma}[Propagation]
\label{lem:propagation}
For every $Y\in E_0$, $k\ge0$, and $s\in\X$:
\begin{enumerate}
\item $\lVert\bar A^kY\rVert_{2,\infty}\le(1-\alpha\nu_{\min}(1-\beta))^k\lVert Y\rVert_{2,\infty}$;
\item $\lVert(\bar A^kY)(s)\rVert_2^2\le(\mathcal V^kq_Y)(s)$;
\item $\mathcal V$ is entrywise nonnegative with row sums at most $1$, and $\sum_{k\ge0}\mathcal V^kD_\nu=\alpha^{-1}(I-\beta^2P^\pi)^{-1}$;
\item $\lVert(\bar A^kY)(s)\rVert_2\le(\mathcal R^kr_Y)(s)$, where $\mathcal R$ is entrywise nonnegative with row sums at most $1$ and $\sum_{k\ge0}\mathcal R^kD_\nu=\alpha^{-1}(I-\beta P^\pi)^{-1}$;
\item entrywise,
\begin{equation}
\lvert\mathcal V^k-\mathcal V^{k+1}\rvert\le\alpha\mathcal V^kD_\nu(I+\beta^2P^\pi),\qquad\lvert\mathcal R^k-\mathcal R^{k+1}\rvert\le\alpha\mathcal R^kD_\nu(I+\beta P^\pi);
\end{equation}
consequently
\begin{equation}
\sum_{k\ge0}\sum_x\lvert\mathcal V^k-\mathcal V^{k+1}\rvert_{s,x}\le\frac2{1-\beta^2},\qquad\sum_{k\ge0}\sum_x\lvert\mathcal R^k-\mathcal R^{k+1}\rvert_{s,x}\le\frac2{1-\beta}.
\end{equation}
\end{enumerate}
\end{lemma}

\begin{proof}
To prove the first statement, note that for each block $u$,
\begin{equation}
(\bar AY)(u)=(1-\alpha\nu(u))Y(u)+\alpha\nu(u)(\mathcal PY)(u),
\end{equation}
a convex combination because $\alpha\nu(u)\le1$. Since $\lVert(\mathcal PY)(u)\rVert_2\le\beta\lVert Y\rVert_{2,\infty}$ by \eqref{eq:sq-mixture},
\begin{equation}
\lVert(\bar AY)(u)\rVert_2\le\bigl(1-\alpha\nu(u)(1-\beta)\bigr)\lVert Y\rVert_{2,\infty},
\end{equation}
and iterating gives the claim.

For the second statement, convexity of $\lVert\cdot\rVert_2^2$ and \eqref{eq:sq-mixture} give
\begin{equation}
\lVert(\bar AY)(u)\rVert_2^2\le(1-\alpha\nu(u))q_Y(u)+\alpha\nu(u)\beta^2(P^\pi q_Y)(u)=(\mathcal Vq_Y)(u).
\end{equation}
Since $\mathcal V$ is entrywise nonnegative, applying this bound to $\bar A^{k-1}Y$ and using monotonicity gives the claim by induction.

For the third statement, the diagonal entries of $\mathcal V$ are at least $1-\alpha\nu(u)\ge0$, the off-diagonal entries are nonnegative, and the row sums equal $1-\alpha\nu(u)(1-\beta^2)<1$. Hence
\begin{equation}
\sum_{k\ge0}\mathcal V^k=\bigl(\alpha D_\nu(I-\beta^2P^\pi)\bigr)^{-1},
\end{equation}
and multiplying by $D_\nu$ on the right gives the claim.

For the fourth statement, the triangle inequality gives $\lVert(\mathcal PY)(u)\rVert_2\le\mathbb E[\lVert L_YY(S')\rVert_2\mid S=u]\le\beta(P^\pi r_Y)(u)$, so, entrywise,
\begin{equation}
r_{\bar AY}\le(1-\alpha\nu)r_Y+\alpha\nu\beta P^\pi r_Y=\mathcal Rr_Y,
\end{equation}
and the rest follows as for $\mathcal V$ with $\beta^2$ replaced by $\beta$.

For the fifth statement, we have
\begin{equation}
\mathcal V^k-\mathcal V^{k+1}=\mathcal V^k(I-\mathcal V)=\alpha\mathcal V^kD_\nu(I-\beta^2P^\pi),
\end{equation}
and, since $\mathcal V^k\ge0$ and $D_\nu\ge0$, its entries are bounded in absolute value by those of $\alpha\mathcal V^kD_\nu(I+\beta^2P^\pi)$. Summing row $s$ over $x$ gives $\alpha(1+\beta^2)(\mathcal V^kD_\nu\bm1)(s)$, and summing over $k$ with the third statement gives
\begin{equation}
(1+\beta^2)\bigl((I-\beta^2P^\pi)^{-1}\bm1\bigr)(s)=\frac{1+\beta^2}{1-\beta^2}\le\frac2{1-\beta^2}.
\end{equation}
The argument for $\mathcal R$ is identical.
\end{proof}

\begin{lemma}[Freedman inequality in Hilbert space]
\label{lem:freedman}
Let $(d_j)_{j<n}$ be a martingale difference sequence in a Hilbert space with respect to a filtration $(\mathcal G_j)$, with $\lVert d_j\rVert\le c$ almost surely, and let
\begin{equation}
\Sigma_n:=\sum_{j<n}\mathbb E\bigl[\lVert d_j\rVert^2\mid\mathcal G_j\bigr].
\end{equation}
There is a universal constant $C_F\ge1$ such that, for all $v>0$ and $\ell\ge1$,
\begin{equation}
\Pr\Bigl(\Bigl\lVert\sum_{j<n}d_j\Bigr\rVert\ge\sqrt{C_Fv\ell}+C_Fc\ell,\ \Sigma_n\le v\Bigr)\le2e^{-\ell}.
\end{equation}
\end{lemma}

\begin{proof}
For martingales in Hilbert space with $\Sigma_n\le v$ almost surely, \citet[Theorem~3.4]{pinelis1994optimum} gives
\begin{equation}
\Pr\Bigl(\Bigl\lVert\sum_jd_j\Bigr\rVert\ge r\Bigr)\le2\exp\Bigl(-\frac{r^2}{C_F(v+cr)}\Bigr);
\end{equation}
see also \citet{peng2025statisticalefficiencydistributionaltemporal}. For the event version, let
\begin{equation}
\tilde d_j:=d_j\,\mathds{1}\Bigl\{\sum_{i\le j}\mathbb E\bigl[\lVert d_i\rVert^2\mid\mathcal G_i\bigr]\le v\Bigr\}.
\end{equation}
The indicator is $\mathcal G_j$-measurable, so $(\tilde d_j)$ is a martingale difference sequence with $\lVert\tilde d_j\rVert\le c$ and total conditional variance at most $v$ almost surely, and $\sum_j\tilde d_j=\sum_jd_j$ on $\{\Sigma_n\le v\}$. With $r=\sqrt{C_Fv\ell}+C_Fc\ell$,
\begin{equation}
r^2-C_Fcr\ell=r\sqrt{C_Fv\ell}\ge C_Fv\ell,
\end{equation}
so the exponent is at most $-\ell$.
\end{proof}

\begin{lemma}[Bernstein bound for Markovian additive functionals]
\label{lem:markov-bernstein}
In the Markovian regime, let $t_0\ge0$, $n\ge1$, $\ell\ge1$, $\Lambda>0$, and let $h_0,\dots,h_{n-1}:\X\to[0,\Lambda]$ be deterministic, with
\begin{equation}
V_\Lambda:=\sum_{j=1}^{n-1}\lVert h_j-h_{j-1}\rVert_\infty.
\end{equation}
Write $C:=C_Q^{\mathrm{Pois}}$. Conditionally on $\mathcal F_{t_0}$ and for any value of $S_{t_0}$, with probability at least $1-4(\lceil\log_2n\rceil+1)e^{-\ell}$,
\begin{equation}
\sum_{j<n}h_j(S_{t_0+j})\le2\sum_{j<n}\mu_{\X}(h_j)+28C_FC\Lambda\ell+6CV_\Lambda.
\end{equation}
\end{lemma}

\begin{proof}
We first decompose the sum with the Poisson equation. Write $S_j$ for $S_{t_0+j}$ and $\mathcal F'_j:=\mathcal F_{t_0+j}$, so that $\Pr(S_{j+1}=x\mid\mathcal F'_j)=Q(S_j,x)$. Let
\begin{equation}
\varphi_j:=Z_Q\bigl(h_j-\mu_{\X}(h_j)\bm1\bigr),\qquad a:=C\Lambda.
\end{equation}
By Lemma~\ref{lem:poisson-facts}(2),
\begin{equation}
\varphi_j-Q\varphi_j=h_j-\mu_{\X}(h_j)\bm1,\qquad\lVert\varphi_j\rVert_\infty\le a,\qquad\lVert\varphi_j-\varphi_{j-1}\rVert_\infty\le2C\lVert h_j-h_{j-1}\rVert_\infty.
\end{equation}
Let $X:=\sum_jh_j(S_j)$ and $\Sigma_\mu:=\sum_j\mu_{\X}(h_j)$. Then $X-\Sigma_\mu=B+M$ with
\begin{equation}
B:=\sum_{j<n}\bigl[\varphi_j(S_j)-\varphi_j(S_{j+1})\bigr],\qquad M:=\sum_{j<n}\bigl[\varphi_j(S_{j+1})-(Q\varphi_j)(S_j)\bigr].
\end{equation}
Summation by parts gives
\begin{equation}
B=\varphi_0(S_0)-\varphi_{n-1}(S_n)+\sum_{j=1}^{n-1}(\varphi_j-\varphi_{j-1})(S_j),\qquad\text{so}\qquad\lvert B\rvert\le2a+2CV_\Lambda.
\end{equation}
The martingale $M$ has increments bounded by $2a$ and total conditional variance
\begin{equation}
V:=\sum_{j<n}\operatorname{Var}_Q(\varphi_j)(S_j),\qquad\operatorname{Var}_Q(\varphi)(u):=(Q\varphi^2)(u)-(Q\varphi)(u)^2.
\end{equation}

Next, we bound $V$. Since $Q\varphi_j=\varphi_j-\tilde h_j$ with $\tilde h_j:=h_j-\mu_{\X}(h_j)$,
\begin{equation}
\operatorname{Var}_Q(\varphi_j)=Q\varphi_j^2-\varphi_j^2+2\varphi_j\tilde h_j-\tilde h_j^2\le Q\varphi_j^2-\varphi_j^2+2a\lvert\tilde h_j\rvert.
\end{equation}
As for $B$, with $M_2:=\sum_j[\varphi_j^2(S_{j+1})-(Q\varphi_j^2)(S_j)]$,
\begin{equation}
\sum_{j<n}\bigl[(Q\varphi_j^2)(S_j)-\varphi_j^2(S_j)\bigr]=\sum_{j<n}\bigl[\varphi_j^2(S_{j+1})-\varphi_j^2(S_j)\bigr]-M_2\le a^2+4aCV_\Lambda-M_2,
\end{equation}
where we used $\lVert\varphi_j^2-\varphi_{j-1}^2\rVert_\infty\le2a\lVert\varphi_j-\varphi_{j-1}\rVert_\infty$. The martingale $M_2$ has increments bounded by $a^2$, and, since $t\mapsto t^2$ is $2a$-Lipschitz on $[-a,a]$, total conditional variance at most $4a^2V$. With $\lvert\tilde h_j\rvert\le h_j+\mu_{\X}(h_j)$,
\begin{equation}\label{eq:bern-V}
V\le a^2+4aCV_\Lambda+\lvert M_2\rvert+2a(X+\Sigma_\mu).
\end{equation}

Since $V$ is random, we apply Freedman's inequality on a grid of variance levels. Since $V\le na^2$, let $v_i:=2^ia^2$ for $0\le i\le\lceil\log_2n\rceil$. By Lemma~\ref{lem:freedman} applied to $M$ and to $M_2$ for each $v_i$, with probability at least $1-4(\lceil\log_2n\rceil+1)e^{-\ell}$, for every $i$ with $V\le v_i$,
\begin{equation}
M\le\sqrt{C_Fv_i\ell}+2C_Fa\ell,\qquad\lvert M_2\rvert\le\sqrt{4C_Fa^2v_i\ell}+C_Fa^2\ell.
\end{equation}
On this event take the smallest $i$ with $v_i\ge\max\{V,a^2\}$, so that $v:=v_i\le2V+2a^2$. Using
\begin{equation}
\sqrt{8C_Fa^2V\ell}\le\frac V2+4C_Fa^2\ell,\qquad\sqrt{8C_Fa^4\ell}\le a^2(2+C_F\ell)
\end{equation}
in \eqref{eq:bern-V}, and then $C_F\ell\ge1$,
\begin{equation}
V\le18C_Fa^2\ell+8aCV_\Lambda+4a(X+\Sigma_\mu).
\end{equation}
Therefore $2V+2a^2\le38C_Fa^2\ell+16aCV_\Lambda+8a(X+\Sigma_\mu)$, and
\begin{equation}
\sqrt{C_F\ell(2V+2a^2)}\le\frac{X+\Sigma_\mu}{4}+8C_Fa\ell+\sqrt{38}\,C_Fa\ell+2C_Fa\ell+2CV_\Lambda,
\end{equation}
where we used
\begin{equation}
\sqrt{8C_F\ell a(X+\Sigma_\mu)}\le\frac{X+\Sigma_\mu}4+8C_Fa\ell,\qquad4\sqrt{C_F\ell aCV_\Lambda}\le2C_Fa\ell+2CV_\Lambda.
\end{equation}

Finally, combining this with $X\le\Sigma_\mu+2a+2CV_\Lambda+2C_Fa\ell+\sqrt{C_Fv\ell}$ gives
\begin{equation}
\frac34X\le\frac54\Sigma_\mu+20.2\,C_Fa\ell+4CV_\Lambda,
\end{equation}
which implies the claim.
\end{proof}

\subsection{Main result}

\begin{theorem}[Discounted finite-iteration bound]
\label{thm:sharp-app}
Suppose Assumption~\ref{ass:sharp} holds, let $\ell\ge1$, and let
\begin{equation}
\alpha\le\frac{1-\beta^2}{2^{11}C_F\ell}.
\end{equation}
In the Markovian regime assume in addition, with $C_Q^\star$ from Appendix~\ref{app:common},
\begin{equation}\label{eq:mk-step}
C_F\,C_Q^\star\,\alpha\ell\le2^{-20}(1-\beta),\qquad C_F^2\,(C_Q^\star)^3\,\alpha^2\ell^2\le2^{-24}(1-\beta).
\end{equation}
Define
\begin{equation}
N:=\sqrt{6C_F\alpha\,\mathrm{TV}_{\max}\ell}+2C_F\alpha b\ell+\mathds{1}_{\mathrm{mk}}\frac{10\alpha C_Q^\star b}{1-\beta},\qquad L:=\Bigl\lceil\frac{\log8}{\alpha\nu_{\min}(1-\beta)}\Bigr\rceil,
\end{equation}
where $\mathds{1}_{\mathrm{mk}}$ is $1$ in the Markovian regime and $0$ otherwise. Let $D_0\ge\max\{\lVert W_0\rVert_{2,\infty},2N\}$ and $K\in\mathbb N$. Then, with probability at least $1-16KL\lvert\X\rvert(\log_2L+3)e^{-\ell}$,
\begin{equation}
\lVert W_{KL}\rVert_{2,\infty}\le2N+2^{-K}D_0.
\end{equation}
\end{theorem}

\begin{proof}
We first decompose the error. Let $g_t:=(\mathcal P-I)W_t\in E_0$ and
\begin{equation}
y_t:=\widehat T(U_t;Y_t)-U_t(S_t)=L_{Y_t}W_t(S'_t)+\xi(Y_t)-W_t(S_t).
\end{equation}
For $c\in\mathbb R^{\X}$ and $Y\in E_0$ write $(c\odot Y)(x):=c(x)Y(x)$, so that $P_uY=\bm1_u\odot Y$ and $D_\nu Y=\nu\odot Y$. In the i.i.d.\ regime $S_t$ is independent of $\mathcal F_t$, so $\mathbb E[P_{S_t}y_t\mid\mathcal F_t]=\rho\odot g_t$. Setting
\begin{equation}
\zeta_t:=P_{S_t}y_t-\rho\odot g_t,\qquad e_t:=0
\end{equation}
gives
\begin{equation}\label{eq:sharp-decomp}
W_{t+1}=\bar AW_t+\alpha\zeta_t+\alpha e_t,\qquad\mathbb E[\zeta_t\mid\mathcal F_t]=0.
\end{equation}
In the Markovian regime, $\mathbb E[P_{S_t}y_t\mid\mathcal F_t]=\bm1_{S_t}\odot g_t$. Setting
\begin{equation}
\zeta_t:=P_{S_t}\bigl(L_{Y_t}W_t(S_{t+1})+\xi(Y_t)-(\mathcal PW_t)(S_t)\bigr),\qquad e_t:=(\bm1_{S_t}-\mu_{\X})\odot g_t
\end{equation}
again gives \eqref{eq:sharp-decomp}.

We run the argument in epochs of length $L$. Set $t_k:=kL$ and $D_{k+1}:=D_k/2+N$. Since $D_0\ge2N$, every $D_k\ge2N$, and
\begin{equation}
D_k-2N=2^{-k}(D_0-2N).
\end{equation}
We show that, conditionally on $\mathcal F_{t_k}$ and on $\{\lVert W_{t_k}\rVert_{2,\infty}\le D_k\}$, with probability at least $1-16L\lvert\X\rvert(\log_2L+3)e^{-\ell}$ one has $\lVert W_{t_{k+1}}\rVert_{2,\infty}\le D_{k+1}$. A union bound over $k<K$ then proves the theorem. Fix $k$, write $D:=D_k$, and let
\begin{equation}
\tau:=\min\bigl\{n\ge0:\lVert W_{t_k+n}\rVert_{2,\infty}>2D\bigr\}.
\end{equation}
The event $\{j<\tau\}$ is $\mathcal F_{t_k+j}$-measurable. For $1\le n\le L$, unrolling \eqref{eq:sharp-decomp} gives
\begin{equation}\label{eq:sharp-unroll}
W_{t_k+n}=\bar A^nW_{t_k}+\alpha\sum_{j<n}\bar A^{n-1-j}\zeta_{t_k+j}+\alpha\sum_{j<n}\bar A^{n-1-j}e_{t_k+j}.
\end{equation}
All probabilities below are conditional on $\mathcal F_{t_k}$, and Lemmas~\ref{lem:freedman} and~\ref{lem:markov-bernstein} apply conditionally because the relevant $c$, $v$, and functions $h_j$ depend only on $D=D_k$, $n$, and $s$. We fix $n\le L$ and $s\in\X$, bound the $s$th block of the last two terms, and measure time from $t_k$, writing $S_j,W_j,y_j,\dots$ for $S_{t_k+j},W_{t_k+j},y_{t_k+j},\dots$. Martingale sums are bounded with their increments multiplied by $\mathds{1}\{j<\tau\}$, and the remaining terms on the event $\{n\le\tau\}$. For $n\le\tau$, both coincide with the unstopped quantities. On $\{j<\tau\}$,
\begin{equation}
\lVert W_j\rVert_{2,\infty}\le2D,\qquad\lVert y_j\rVert_2\le b+4D,\qquad\lVert g_j\rVert_{2,\infty}\le(1+\beta)2D\le4D.
\end{equation}
We use $\alpha\le(1-\beta^2)/(2^{11}C_F\ell)$ repeatedly, in the forms
\begin{equation}
C_F\alpha\ell\le2^{-11},\qquad\frac{C_F\alpha\ell}{1-\beta^2}\le2^{-11}.
\end{equation}

We now bound the noise term in the i.i.d.\ regime. Let $d_j:=\alpha\mathds{1}\{j<\tau\}(\bar A^{n-1-j}\zeta_j)(s)$. Then $(d_j)$ is a martingale difference sequence, and Lemma~\ref{lem:propagation}(1) gives $\lVert d_j\rVert_2\le2\alpha(b+4D)=:c$. Since a conditional variance is at most the conditional second moment, Lemma~\ref{lem:propagation}(2) gives
\begin{equation}
\mathbb E\bigl[\lVert d_j\rVert_2^2\mid\mathcal F_{t_k+j}\bigr]\le\alpha^2\sum_u\rho(u)(\mathcal V^{n-1-j})_{s,u}\,\mathbb E\bigl[\lVert y\rVert_2^2\mid S=u\bigr],
\end{equation}
and
\begin{equation}
\mathbb E\bigl[\lVert y\rVert_2^2\mid S=u\bigr]\le2\sigma^2(u)+2(2\beta D+2D)^2\le2\sigma^2(u)+32D^2.
\end{equation}
Lemma~\ref{lem:propagation}(3) then gives, almost surely,
\begin{equation}
\Sigma_n\le\alpha\Bigl((I-\beta^2P^\pi)^{-1}\bigl(2\sigma^2+32D^2\bm1\bigr)\Bigr)(s)\le\alpha\Bigl(2\mathrm{TV}_{\max}+\frac{32D^2}{1-\beta^2}\Bigr)=:v.
\end{equation}
By Lemma~\ref{lem:freedman}, with probability at least $1-2e^{-\ell}$,
\begin{equation}
\Bigl\lVert\sum_{j<n}d_j\Bigr\rVert_2\le\sqrt{2C_F\alpha\mathrm{TV}_{\max}\ell}+\sqrt{\frac{32C_F\alpha\ell}{1-\beta^2}}\,D+2C_F\alpha\ell(b+4D)\le\frac D4+N,
\end{equation}
because $\sqrt{32C_F\alpha\ell/(1-\beta^2)}\le1/8$ and $8C_F\alpha\ell\le1/8$.

For the noise term in the Markovian regime, define $d_j$ as in the i.i.d.\ regime. Now, $\zeta_j$ is supported on block $S_j$ and $\lVert\zeta_j\rVert_2\le b+4D$ on $\{j<\tau\}$, so $c:=\alpha(b+4D)$, and
\begin{equation}
\mathbb E\bigl[\lVert\zeta_j\rVert_2^2\mid\mathcal F_{t_k+j}\bigr]\le2\sigma^2(S_j)+8D^2=:f(S_j)\qquad\text{on }\{j<\tau\}.
\end{equation}
Lemma~\ref{lem:propagation}(2) gives
\begin{equation}
\Sigma_n\le\sum_{j<n}h_j(S_j),\qquad h_j(u):=\alpha^2(\mathcal V^{n-1-j})_{s,u}f(u)\in\bigl[0,\alpha^2\lVert f\rVert_\infty\bigr],
\end{equation}
where $f$ is deterministic because $D=D_k$ is deterministic. By Lemma~\ref{lem:propagation}(3) with $\nu=\mu_{\X}$ and by Lemma~\ref{lem:propagation}(5),
\begin{equation}
\sum_j\mu_{\X}(h_j)\le\alpha\bigl((I-\beta^2P^\pi)^{-1}f\bigr)(s)\le\alpha\Bigl(2\mathrm{TV}_{\max}+\frac{8D^2}{1-\beta^2}\Bigr),\qquad V_\Lambda\le\frac{2\alpha^2\lVert f\rVert_\infty}{1-\beta^2}.
\end{equation}
Lemma~\ref{lem:markov-bernstein} gives, with probability at least $1-4(\log_2L+2)e^{-\ell}$,
\begin{equation}
\Sigma_n\le2\alpha\Bigl(2\mathrm{TV}_{\max}+\frac{8D^2}{1-\beta^2}\Bigr)+\alpha^2\lVert f\rVert_\infty C_Q^\star\Bigl(28C_F\ell+\frac{12}{1-\beta^2}\Bigr)\le\alpha\Bigl(6\mathrm{TV}_{\max}+\frac{24D^2}{1-\beta^2}\Bigr)=:v,
\end{equation}
because \eqref{eq:mk-step} gives $\alpha C_Q^\star(28C_F\ell+12/(1-\beta^2))\le1$ and $\lVert f\rVert_\infty\le2\mathrm{TV}_{\max}+8D^2$, using $\sigma^2\le\mathrm{TV}$. Lemma~\ref{lem:freedman} then gives, with probability at least $1-2e^{-\ell}$ more,
\begin{equation}
\Bigl\lVert\sum_{j<n}d_j\Bigr\rVert_2\le\sqrt{6C_F\alpha\mathrm{TV}_{\max}\ell}+C_F\alpha b\ell+\frac D8+\frac D{64},
\end{equation}
since $\sqrt{24C_F\alpha\ell/(1-\beta^2)}\le1/8$ and $4C_F\alpha\ell\le1/64$.

We next bound the drift fluctuation in the Markovian regime. Let $z$ and $v_x$ be as in Lemma~\ref{lem:poisson-facts}, write $C^\star:=C_Q^\star$, and write
\begin{equation}
(Qz)(u,x):=\sum_{u'}Q(u,u')z(u',x),
\end{equation}
and similarly $(Q\lvert z\rvert)$ and $(Qz^2)$. By Lemma~\ref{lem:poisson-facts}(1),
\begin{equation}
\mathds{1}\{u=x\}-\mu_{\X}(x)=z(u,x)-(Qz)(u,x).
\end{equation}
With
\begin{equation}
\Gamma_j:=z(S_j,\cdot)\odot g_j,\qquad\Gamma_j^+:=z(S_{j+1},\cdot)\odot g_j,\qquad m_j:=z(S_{j+1},\cdot)-(Qz)(S_j,\cdot),
\end{equation}
we have $e_j=\Gamma_j-\Gamma_j^++m_j\odot g_j$. Moreover,
\begin{equation}
\Gamma_j-\Gamma^+_{j-1}=z(S_j,\cdot)\odot(g_j-g_{j-1}),\qquad g_j-g_{j-1}=(\mathcal P-I)(W_j-W_{j-1})=\alpha(\mathcal P-I)P_{S_{j-1}}y_{j-1},
\end{equation}
so summation by parts gives
\begin{equation}\label{eq:mk-parts}
\sum_{j<n}\bar A^{n-1-j}e_j=\underbrace{\bar A^{n-1}\Gamma_0-\Gamma^+_{n-1}}_{T_1}+\underbrace{\sum_{j=1}^{n-1}(\bar A^{n-1-j}-\bar A^{n-j})\Gamma^+_{j-1}}_{T_2}+\alpha\underbrace{\sum_{j=1}^{n-1}\bar A^{n-1-j}X_j}_{T_3}+\underbrace{\sum_{j<n}\bar A^{n-1-j}(m_j\odot g_j)}_{T_4},
\end{equation}
with
\begin{equation}
X_j:=z(S_j,\cdot)\odot\bigl((\mathcal P-I)P_{S_{j-1}}y_{j-1}\bigr).
\end{equation}
We use the following bounds. By Lemma~\ref{lem:poisson-facts}(3), $\lvert z\rvert\le2C^\star$, so
\begin{equation}
\lVert\Gamma_j\rVert_{2,\infty},\ \lVert\Gamma^+_j\rVert_{2,\infty}\le8C^\star D\quad\text{on }\{j<\tau\},\qquad\lvert m_j(x)\rvert\le4C^\star.
\end{equation}
For $u\in\X$ and $y\in E_0(u)$, we have
\begin{equation}
(\mathcal PP_uy)(x)=\mathbb E\bigl[\mathds{1}\{S'=u\}L_Yy\mid S=x\bigr],
\end{equation}
so
\begin{equation}
\bigl\lVert\bigl((\mathcal P-I)P_uy\bigr)(x)\bigr\rVert_2\le\bigl(\beta Q(x,u)+\mathds{1}\{x=u\}\bigr)\lVert y\rVert_2,\qquad\beta Q(x,u)+\mathds{1}\{x=u\}\le2.
\end{equation}
We also use $\alpha C^\star\ell C_F\le2^{-20}(1-\beta)$, which follows from \eqref{eq:mk-step}.

We first bound $T_1$ and $T_2$. On $\{n\le\tau\}$, we have
\begin{equation}
\alpha\lVert T_1\rVert_{2,\infty}\le16\alpha C^\star D\le\frac D{64}.
\end{equation}
Next, we have $\bar A^m-\bar A^{m+1}=\bar A^m\alpha D_\mu(I-\mathcal P)$, and $Y:=\alpha D_\mu(I-\mathcal P)\Gamma^+_{j-1}$ satisfies, entrywise,
\begin{equation}
r_Y\le\alpha(1+\beta)8C^\star D\mu_{\X}\le16\alpha C^\star D\mu_{\X}.
\end{equation}
Lemma~\ref{lem:propagation}(4) gives
\begin{equation}
\lVert(\bar A^mY)(s)\rVert_2\le16\alpha C^\star D(\mathcal R^mD_\mu\bm1)(s),
\end{equation}
and summing over $m$ gives, on $\{n\le\tau\}$,
\begin{equation}
\alpha\lVert T_2(s)\rVert_2\le\frac{16\alpha C^\star D}{1-\beta}\le\frac D{64}.
\end{equation}

We now turn to $T_3$ and first bound its conditional mean. Let
\begin{equation}
X_j':=\mathds{1}\{j-1<\tau\}X_j,\qquad\bar X_j:=\mathbb E[X_j'\mid\mathcal F_{t_k+j-1}],
\end{equation}
and write $u:=S_{j-1}$. Since $y_{j-1}$ is bounded by $b+4D$ on $\{j-1<\tau\}$,
\begin{equation}
\lVert\bar X_j(x)\rVert_2\le(b+4D)\,w(u,x),\qquad w(u,x):=\bigl(\beta Q(x,u)+\mathds{1}\{x=u\}\bigr)(Q\lvert z\rvert)(u,x)\in[0,4C^\star].
\end{equation}
By Lemma~\ref{lem:propagation}(4),
\begin{equation}
\lVert(\bar A^{n-1-j}\bar X_j)(s)\rVert_2\le(b+4D)h_j(S_{j-1}),\qquad h_j(u):=\sum_x(\mathcal R^{n-1-j})_{s,x}w(u,x)\in[0,4C^\star].
\end{equation}
Stationarity and $Q(x,u)\le1$ give, by Lemma~\ref{lem:poisson-facts}(3),
\begin{equation}
\sum_u\mu_{\X}(u)w(u,x)\le\sum_{u'}\mu_{\X}(u')\lvert z(u',x)\rvert+\mu_{\X}(x)(Q\lvert z\rvert)(x,x)\le4C^\star\mu_{\X}(x),
\end{equation}
so
\begin{equation}
\sum_j\mu_{\X}(h_j)\le4C^\star\sum_{m\ge0}(\mathcal R^mD_\mu\bm1)(s)=\frac{4C^\star}{\alpha(1-\beta)},
\end{equation}
and Lemma~\ref{lem:propagation}(5) gives $V_\Lambda\le8C^\star/(1-\beta)$. Lemma~\ref{lem:markov-bernstein}, applied to $h_{j+1}$ for $0\le j\le n-2$, gives, with probability at least $1-4(\log_2L+2)e^{-\ell}$,
\begin{equation}
\sum_{j=1}^{n-1}h_j(S_{j-1})\le\frac{8C^\star}{\alpha(1-\beta)}+112C_F(C^\star)^2\ell+\frac{48(C^\star)^2}{1-\beta}\le\frac{9C^\star}{\alpha(1-\beta)},
\end{equation}
because $\alpha C^\star(112C_F\ell+48)\le1$. Hence, on $\{n\le\tau\}$,
\begin{equation}
\alpha^2\Bigl\lVert\sum_j(\bar A^{n-1-j}\bar X_j)(s)\Bigr\rVert_2\le\frac{9\alpha C^\star(b+4D)}{1-\beta}\le\frac{9\alpha C^\star b}{1-\beta}+\frac D{64}.
\end{equation}

Next, we bound the fluctuation of $T_3$ around its conditional mean. For $1\le j\le n-1$, let
\begin{equation}
d_j:=\alpha^2\bigl(\bar A^{n-1-j}(X_j'-\bar X_j)\bigr)(s),
\end{equation}
a martingale difference sequence with respect to $(\mathcal F_{t_k+j})$. Since $\lVert X_j'\rVert_{2,\infty}\le4C^\star(b+4D)$, we have $\lVert d_j\rVert_2\le8\alpha^2C^\star(b+4D)=:c$. By Lemma~\ref{lem:propagation}(2),
\begin{equation}
\mathbb E\bigl[\lVert d_j\rVert_2^2\mid\mathcal F_{t_k+j-1}\bigr]\le\alpha^4\,\mathbb E\bigl[(\mathcal V^{n-1-j}q_{X'_j})(s)\mid\mathcal F_{t_k+j-1}\bigr]\le\alpha^4(b+4D)^2h'_j(S_{j-1}),
\end{equation}
with
\begin{equation}
h'_j(u):=\sum_x(\mathcal V^{n-1-j})_{s,x}\bigl(\beta Q(x,u)+\mathds{1}\{x=u\}\bigr)^2(Qz^2)(u,x)\in\bigl[0,16(C^\star)^2\bigr].
\end{equation}
Using $(\beta Q(x,u)+\mathds{1}\{x=u\})^2\le2+2\cdot\mathds{1}\{x=u\}$, stationarity, and Lemma~\ref{lem:poisson-facts}(3)--(4), we obtain
\begin{equation}
\sum_u\mu_{\X}(u)\bigl(\beta Q(x,u)+\mathds{1}\{x=u\}\bigr)^2(Qz^2)(u,x)\le2\mu_{\X}\bigl(z(\cdot,x)^2\bigr)+2\mu_{\X}(x)\cdot4(C^\star)^2\le16(C^\star)^2\mu_{\X}(x),
\end{equation}
so
\begin{equation}
\sum_j\mu_{\X}(h'_j)\le\frac{16(C^\star)^2}{\alpha(1-\beta^2)},\qquad V_\Lambda\le\frac{32(C^\star)^2}{1-\beta^2}.
\end{equation}
Lemma~\ref{lem:markov-bernstein} gives, with probability at least $1-4(\log_2L+2)e^{-\ell}$,
\begin{equation}
\Sigma_n\le\alpha^4(b+4D)^2\Bigl(\frac{32(C^\star)^2}{\alpha(1-\beta^2)}+448C_F(C^\star)^3\ell+\frac{192(C^\star)^3}{1-\beta^2}\Bigr)\le\frac{33\alpha^3(C^\star)^2(b+4D)^2}{1-\beta^2}=:v,
\end{equation}
because $\alpha C^\star(448C_F\ell+192)\le1$. By Lemma~\ref{lem:freedman}, with probability at least $1-2e^{-\ell}$ more,
\begin{equation}
\Bigl\lVert\sum_jd_j\Bigr\rVert_2\le\alpha C^\star(b+4D)\Bigl(\sqrt{\frac{33C_F\alpha\ell}{1-\beta^2}}+8C_F\alpha\ell\Bigr)\le0.14\,\alpha C^\star(b+4D)\le\frac{0.14\,\alpha C^\star b}{1-\beta}+\frac D{64}.
\end{equation}

It remains to bound $T_4$. Let
\begin{equation}
d_j:=\alpha\mathds{1}\{j<\tau\}\bigl(\bar A^{n-1-j}(m_j\odot g_j)\bigr)(s),
\end{equation}
a martingale difference sequence because $\mathbb E[m_j\mid\mathcal F_{t_k+j}]=0$ and $g_j$ is $\mathcal F_{t_k+j}$-measurable. Here, $\lVert d_j\rVert_2\le16\alpha C^\star D=:c$. By Lemma~\ref{lem:propagation}(2),
\begin{equation}
\mathbb E\bigl[\lVert d_j\rVert_2^2\mid\mathcal F_{t_k+j}\bigr]\le16\alpha^2D^2h_j(S_j),\qquad h_j(u):=\sum_x(\mathcal V^{n-1-j})_{s,x}v_x(u)\in\bigl[0,4(C^\star)^2\bigr].
\end{equation}
Lemma~\ref{lem:poisson-facts}(4) and Lemma~\ref{lem:propagation}(5) give
\begin{equation}
\sum_j\mu_{\X}(h_j)\le\frac{4C^\star}{\alpha(1-\beta^2)},\qquad V_\Lambda\le\frac{8(C^\star)^2}{1-\beta^2}.
\end{equation}
Lemma~\ref{lem:markov-bernstein} gives, with probability at least $1-4(\log_2L+2)e^{-\ell}$,
\begin{equation}
\Sigma_n\le16D^2\Bigl(\frac{8\alpha C^\star}{1-\beta^2}+112C_F\alpha^2(C^\star)^3\ell+\frac{48\alpha^2(C^\star)^3}{1-\beta^2}\Bigr)=:v.
\end{equation}
By \eqref{eq:mk-step},
\begin{equation}
C_Fv\ell\le16D^2\bigl(2^{-17}+112\cdot2^{-24}+48\cdot2^{-24}\bigr)\le2^{-10}D^2,
\end{equation}
so Lemma~\ref{lem:freedman} gives, with probability at least $1-2e^{-\ell}$ more,
\begin{equation}
\Bigl\lVert\sum_jd_j\Bigr\rVert_2\le\frac D{32}+16C_F\alpha C^\star\ell D\le\frac D{32}+\frac D{64}.
\end{equation}

We now close the epoch. Collect the events above for all $n\le L$ and $s\in\X$. Their probability is at least $1-16L\lvert\X\rvert(\log_2L+3)e^{-\ell}$. On this event, for every $n\le\min\{\tau,L\}$, the sum of the last two terms in \eqref{eq:sharp-unroll} has $\lVert\cdot\rVert_{2,\infty}$-norm at most $D/4+N$ in the i.i.d.\ regime. In the Markovian regime, the bounds above give at most
\begin{equation}
\frac D8+\frac D{64}+\frac D{64}+\frac D{64}+\frac D{64}+\frac D{64}+\frac{3D}{64}+\sqrt{6C_F\alpha\mathrm{TV}_{\max}\ell}+C_F\alpha b\ell+\frac{9.14\,\alpha C^\star b}{1-\beta}\le\frac D4+N.
\end{equation}
On $\{\lVert W_{t_k}\rVert_{2,\infty}\le D\}$, we have $\tau\ge1$. If $\tau\le L$, then \eqref{eq:sharp-unroll} at $n=\tau$ and Lemma~\ref{lem:propagation}(1) give
\begin{equation}
\lVert W_{t_k+\tau}\rVert_{2,\infty}\le D+\frac D4+N\le2D,
\end{equation}
because $N\le D/2$. This contradicts the definition of $\tau$, so $\tau>L$. At $n=L$, Lemma~\ref{lem:propagation}(1) and the choice of $L$ give $\lVert\bar A^LW_{t_k}\rVert_{2,\infty}\le D/8$, hence
\begin{equation}
\lVert W_{t_{k+1}}\rVert_{2,\infty}\le\frac D8+\frac{D}{4}+N\le\frac D2+N=D_{k+1}.
\end{equation}
This completes the epoch step and the proof.
\end{proof}

\subsection{Markovian sampling in reverse time}
\label{app:sharp-rev}

This subsection proves the Markovian part of Theorem~\ref{thm:sharp} by adapting the reverse-time argument that \citet{peng2026sharp} develops for scalar tabular TD to the operator-valued recursions of Assumption~\ref{ass:sharp}. Three of its ingredients are used here, namely, reading the trajectory backward from the terminal time, a Poisson equation of the time-reversed chain anchored at the source state of each transition, and a killed-particle representation of the remaining initial error. What has to be added for the distributional recursions is that the weights are now operators, dominated by Peng's scalar weights (Lemma~\ref{lem:rev-dom}), and that the accumulated conditional variance must be charged to the discounted total variance (Proposition~\ref{prop:rev-energy}).

Throughout this subsection, the sampling is Markovian, that is, given the past,
\begin{equation}
A_t\sim\pi(\cdot\mid S_t),\qquad S_{t+1}\sim P(\cdot\mid S_t,A_t),\qquad Y_t=(S_t,A_t,S_{t+1}),
\end{equation}
and $Q=P^\pi$ is irreducible with stationary distribution $\mu:=\mu_{\X}$. The law of $S_0$ is arbitrary unless stated otherwise, and $\hat Q(x,u)=\mu(u)Q(u,x)/\mu(x)$ is the time reversal of Appendix~\ref{app:common}.

We first derive an exact recursion for the error $W_t$. Since $\alpha\le1$, Lemma~\ref{lem:invariant-domain} and the samplewise preservation verified in Appendices~\ref{app:ctd} and~\ref{app:mtd} give $U_t\in\mathcal D$ for all $t$. The same preservation, applied at $U^*$, shows that $\widehat T(U^*;y)$ lies in the $s$th block of $\mathcal D$, so $\xi(y)\in E_0(s)$ for every sample $y=(s,a,s')$. Assumption~\ref{ass:sharp}(1) then gives
\begin{equation}
\widehat T(U_t;Y_t)-U^*(S_t)=L_{Y_t}W_t(S_{t+1})+\xi(Y_t),
\end{equation}
and therefore
\begin{equation}\label{eq:rev-exact}
W_{t+1}=\hat A_tW_t+\alpha P_{S_t}\xi(Y_t),\qquad\hat A_tW:=W+\alpha P_{S_t}\bigl(L_{Y_t}W(S_{t+1})-W(S_t)\bigr)\quad(W\in E_0).
\end{equation}
With $\Phi_{t,s}:=\hat A_{t-1}\cdots\hat A_s$ for $t>s$ and $\Phi_{s,s}:=I$, unrolling gives
\begin{equation}\label{eq:rev-unroll}
W_T=\Phi_{T,0}W_0+\alpha\sum_{t<T}\Phi_{T,t+1}P_{S_t}\xi(Y_t).
\end{equation}
Unlike \eqref{eq:sharp-unroll}, this expansion is around the random maps $\hat A_t$, so neither the averaged operator nor the drift term $e_t$ appears.

We now read the trajectory backward from the terminal time. Fix $T\ge1$, and let
\begin{equation}
X_n:=S_{T-n}\quad(0\le n\le T),\qquad\mathsf A_{n+1}:=A_{T-n-1}\quad(n<T),\qquad\mathcal G_n:=\sigma\bigl(X_0,\mathsf A_1,X_1,\dots,\mathsf A_n,X_n\bigr).
\end{equation}
The $(n+1)$st reverse step traverses the forward transition $(X_{n+1},\mathsf A_{n+1},X_n)$.

\begin{lemma}[Reverse law]
\label{lem:rev-law}
Suppose $S_0\sim\mu$. Then, for $0\le n<T$, $u\in\X$, and $a\in\A$,
\begin{equation}
\Pr\bigl(X_{n+1}=u,\ \mathsf A_{n+1}=a\mid\mathcal G_n\bigr)=\frac{\mu(u)\pi(a\mid u)P(X_n\mid u,a)}{\mu(X_n)}.
\end{equation}
In particular, $(X_n)$ is a Markov chain with kernel $\hat Q$ and $X_0\sim\mu$, and, given $\mathcal G_n$, $X_n=x$, and $X_{n+1}=u$, the action $\mathsf A_{n+1}$ has the conditional law of $A$ given $S=u$ and $S'=x$ under forward sampling.
\end{lemma}

\begin{proof}
Let $t:=T-n-1$. Then $\mathcal G_n$ is generated by $S_{t+1}$ and the future $(A_{t+1},S_{t+2},\dots,A_{T-1},S_T)$. By the Markov property of $(S_k,A_k)$, this future is conditionally independent of $(S_t,A_t)$ given $S_{t+1}$, and, by stationarity,
\begin{equation}
\Pr(S_t=u,A_t=a,S_{t+1}=x)=\mu(u)\pi(a\mid u)P(x\mid u,a).
\end{equation}
Bayes' rule gives the display, and summing over $a$ gives $\Pr(X_{n+1}=u\mid\mathcal G_n)=\hat Q(X_n,u)$.
\end{proof}

For the uniform statements below, we also use the \emph{reverse process from} $x\in\X$: $X_0=x$ and, given $\mathcal G_n$, the pair $(X_{n+1},\mathsf A_{n+1})$ is drawn from the law of Lemma~\ref{lem:rev-law}. By that lemma, if $S_0\sim\mu$, then, conditionally on $X_0=x$, the reversed trajectory is the reverse process from $x$. We write $\mathbb E_x$ and $\Pr_x$ for this process.

We next introduce the weights that the reversed recursion assigns to the blocks. Fix $j\in\X$. For $0\le n\le T$, define $\Upsilon_n:E_0\to E_0(j)$ and its blocks $\Upsilon_n^{(i)}:E_0(i)\to E_0(j)$ by
\begin{equation}
\Upsilon_nW:=(\Phi_{T,T-n}W)(j)=\sum_i\Upsilon_n^{(i)}W(i).
\end{equation}
Then $\Upsilon_0^{(i)}=\mathds{1}\{i=j\}I$ and $\Upsilon_{n+1}=\Upsilon_n\hat A_{T-n-1}$. Write
\begin{equation}
u:=X_{n+1},\qquad x:=X_n,\qquad L_{n+1}:=L_{(X_{n+1},\mathsf A_{n+1},X_n)}:E_0(x)\to E_0(u).
\end{equation}
Since $\hat A_{T-n-1}$ changes only block $u$, the block form of \eqref{eq:rev-exact} gives, for $x\ne u$,
\begin{equation}\label{eq:rev-Pi1}
\Upsilon_{n+1}^{(u)}=(1-\alpha)\Upsilon_n^{(u)},\qquad\Upsilon_{n+1}^{(x)}=\Upsilon_n^{(x)}+\alpha\Upsilon_n^{(u)}L_{n+1},\qquad\Upsilon_{n+1}^{(i)}=\Upsilon_n^{(i)}\quad(i\ne x,u),
\end{equation}
and, for $x=u$,
\begin{equation}\label{eq:rev-Pi2}
\Upsilon_{n+1}^{(u)}=\Upsilon_n^{(u)}\bigl((1-\alpha)I+\alpha L_{n+1}\bigr),\qquad\Upsilon_{n+1}^{(i)}=\Upsilon_n^{(i)}\quad(i\ne u).
\end{equation}
Define scalar weights $\hat p_n\in\mathbb R^{\X}$ by $\hat p_0:=\bm1_j$ and
\begin{equation}\label{eq:rev-p}
\hat p_{n+1}:=\hat p_n+\alpha\hat p_n(u)\bigl(\beta\bm1_x-\bm1_u\bigr).
\end{equation}
This is the weight recursion of \citet{peng2026sharp} for scalar TD, with his step size and discount factor replaced by $\alpha$ and $\beta$.

\begin{lemma}[Domination and mass budget]
\label{lem:rev-dom}
Let $(\Upsilon_n)$ and $(\hat p_n)$ obey \eqref{eq:rev-Pi1}, \eqref{eq:rev-Pi2}, and \eqref{eq:rev-p} along any path, from any $\hat p_0\ge0$ with $\lVert\hat p_0\rVert_1\le1$ and any $\Upsilon_0$ with $\lVert\Upsilon_0^{(i)}\rVert\le\hat p_0(i)$ for all $i$. Then, for every $n$, $\hat p_n\ge0$, $\lVert\Upsilon_n^{(i)}\rVert\le\hat p_n(i)$ for all $i$, and
\begin{equation}\label{eq:rev-mass}
\lVert\hat p_{n+1}\rVert_1=\lVert\hat p_n\rVert_1-\alpha(1-\beta)\hat p_n(X_{n+1}),\qquad\alpha\sum_{n<N}\hat p_n(X_{n+1})\le\frac{1}{1-\beta}\quad\text{for all }N.
\end{equation}
In particular, for $V\in E_0$,
\begin{equation}
\lVert\Upsilon_nV\rVert_2\le\sum_i\hat p_n(i)\lVert V(i)\rVert_2\le\max_i\lVert V(i)\rVert_2.
\end{equation}
Moreover, $\hat p_n$ is a function of $\hat p_0$ and $X_0,\dots,X_n$, and $\Upsilon_n$ is $\mathcal G_n$-measurable when $\Upsilon_0$ is $\mathcal G_0$-measurable.
\end{lemma}

\begin{proof}
If $x\ne u$, \eqref{eq:rev-p} multiplies $\hat p_n(u)$ by $1-\alpha\ge0$ and adds $\alpha\beta\hat p_n(u)\ge0$ to $\hat p_n(x)$. If $x=u$, it multiplies $\hat p_n(u)$ by $1-\alpha(1-\beta)\ge0$. This gives nonnegativity. We prove the domination by induction, using $\lVert L_{n+1}\rVert\le\beta$. If $x\ne u$, then
\begin{equation}
\lVert\Upsilon_{n+1}^{(u)}\rVert\le(1-\alpha)\hat p_n(u)=\hat p_{n+1}(u),\qquad\lVert\Upsilon_{n+1}^{(x)}\rVert\le\hat p_n(x)+\alpha\beta\hat p_n(u)=\hat p_{n+1}(x),
\end{equation}
and if $x=u$, then $\lVert(1-\alpha)I+\alpha L_{n+1}\rVert\le1-\alpha(1-\beta)$. Summing \eqref{eq:rev-p} over coordinates gives the mass identity, and telescoping it with $\lVert\hat p_N\rVert_1\ge0$ gives the budget. The measurability statements follow from the recursions, since $L_{n+1}$ depends on $\mathsf A_{n+1}$, which is $\mathcal G_{n+1}$-measurable.
\end{proof}

By \eqref{eq:rev-unroll} with $t=T-n-1$, block $j$ of $W_T$ splits as
\begin{equation}\label{eq:rev-split}
W_T(j)=\Upsilon_TW_0+\Xi_T(j),\qquad\Xi_T(j):=\alpha\sum_{n<T}\Upsilon_n^{(X_{n+1})}\xi_{n+1},\qquad\xi_{n+1}:=\xi(X_{n+1},\mathsf A_{n+1},X_n),
\end{equation}
and Lemma~\ref{lem:rev-dom} gives
\begin{equation}\label{eq:rev-init}
\lVert\Upsilon_TW_0\rVert_2\le\lVert\hat p_T\rVert_1\lVert W_0\rVert_{2,\infty}.
\end{equation}

We now solve a Poisson equation of the time-reversed chain, anchored at the source state of each transition. Let $h(u,a,x)\in E_0(u)$ be an observable of a forward transition with $\lVert h\rVert_2\le b_h$, centered at every source state:
\begin{equation}\label{eq:rev-centered}
\sum_{a,x}\pi(a\mid u)P(x\mid u,a)\,h(u,a,x)=0\qquad\text{for every }u\in\X.
\end{equation}
For $i,x\in\X$, let
\begin{equation}
\bar h_i(x):=\begin{cases}\mathbb E[h(i,A,S')\mid S=i,S'=x]&\text{if }Q(i,x)>0,\\0&\text{otherwise,}\end{cases}\qquad g_i(x):=\hat Q(x,i)\bar h_i(x),
\end{equation}
so that $\bar h_i(x),g_i(x)\in E_0(i)$.

\begin{lemma}[Anchored Poisson solution]
\label{lem:rev-anchor}
For each $i\in\X$, there exists $\Psi_i:\X\to E_0(i)$ with
\begin{equation}
\Psi_i-\hat Q\Psi_i=g_i,\qquad\Psi_i(i)=0,\qquad\lVert\Psi_i(x)\rVert_2\le b_h\quad\text{for all }x\in\X.
\end{equation}
\end{lemma}

\begin{proof}
For the chain with kernel $\hat Q$ started at $x$, let $\tau_i^+:=\inf\{k\ge1:X_k=i\}$, which is finite almost surely by irreducibility. Set $\Psi_i(i):=0$ and, for $x\ne i$,
\begin{equation}
\Psi_i(x):=\mathbb E_x\bigl[\bar h_i(X_{\tau_i^+-1})\bigr].
\end{equation}
The integrand is a value $\bar h_i(z)$ with $\hat Q(z,i)>0$, hence $Q(i,z)>0$, and each such value is an average of values of $h$, so $\lVert\Psi_i(x)\rVert_2\le b_h$. For $x\ne i$, conditioning on the first step gives
\begin{equation}
\Psi_i(x)=\hat Q(x,i)\bar h_i(x)+\sum_{u\ne i}\hat Q(x,u)\Psi_i(u)=g_i(x)+(\hat Q\Psi_i)(x).
\end{equation}
It remains to check the equation at $x=i$. Since $\mu$ is stationary for $\hat Q$, we have
\begin{equation}
\sum_x\mu(x)(\Psi_i-\hat Q\Psi_i)(x)=0,
\end{equation}
and \eqref{eq:rev-centered} gives
\begin{equation}
\sum_x\mu(x)g_i(x)=\mu(i)\sum_xQ(i,x)\bar h_i(x)=\mu(i)\,\mathbb E[h(i,A,S')\mid S=i]=0.
\end{equation}
The two sides of the equation agree at every $x\ne i$ and have the same $\mu$-weighted sum, and $\mu(i)>0$, so they agree at $i$.
\end{proof}

Write
\begin{equation}
\Psi(x):=(\Psi_i(x))_{i\in\X}\in E_0,\qquad h_{n+1}:=h(X_{n+1},\mathsf A_{n+1},X_n),
\end{equation}
and, with $u=X_{n+1}$ and $x=X_n$, define
\begin{equation}
m_{n+1}:=\bm1_u\odot h_{n+1}-\bigl[\Psi(x)-\Psi(u)\bigr]\in E_0,\qquad\text{that is,}\qquad m_{n+1}(i)=\mathds{1}\{u=i\}h_{n+1}-\Psi_i(x)+\Psi_i(u).
\end{equation}

\begin{lemma}[Operator-weighted Poisson decomposition]
\label{lem:rev-poisson}
Consider the reverse process from any $x_0\in\X$, or the reversed trajectory with $S_0\sim\mu$, and weights as in Lemma~\ref{lem:rev-dom} with $\Upsilon_0$ and $\hat p_0$ deterministic. Then, for every $n$ and $i$,
\begin{equation}
\mathbb E[m_{n+1}\mid\mathcal G_n]=0,\qquad\lVert m_{n+1}(i)\rVert_2\le2b_h,
\end{equation}
and, for every $N\ge0$,
\begin{equation}\label{eq:rev-A-identity}
\alpha\sum_{n<N}\Upsilon_n^{(X_{n+1})}h_{n+1}=\alpha\sum_{n<N}\Upsilon_nm_{n+1}+\alpha\bigl[\Upsilon_0\Psi(X_0)-\Upsilon_N\Psi(X_N)\bigr]+\alpha^2\sum_{n<N}\Upsilon_n^{(X_{n+1})}L_{n+1}\Psi_{X_n}(X_{n+1}).
\end{equation}
Consequently, along every path,
\begin{equation}\label{eq:rev-A-bound}
\Bigl\lVert\alpha\sum_{n<N}\Upsilon_n^{(X_{n+1})}h_{n+1}-\alpha\sum_{n<N}\Upsilon_nm_{n+1}\Bigr\rVert_2\le\alpha b_h\Bigl(2+\frac{\beta}{1-\beta}\Bigr)\le\frac{2\alpha b_h}{1-\beta}.
\end{equation}
The same holds with scalar blocks $E_0(i)$ replaced by $\mathbb R$, $L_{n+1}$ replaced by $\beta$, and $\Upsilon_n$ replaced by $\hat p_n^{\mathsf T}$.
\end{lemma}

\begin{proof}
We first prove the martingale property. Condition on $\mathcal G_n$ with $X_n=x$. By Lemma~\ref{lem:rev-law},
\begin{equation}
\mathbb E\bigl[\mathds{1}\{X_{n+1}=i\}h_{n+1}\mid\mathcal G_n\bigr]=\hat Q(x,i)\bar h_i(x)=g_i(x),\qquad\mathbb E\bigl[\Psi_i(x)-\Psi_i(X_{n+1})\mid\mathcal G_n\bigr]=(\Psi_i-\hat Q\Psi_i)(x)=g_i(x),
\end{equation}
so $\mathbb E[m_{n+1}(i)\mid\mathcal G_n]=0$.

Next, we bound the increments. If $u=i$, then $\Psi_i(u)=0$ and $m_{n+1}(i)=h_{n+1}-\Psi_i(x)$, a difference of two vectors of norm at most $b_h$. If $u\ne i$, then $m_{n+1}(i)=\Psi_i(u)-\Psi_i(x)$. In both cases, $\lVert m_{n+1}(i)\rVert_2\le2b_h$.

We now prove \eqref{eq:rev-A-identity}. By the definition of $m_{n+1}$,
\begin{equation}
\Upsilon_n^{(u)}h_{n+1}=\Upsilon_n(\bm1_u\odot h_{n+1})=\Upsilon_nm_{n+1}+\Upsilon_n\Psi(X_n)-\Upsilon_n\Psi(X_{n+1}),
\end{equation}
and summation by parts gives
\begin{equation}
\sum_{n<N}\bigl[\Upsilon_n\Psi(X_n)-\Upsilon_n\Psi(X_{n+1})\bigr]=\Upsilon_0\Psi(X_0)-\Upsilon_N\Psi(X_N)+\sum_{n<N}(\Upsilon_{n+1}-\Upsilon_n)\Psi(X_{n+1}).
\end{equation}
If $x\ne u$, \eqref{eq:rev-Pi1} and the anchor $\Psi_u(u)=0$ give
\begin{equation}
(\Upsilon_{n+1}-\Upsilon_n)\Psi(u)=-\alpha\Upsilon_n^{(u)}\Psi_u(u)+\alpha\Upsilon_n^{(u)}L_{n+1}\Psi_x(u)=\alpha\Upsilon_n^{(u)}L_{n+1}\Psi_x(u).
\end{equation}
If $x=u$, \eqref{eq:rev-Pi2} gives
\begin{equation}
(\Upsilon_{n+1}-\Upsilon_n)\Psi(u)=\alpha\Upsilon_n^{(u)}(L_{n+1}-I)\Psi_u(u)=0=\alpha\Upsilon_n^{(u)}L_{n+1}\Psi_x(u).
\end{equation}
This proves \eqref{eq:rev-A-identity}.

Finally, we prove \eqref{eq:rev-A-bound}. By Lemma~\ref{lem:rev-dom}, the boundary term has norm at most $2\alpha b_h$, and the last term has norm at most
\begin{equation}
\alpha^2\beta b_h\sum_{n<N}\hat p_n(X_{n+1})\le\frac{\alpha\beta b_h}{1-\beta}
\end{equation}
by \eqref{eq:rev-mass}. It remains to note that $2+\beta/(1-\beta)=(2-\beta)/(1-\beta)$. The scalar case is the same computation, since $\hat p_n^{\mathsf T}$ obeys \eqref{eq:rev-Pi1} and \eqref{eq:rev-Pi2} with $L_{n+1}=\beta$.
\end{proof}

In Lemma~\ref{lem:rev-poisson}, the Poisson solution is the expected innovation on a single transition, rather than a sum along a hitting time, so the correction \eqref{eq:rev-A-bound} contains no Poisson or mixing constant.

We now bound the energy of the martingale part. From now on, we apply Lemma~\ref{lem:rev-poisson} with $h=\xi$, which is centered at every source state by Assumption~\ref{ass:sharp}(2), and $b_h=b$. Let $\Psi$ and $m_{n+1}$ be the corresponding objects, and define
\begin{equation}
\omega_i(x):=\mathbb E\bigl[\lVert m_{n+1}(i)\rVert_2^2\bigm|\mathcal G_n\bigr]\quad\text{on }\{X_n=x\},
\end{equation}
which, by Lemma~\ref{lem:rev-law}, depends only on $i$ and $x$. Write $\omega(x):=(\omega_i(x))_{i\in\X}$.

\begin{lemma}[Conditional variance]
\label{lem:rev-cv}
Under the conditions of Lemma~\ref{lem:rev-poisson} with $h=\xi$,
\begin{equation}
\mathbb E\bigl[\lVert\Upsilon_nm_{n+1}\rVert_2^2\bigm|\mathcal G_n\bigr]\le\hat p_n^{\mathsf T}\omega(X_n).
\end{equation}
\end{lemma}

\begin{proof}
By Lemma~\ref{lem:rev-dom} and the Cauchy--Schwarz inequality with $\lVert\hat p_n\rVert_1\le1$,
\begin{equation}
\lVert\Upsilon_nm_{n+1}\rVert_2^2\le\Bigl(\sum_i\hat p_n(i)\lVert m_{n+1}(i)\rVert_2\Bigr)^2\le\sum_i\hat p_n(i)\lVert m_{n+1}(i)\rVert_2^2.
\end{equation}
Taking conditional expectations, with $\hat p_n$ being $\mathcal G_n$-measurable, gives the claim.
\end{proof}

\begin{lemma}[Variance compensation]
\label{lem:rev-comp}
For all $i,x\in\X$,
\begin{equation}\label{eq:rev-omega}
\omega_i(x)=\mathbb E\bigl[\mathds{1}\{X_{n+1}=i\}\lVert\xi_{n+1}\rVert_2^2\bigm|X_n=x\bigr]+\bigl(\hat Q\lVert\Psi_i\rVert_2^2\bigr)(x)-\lVert\Psi_i(x)\rVert_2^2.
\end{equation}
Let
\begin{equation}
F_n:=\sum_i\hat p_n(i)\lVert\Psi_i(X_n)\rVert_2^2,
\end{equation}
so that $0\le F_n\le b^2$. Then
\begin{equation}\label{eq:rev-comp}
\hat p_n^{\mathsf T}\omega(X_n)\le\mathbb E\bigl[\hat p_n(X_{n+1})\lVert\xi_{n+1}\rVert_2^2+F_{n+1}-F_n\bigm|\mathcal G_n\bigr].
\end{equation}
\end{lemma}

\begin{proof}
Fix $i$, condition on $X_n=x$, and let $\mathsf a:=\mathds{1}\{X_{n+1}=i\}\xi_{n+1}\in E_0(i)$. Then
\begin{equation}
m_{n+1}(i)=\mathsf a-\Psi_i(x)+\Psi_i(X_{n+1}),\qquad\langle\mathsf a,\Psi_i(X_{n+1})\rangle=0,
\end{equation}
where the second identity holds pathwise because either the indicator vanishes or $X_{n+1}=i$ and $\Psi_i(i)=0$. Since
\begin{equation}
\mathbb E[\mathsf a\mid X_n=x]=g_i(x)=\Psi_i(x)-(\hat Q\Psi_i)(x),
\end{equation}
expanding the square gives
\begin{equation}
\omega_i(x)=\mathbb E\lVert\mathsf a\rVert_2^2+\lVert\Psi_i(x)\rVert_2^2+\bigl(\hat Q\lVert\Psi_i\rVert_2^2\bigr)(x)-2\bigl\langle\Psi_i(x),g_i(x)\bigr\rangle-2\bigl\langle\Psi_i(x),(\hat Q\Psi_i)(x)\bigr\rangle,
\end{equation}
and the last two terms sum to $-2\lVert\Psi_i(x)\rVert_2^2$. This is \eqref{eq:rev-omega}. For \eqref{eq:rev-comp}, with $x=X_n$ and $u=X_{n+1}$, \eqref{eq:rev-p} gives the pathwise identity
\begin{equation}
F_{n+1}-F_n=\sum_i\hat p_n(i)\bigl(\lVert\Psi_i(u)\rVert_2^2-\lVert\Psi_i(x)\rVert_2^2\bigr)+\alpha\hat p_n(u)\bigl(\beta\lVert\Psi_x(u)\rVert_2^2-\lVert\Psi_u(u)\rVert_2^2\bigr),
\end{equation}
in which the last term equals $\alpha\beta\hat p_n(u)\lVert\Psi_x(u)\rVert_2^2\ge0$ because $\Psi_u(u)=0$. Multiplying \eqref{eq:rev-omega} by $\hat p_n(i)$ and summing over $i$, we obtain
\begin{equation}
\hat p_n^{\mathsf T}\omega(x)=\mathbb E\bigl[\hat p_n(X_{n+1})\lVert\xi_{n+1}\rVert_2^2\bigm|\mathcal G_n\bigr]+\mathbb E\bigl[F_{n+1}-F_n\bigm|\mathcal G_n\bigr]-\alpha\beta\,\mathbb E\bigl[\hat p_n(u)\lVert\Psi_x(u)\rVert_2^2\bigm|\mathcal G_n\bigr],
\end{equation}
and dropping the last term gives \eqref{eq:rev-comp}. Finally, $0\le F_n\le\lVert\hat p_n\rVert_1b^2\le b^2$.
\end{proof}

Let
\begin{equation}
\mathrm u:=\mathrm{TV}=(I-\beta^2Q)^{-1}\sigma^2,\qquad\text{so that}\qquad\mathrm u=\sigma^2+\beta^2Q\mathrm u,\qquad0\le\mathrm u\le\mathrm{TV}_{\max}.
\end{equation}

\begin{lemma}[Weighted noise energy]
\label{lem:rev-wenergy}
Assume $\alpha\le1-\beta$. For the reverse process from any $x_0\in\X$, any deterministic $\hat p_0\ge0$ with $\lVert\hat p_0\rVert_1\le1$, and every $N$,
\begin{equation}
\mathbb E\Bigl[\alpha\sum_{n<N}\hat p_n(X_{n+1})\lVert\xi_{n+1}\rVert_2^2\Bigr]\le(1+\beta)\mathrm{TV}_{\max}+\frac{2\alpha(b^2+\mathrm{TV}_{\max})}{1-\beta}\le4\,\mathrm{TV}_{\max}+2b^2.
\end{equation}
\end{lemma}

\begin{proof}
The scalar observable
\begin{equation}
\psi(u,a,x):=\lVert\xi(u,a,x)\rVert_2^2+\beta^2\mathrm u(x)-\mathrm u(u)
\end{equation}
is centered at every source state, because
\begin{equation}
\mathbb E\bigl[\lVert\xi(Y)\rVert_2^2\mid S=u\bigr]=\sigma^2(u),\qquad\mathbb E\bigl[\mathrm u(S')\mid S=u\bigr]=(Q\mathrm u)(u).
\end{equation}
Its values lie in the interval $[-\mathrm u(u),\,b^2+\beta^2\mathrm{TV}_{\max}-\mathrm u(u)]$, which contains $0$, so $\lvert\psi\rvert\le b^2+\mathrm{TV}_{\max}$. The scalar case of Lemma~\ref{lem:rev-poisson}, whose martingale term has mean zero, gives
\begin{equation}\label{eq:rev-psi}
\Bigl\lvert\mathbb E\Bigl[\alpha\sum_{n<N}\hat p_n(X_{n+1})\psi_{n+1}\Bigr]\Bigr\rvert\le\frac{2\alpha(b^2+\mathrm{TV}_{\max})}{1-\beta},\qquad\psi_{n+1}:=\psi(X_{n+1},\mathsf A_{n+1},X_n).
\end{equation}
Telescoping $\hat p_n^{\mathsf T}\mathrm u$ along \eqref{eq:rev-p} gives
\begin{equation}
\hat p_0^{\mathsf T}\mathrm u-\hat p_N^{\mathsf T}\mathrm u=\alpha\sum_{n<N}\hat p_n(X_{n+1})\bigl(\mathrm u(X_{n+1})-\beta\mathrm u(X_n)\bigr),
\end{equation}
and adding $\alpha\sum_{n<N}\hat p_n(X_{n+1})\psi_{n+1}$ to both sides gives the pathwise identity
\begin{equation}
\alpha\sum_{n<N}\hat p_n(X_{n+1})\lVert\xi_{n+1}\rVert_2^2=\hat p_0^{\mathsf T}\mathrm u-\hat p_N^{\mathsf T}\mathrm u+\alpha\sum_{n<N}\hat p_n(X_{n+1})\psi_{n+1}+\alpha\beta(1-\beta)\sum_{n<N}\hat p_n(X_{n+1})\mathrm u(X_n).
\end{equation}
The first two terms are at most $\mathrm{TV}_{\max}$, and, by \eqref{eq:rev-mass}, the last is at most $\beta\,\mathrm{TV}_{\max}$. Taking expectations and using \eqref{eq:rev-psi} proves the first bound, and $\alpha\le1-\beta$ gives the second.
\end{proof}

\begin{proposition}[Uniform energy bound]
\label{prop:rev-energy}
Assume $\alpha\le1-\beta$ and let
\begin{equation}
\mathcal E:=4\,\mathrm{TV}_{\max}+3b^2.
\end{equation}
For the reverse process from any $x_0\in\X$, any deterministic $\hat p_0\ge0$ with $\lVert\hat p_0\rVert_1\le1$, and every $N$,
\begin{equation}
\mathbb E\Bigl[\alpha\sum_{n<N}\hat p_n^{\mathsf T}\omega(X_n)\Bigr]\le\mathcal E.
\end{equation}
\end{proposition}

\begin{proof}
We multiply \eqref{eq:rev-comp} by $\alpha$, sum over $n<N$, and take expectations. Lemma~\ref{lem:rev-wenergy} bounds the noise part by $4\,\mathrm{TV}_{\max}+2b^2$, and the potential part is $\alpha\,\mathbb E[F_N-F_0]\le\alpha b^2\le b^2$.
\end{proof}

We next turn the energy bound into an exponential tail bound. The pair $\mathsf Z_n:=(\hat p_n,X_n)$ of the reverse process is a time-homogeneous Markov chain on
\begin{equation}
\mathcal K:=\bigl\{p\in\mathbb R^{\X}:p\ge0,\ \lVert p\rVert_1\le1\bigr\}\times\X,
\end{equation}
because $X_{n+1}$ is drawn from $\hat Q(X_n,\cdot)$ independently of the past and $\hat p_{n+1}$ is a function of $(\hat p_n,X_n,X_{n+1})$. By Lemma~\ref{lem:rev-dom}, $\mathcal K$ is invariant. Proposition~\ref{prop:rev-energy} bounds the expected additive functional of $f(p,x):=\alpha p^{\mathsf T}\omega(x)\ge0$ uniformly over starting points in $\mathcal K$ and horizons. The following discrete Khasminskii bound is Lemma~B.1 of \citet{peng2026sharp}, and we include its short proof.

\begin{lemma}[Exponential moments from expected energy]
\label{lem:rev-khas}
Let $(\mathsf Z_n)_{n\ge0}$ be a time-homogeneous Markov chain on a measurable space, let $f\ge0$ be measurable, and suppose that
\begin{equation}
\sup_{\mathsf z}\sup_{N\ge1}\mathbb E_{\mathsf z}\Bigl[\sum_{n<N}f(\mathsf Z_n)\Bigr]\le\mathcal E<\infty.
\end{equation}
Then, with $A_N:=\sum_{n<N}f(\mathsf Z_n)$, for every starting point $\mathsf z$, every $N$, and every $\ell>0$,
\begin{equation}
\Pr_{\mathsf z}\bigl(A_N>2\mathcal E\ell\bigr)\le2e^{-\ell}.
\end{equation}
\end{lemma}

\begin{proof}
We show $\mathbb E_{\mathsf z}A_N^k\le k!\,\mathcal E^k$ for all $k\ge1$, $\mathsf z$, and $N$, by induction on $k$. The case $k=1$ is the assumption. Let
\begin{equation}
B_n:=\sum_{r=n}^{N-1}f(\mathsf Z_r),
\end{equation}
so that $B_0=A_N$, $B_N=0$, and $B_n\ge B_{n+1}\ge0$. Then
\begin{equation}
B_n^k-B_{n+1}^k\le k\bigl(B_n-B_{n+1}\bigr)B_n^{k-1}=kf(\mathsf Z_n)B_n^{k-1}.
\end{equation}
By the Markov property, $B_n$ given $\mathsf Z_0,\dots,\mathsf Z_n$ is distributed as $A_{N-n}$ started from $\mathsf Z_n$, so the induction hypothesis gives
\begin{equation}
\mathbb E\bigl[B_n^{k-1}\mid\mathsf Z_0,\dots,\mathsf Z_n\bigr]\le(k-1)!\,\mathcal E^{k-1}.
\end{equation}
Hence
\begin{equation}
\mathbb E_{\mathsf z}A_N^k=\sum_{n<N}\mathbb E_{\mathsf z}\bigl[B_n^k-B_{n+1}^k\bigr]\le k!\,\mathcal E^{k-1}\,\mathbb E_{\mathsf z}\Bigl[\sum_{n<N}f(\mathsf Z_n)\Bigr]\le k!\,\mathcal E^k.
\end{equation}
Therefore,
\begin{equation}
\mathbb E_{\mathsf z}\exp\Bigl(\frac{A_N}{2\mathcal E}\Bigr)=\sum_{k\ge0}\frac{\mathbb E_{\mathsf z}A_N^k}{k!\,(2\mathcal E)^k}\le\sum_{k\ge0}2^{-k}=2,
\end{equation}
and Markov's inequality gives the tail bound.
\end{proof}

\begin{proposition}[Noise in reverse time]
\label{prop:rev-noise}
Assume $\alpha\le1-\beta$, let $\ell\ge1$ and $T\ge1$, and let
\begin{equation}
N^{\mathrm{mk}}:=\ell\sqrt{2C_F\alpha\mathcal E}+2C_F\alpha b\ell+\frac{2\alpha b}{1-\beta}.
\end{equation}
Under any law of $S_0$, with probability at least $1-4\lvert\X\rvert e^{-\ell}/\mu_{\min}$,
\begin{equation}
\lVert\Xi_T(j)\rVert_2\le N^{\mathrm{mk}}\qquad\text{for every }j\in\X.
\end{equation}
\end{proposition}

\begin{proof}
Suppose first that $S_0\sim\mu$, and fix $j$. By Lemma~\ref{lem:rev-poisson} with $N=T$, $\Upsilon_0^{(i)}=\mathds{1}\{i=j\}I$, and $\hat p_0=\bm1_j$,
\begin{equation}
\lVert\Xi_T(j)\rVert_2\le\Bigl\lVert\sum_{n<T}d_n\Bigr\rVert_2+\frac{2\alpha b}{1-\beta},\qquad d_n:=\alpha\Upsilon_nm_{n+1}.
\end{equation}
The sequence $(d_n)$ is a martingale difference sequence with respect to $(\mathcal G_n)$, with $\lVert d_n\rVert_2\le2\alpha b$ by Lemma~\ref{lem:rev-dom}. By Lemma~\ref{lem:rev-cv},
\begin{equation}
\Sigma_T:=\sum_{n<T}\mathbb E\bigl[\lVert d_n\rVert_2^2\bigm|\mathcal G_n\bigr]\le\alpha A_T,\qquad A_T:=\alpha\sum_{n<T}\hat p_n^{\mathsf T}\omega(X_n).
\end{equation}
Conditionally on $X_0=x$, the reversed trajectory is the reverse process from $x$, so Proposition~\ref{prop:rev-energy} and Lemma~\ref{lem:rev-khas}, applied to $\mathsf Z_n=(\hat p_n,X_n)$ from $(\bm1_j,x)$, give
\begin{equation}
\Pr\bigl(A_T>2\mathcal E\ell\bigr)\le2e^{-\ell}.
\end{equation}
Lemma~\ref{lem:freedman}, with $c=2\alpha b$ and $v=2\alpha\mathcal E\ell$, gives
\begin{equation}
\Pr\Bigl(\Bigl\lVert\sum_{n<T}d_n\Bigr\rVert_2\ge\ell\sqrt{2C_F\alpha\mathcal E}+2C_F\alpha b\ell,\ \Sigma_T\le2\alpha\mathcal E\ell\Bigr)\le2e^{-\ell}.
\end{equation}
Hence $\lVert\Xi_T(j)\rVert_2\le N^{\mathrm{mk}}$ outside an event of probability at most $4e^{-\ell}$, and a union bound over $j$ gives the claim with probability at least $1-4\lvert\X\rvert e^{-\ell}$. Now, let $S_0$ have an arbitrary law $\nu_0$. Every event $\mathcal B$ determined by the forward trajectory satisfies
\begin{equation}
\Pr_{\nu_0}(\mathcal B)=\mathbb E_\mu\Bigl[\frac{\nu_0(S_0)}{\mu(S_0)}\mathds{1}_{\mathcal B}\Bigr]\le\frac{\Pr_\mu(\mathcal B)}{\mu_{\min}},
\end{equation}
because the conditional law of the trajectory given $S_0$ does not depend on the law of $S_0$. This gives the claim.
\end{proof}

It remains to control the initial error. For the chain with kernel $\hat Q$ started at $x$, let $\tau_i:=\inf\{n\ge0:X_n=i\}$, and define the reverse hitting time
\begin{equation}
h_Q:=\max_{i,x\in\X}\mathbb E_x[\tau_i].
\end{equation}

\begin{lemma}[Reverse hitting time]
\label{lem:rev-hit}
The reverse hitting time satisfies
\begin{equation}
h_Q\le\frac{2C_{\hat Q}^{\mathrm{Pois}}}{\mu_{\min}}.
\end{equation}
If the chain is aperiodic with mixing time $t_{\mathrm{mix}}$ as in \eqref{eq:tmix}, then
\begin{equation}
h_Q\le\frac{2t_{\mathrm{mix}}\lceil\log_2(2/\mu_{\min})\rceil}{\mu_{\min}}.
\end{equation}
\end{lemma}

\begin{proof}
We first prove the bound in terms of $C_{\hat Q}^{\mathrm{Pois}}$. Write $Z:=Z_{\hat Q}$ and fix $i$. Multiplying $(I-\hat Q+\bm1\mu^{\mathsf T})Z=I$ on the left by $\mu^{\mathsf T}$ gives $\mu^{\mathsf T}Z=\mu^{\mathsf T}$, so column $i$ of this identity reads
\begin{equation}
Z(\cdot,i)-\hat QZ(\cdot,i)=\bm1_i-\mu(i)\bm1.
\end{equation}
Hence
\begin{equation}
f(x):=\frac{Z(i,i)-Z(x,i)}{\mu(i)}\qquad\text{satisfies}\qquad f(i)=0,\qquad f(x)=1+(\hat Qf)(x)\quad(x\ne i).
\end{equation}
These are the equations of $x\mapsto\mathbb E_x[\tau_i]$, whose solution is unique. Indeed, the difference $g$ of two solutions satisfies $g(i)=0$ and $g=\hat Qg$ off $i$. If $g$ attains its maximum $M$ at some $x\ne i$, then every successor of $x$ also attains $M$, and following a path from $x$ to $i$, which exists by irreducibility, gives $M=g(i)=0$. Hence $g\le0$, and the same argument for $-g$ gives $g=0$. Therefore
\begin{equation}
\mathbb E_x[\tau_i]=\frac{Z(i,i)-Z(x,i)}{\mu(i)}\le\frac{2C_{\hat Q}^{\mathrm{Pois}}}{\mu_{\min}}.
\end{equation}

The bound in terms of the mixing time is Lemma~7.5 of \citet{peng2026sharp}. Let $r:=\lceil\log_2(2/\mu_{\min})\rceil$ and $t:=rt_{\mathrm{mix}}$. Then
\begin{equation}
\max_i\mathrm{TV}\bigl(Q^t(i,\cdot),\mu\bigr)\le4^{-r}\le\frac{\mu_{\min}}2,
\end{equation}
so $Q^t(i,x)\ge\mu(x)/2$ and
\begin{equation}
\hat Q^t(x,i)=\frac{\mu(i)Q^t(i,x)}{\mu(x)}\ge\frac{\mu(i)}2\qquad\text{for all }i,x.
\end{equation}
We observe the reverse chain at times $t,2t,\dots$. Each observation equals $i$ with conditional probability at least $\mu(i)/2$ given the earlier ones, so the index of the first observation at $i$ has mean at most $2/\mu(i)$. Since $\tau_i$ is at most $t$ times this index, we obtain $\mathbb E_x[\tau_i]\le2t/\mu_{\min}$.
\end{proof}

\begin{lemma}[Forgetting the initial error]
\label{lem:rev-forget}
Assume $\alpha\le1-\beta$, let $\varepsilon_0,\delta\in(0,1]$, and let
\begin{equation}\label{eq:rev-LK}
L_{\mathrm{mk}}:=\Bigl\lceil8\Bigl(\frac{1}{\alpha\mu_{\min}(1-\beta)}+h_Q\Bigr)\Bigr\rceil,\qquad K_{\mathrm{mk}}:=\Bigl\lceil\log_2\frac{2\lvert\X\rvert}{\delta\varepsilon_0\mu_{\min}}\Bigr\rceil.
\end{equation}
For every deterministic $T\ge K_{\mathrm{mk}}L_{\mathrm{mk}}$ and any law of $S_0$, with probability at least $1-\delta/2$, the weights $\hat p_T^{(j)}$ of \eqref{eq:rev-p} started from $\hat p_0=\bm1_j$ satisfy
\begin{equation}
\lVert\hat p_T^{(j)}\rVert_1\le\varepsilon_0\qquad\text{for every }j\in\X.
\end{equation}
\end{lemma}

\begin{proof}
The weights $\hat p_n^{(j)}$ depend only on the reversed state path $(X_n)$, and \eqref{eq:rev-p} is the weight recursion of \citet[Section~7]{peng2026sharp} with step size $\alpha$ and discount factor $\beta$, driven by the same reverse kernel $\hat Q$. Their step-size condition, which requires the step size times the effective horizon to be at most one, reads $\alpha\le1-\beta$. Their particle representation gives $\lVert\hat p_T^{(j)}\rVert_1$ as the conditional probability, given the reversed path, that an auxiliary particle started at $j$ is still alive at time $T$. Their Lemma~7.3 bounds the expected lifetime of this particle, uniformly over its position and the driver state, by
\begin{equation}
4\Bigl(\frac1{\alpha\mu_{\min}(1-\beta)}+h_Q\Bigr)\le\frac{L_{\mathrm{mk}}}2.
\end{equation}
By Markov's inequality and the Markov property, the particle survives $K_{\mathrm{mk}}$ consecutive blocks of length $L_{\mathrm{mk}}$ with probability at most $2^{-K_{\mathrm{mk}}}$. For $S_0\sim\mu$, this gives $\mathbb E\lVert\hat p_T^{(j)}\rVert_1\le2^{-K_{\mathrm{mk}}}$ for every $j$, and Markov's inequality, a union bound over $j$, and the change of initial law in the proof of Proposition~\ref{prop:rev-noise} give
\begin{equation}
\Pr\Bigl(\max_j\lVert\hat p_T^{(j)}\rVert_1>\varepsilon_0\Bigr)\le\frac{\lvert\X\rvert2^{-K_{\mathrm{mk}}}}{\varepsilon_0\mu_{\min}}\le\frac\delta2.
\end{equation}
This is the argument of \citet[Proposition~7.4]{peng2026sharp} with his threshold replaced by $\varepsilon_0$.
\end{proof}

\begin{theorem}[Markovian bound in reverse time]
\label{thm:sharp-rev}
Suppose Assumption~\ref{ass:sharp} holds and the sampling is Markovian. Let $\alpha\le1-\beta$, $\ell\ge1$, and $\varepsilon_0,\delta\in(0,1]$, and let $\mathcal E$, $N^{\mathrm{mk}}$, $L_{\mathrm{mk}}$, and $K_{\mathrm{mk}}$ be as in Propositions~\ref{prop:rev-energy} and~\ref{prop:rev-noise} and \eqref{eq:rev-LK}. Then, for every deterministic $T\ge K_{\mathrm{mk}}L_{\mathrm{mk}}$ and any law of $S_0$, with probability at least $1-\delta/2-4\lvert\X\rvert e^{-\ell}/\mu_{\min}$,
\begin{equation}
\lVert W_T\rVert_{2,\infty}\le\varepsilon_0\lVert W_0\rVert_{2,\infty}+N^{\mathrm{mk}}.
\end{equation}
\end{theorem}

\begin{proof}
For each $j$, \eqref{eq:rev-split} and \eqref{eq:rev-init} give
\begin{equation}
\lVert W_T(j)\rVert_2\le\lVert\hat p_T^{(j)}\rVert_1\lVert W_0\rVert_{2,\infty}+\lVert\Xi_T(j)\rVert_2.
\end{equation}
Lemma~\ref{lem:rev-forget} and Proposition~\ref{prop:rev-noise}, with a union bound, give the claim.
\end{proof}

\begin{proof}[Proof of Theorem~\ref{thm:sharp}]
Fix $\star\in\{\mathrm C,\mathrm M\}$. The parts of the proofs of Lemmas~\ref{lem:ctd-tv} and~\ref{lem:mtd-tv} that verify the structural conditions use neither the endpoint assumption on the supports nor the condition on $h_\Theta$, so Assumption~\ref{ass:sharp} holds with $\beta=\beta_\star$ and $b=B_\star$, by Proposition~\ref{prop:ctd-block-bound} for CTD and Proposition~\ref{prop:mtd-perturbation} for MTD. Its total variance is $\mathrm{TV}_{\star,\max}$. By Propositions~\ref{prop:ctd-isometry} and~\ref{prop:mtd-isometry}, $e_{\star,k}=\lVert W_k\rVert_{2,\infty}$. Under i.i.d.\ sampling, Theorem~\ref{thm:sharp-app} with these constants is part 1, since $N=N_\star$ and $\nu_{\min}=\rho_{\min}$. Under Markovian sampling, Theorem~\ref{thm:sharp-rev} is part 2, since $\mathcal E=\mathcal E_\star$ and $N^{\mathrm{mk}}=N_\star^{\mathrm{mk}}$.
\end{proof}

\subsection{Discounted total variance for CTD}

\begin{lemma}[Cram\'er total variance]
\label{lem:ctd-tv}
Consider CTD with rewards in $[0,1]$ and supports with $\theta_1(s)=0$ and $\theta_d(s)=(1-\gamma)^{-1}$ for every $s$, and let
\begin{equation}
\Delta_{\max}:=\max_{s,j}\bigl(\theta_{j+1}(s)-\theta_j(s)\bigr).
\end{equation}
Then Assumption~\ref{ass:sharp} holds with $\beta=\sqrt\gamma$ and $b=B_{\mathrm C}=(1-\gamma)^{-1/2}$, and
\begin{equation}
\mathrm{TV}(s)\le\frac{2}{1-\gamma}+\frac{8\Delta_{\max}}{(1-\gamma)^2}\qquad\text{for every }s\in\X.
\end{equation}
Moreover, the projection bias satisfies
\begin{equation}
\ell_{\mathrm C,\infty}(\eta^*,\eta^\pi)\le\sqrt{\frac{\Delta_{\max}}{1-\gamma}}.
\end{equation}
\end{lemma}

\begin{proof}
We first verify the structural conditions of Assumption~\ref{ass:sharp}. The map $\eta\mapsto I_{\mathrm C,s}(\Pi_{\mathrm C}^{\Theta(s)}((f_{r,\gamma})_\#\eta))$ is affine on probability laws on $\Theta(s')$, and $I_{\mathrm C,s'}$ is an affine bijection onto the $s'$th block of $\mathcal D_{\mathrm C}$. Their composition is affine, and its linear part $L_y$ has norm at most $\sqrt\gamma$ on $E_0(s')$ by Lemma~\ref{lem:ctd-projection}, since every element of $E_0(s')$ is a multiple of $I_{\mathrm C,s'}(\mu)-I_{\mathrm C,s'}(\mu')$ for probability laws $\mu,\mu'$. The fixed-point noise has conditional mean zero because
\begin{equation}
U^*(s)=I_{\mathrm C,s}\bigl(\Pi_{\mathrm C}^{\Theta(s)}(T^\pi\eta^*)(s)\bigr)
\end{equation}
is the conditional mean of $\widehat T(U^*;Y)$, and $\lVert\xi\rVert_2\le B_{\mathrm C}$ because $\widehat T(U^*;y)$ and $U^*(s)$ are embedded probability laws at state $s$, whose $j$th coordinates lie in $[0,\sqrt{\Delta_j(s)}]$ by the proof of Proposition~\ref{prop:ctd-block-bound}.

Next, we bound the total variance. Fix $s$ and let
\begin{equation}
\nu_\omega:=(f_{R(s,A),\gamma})_\#\eta^*(S'),\qquad(A,S')\sim\pi(\cdot\mid s)P(\cdot\mid s,A).
\end{equation}
By the proof of Lemma~\ref{lem:ctd-projection}, the $j$th coordinate of $I_{\mathrm C,s}(\Pi_{\mathrm C}^{\Theta(s)}\nu)$ is $\sqrt{\Delta_j(s)}$ times the average of $F_\nu$ over the cell $[\theta_j(s),\theta_{j+1}(s)]$, so Jensen's inequality on each cell gives
\begin{equation}
\lVert\xi\rVert_2^2\le\int_{\mathbb R}\bigl(F_{\nu_\omega}-\mathbb E[F_{\nu_\omega}]\bigr)^2\,\mathrm dz,
\end{equation}
and hence
\begin{equation}
\sigma^2(s)\le v(s):=\int_{\mathbb R}\operatorname{Var}\bigl(F_{\nu_\omega}(z)\bigr)\,\mathrm dz.
\end{equation}

We use a martingale of conditional distribution functions. Run the chain from $S_0=s$ with $A_k\sim\pi(\cdot\mid S_k)$ and $S_{k+1}\sim P(\cdot\mid S_k,A_k)$, and let
\begin{equation}
G_k:=\sum_{i<k}\gamma^iR(S_i,A_i),\qquad\mathcal H_k:=\sigma(S_0,A_0,\dots,S_k),\qquad Y_k(z):=F_{\eta^*(S_k)}\bigl((z-G_k)/\gamma^k\bigr)\quad(z\in\mathbb R).
\end{equation}
Since $F_{(f_{r,\gamma})_\#\eta}(u)=F_\eta((u-r)/\gamma)$, we have
\begin{equation}
\mathbb E[Y_{k+1}(z)\mid\mathcal H_k]=F_{(T^\pi\eta^*)(S_k)}\bigl((z-G_k)/\gamma^k\bigr).
\end{equation}
Write $Y_{k+1}-Y_k=D_{k+1}+A_k$ with
\begin{equation}
D_{k+1}:=Y_{k+1}-\mathbb E[Y_{k+1}\mid\mathcal H_k],\qquad A_k(z):=\bigl(F_{(T^\pi\eta^*)(S_k)}-F_{\eta^*(S_k)}\bigr)\bigl((z-G_k)/\gamma^k\bigr).
\end{equation}
The substitution $z=G_k+\gamma^ku$ gives
\begin{equation}
\int\mathbb E\bigl[D_{k+1}(z)^2\mid\mathcal H_k\bigr]\,\mathrm dz=\gamma^kv(S_k).
\end{equation}
The increments $D_{k+1}(z)$ are orthogonal martingale differences, so, for every $K$,
\begin{equation}
\sum_{k<K}\gamma^k\mathbb E[v(S_k)]=\int\mathbb E\Bigl[\Bigl(Y_K-Y_0-\sum_{k<K}A_k\Bigr)^2\Bigr]\mathrm dz\le2\int\mathbb E\bigl[(Y_K-Y_0)^2\bigr]\,\mathrm dz+2\,\mathbb E\Bigl[\int\Bigl(\sum_{k<K}A_k\Bigr)^2\mathrm dz\Bigr].
\end{equation}

We now bound the two terms. Both $Y_K$ and $Y_0$ take values in $[0,1]$, vanish for $z<0$, and equal $1$ for $z\ge(1-\gamma)^{-1}$, because $G_K\ge0$ and $G_K+\gamma^K(1-\gamma)^{-1}\le(1-\gamma)^{-1}$. Hence, the first integral is at most $(1-\gamma)^{-1}$. For the second, note that $\eta^*(u)=\Pi_{\mathrm C}^{\Theta(u)}(T^\pi\eta^*)(u)$ and that $(T^\pi\eta^*)(u)$ is supported in $[0,(1-\gamma)^{-1}]$. On each cell, the distribution function differs from its cell average by at most the mass of the cell, so
\begin{equation}
\int\bigl(F_{(T^\pi\eta^*)(u)}-F_{\eta^*(u)}\bigr)^2\,\mathrm dz\le\Delta_{\max}.
\end{equation}
The same substitution gives, pathwise,
\begin{equation}
\lVert A_k\rVert_{L^2(\mathrm dz)}\le\gamma^{k/2}\sqrt{\Delta_{\max}},
\end{equation}
and Minkowski's inequality gives
\begin{equation}
\int\Bigl(\sum_kA_k\Bigr)^2\mathrm dz\le\frac{\Delta_{\max}}{(1-\sqrt\gamma)^2}\le\frac{4\Delta_{\max}}{(1-\gamma)^2}.
\end{equation}
Letting $K\to\infty$ and using $((P^\pi)^kv)(s)=\mathbb E[v(S_k)]$ and $\sigma^2\le v$ proves the bound on $\mathrm{TV}(s)$.

We now bound the projection bias. For a probability law $\mu$ on $[0,(1-\gamma)^{-1}]$, the distribution function of $\Pi_{\mathrm C}^{\Theta(u)}\mu$ is the $L^2(\mathrm dz)$-orthogonal projection of $F_\mu$ onto the affine set of functions that vanish below $\theta_1(u)$, equal $1$ from $\theta_d(u)$ on, and are constant on each cell, because it equals the cell averages of $F_\mu$ (proof of Lemma~\ref{lem:ctd-projection}). Since $\eta^*(u)=\Pi_{\mathrm C}^{\Theta(u)}(T^\pi\eta^*)(u)$ and $\eta^\pi(u)=(T^\pi\eta^\pi)(u)$ are supported in $[0,(1-\gamma)^{-1}]$, Pythagoras' theorem, the nonexpansiveness of the projection, and the scaling and mixture parts of Lemma~\ref{lem:ctd-projection} give
\begin{equation}
\ell_{\mathrm C}\bigl(\eta^*(u),\eta^\pi(u)\bigr)^2\le\gamma\,\ell_{\mathrm C,\infty}(\eta^*,\eta^\pi)^2+\ell_{\mathrm C}\bigl(\Pi_{\mathrm C}^{\Theta(u)}\eta^\pi(u),\eta^\pi(u)\bigr)^2.
\end{equation}
On each cell, $F_{\eta^\pi(u)}$ differs from its cell average by at most the mass of the cell, so the last term is at most $\sum_j\Delta_j(u)m_j^2\le\Delta_{\max}$, where $m_j$ are the cell masses. Taking the maximum over $u$ gives
\begin{equation}
(1-\gamma)\,\ell_{\mathrm C,\infty}(\eta^*,\eta^\pi)^2\le\Delta_{\max}.
\end{equation}
\end{proof}

\subsection{Discounted total variance for MTD}

Recall from the proof of Proposition~\ref{prop:mtd-contraction-app} that $d_\kappa(x,y)=\mathrm{MMD}_\kappa(\delta_x,\delta_y)^2$ is the semimetric inducing $\kappa$ up to a positive factor, so it is shift invariant and $c$-homogeneous. We write $\psi(\mu)\in\mathcal H_\kappa$ for the mean embedding of a compactly supported finite signed measure $\mu$, so that
\begin{equation}
\mathrm{MMD}_\kappa(\mu,\nu)=\lVert\psi(\mu)-\psi(\nu)\rVert_{\mathcal H_\kappa},\qquad\lVert\psi(\delta_x)-\psi(\delta_y)\rVert^2=d_\kappa(x,y),
\end{equation}
and
\begin{equation}
\bigl\lVert\psi((f_{r,\gamma})_\#\mu)-\psi((f_{r,\gamma})_\#\nu)\bigr\rVert=\gamma^{c/2}\lVert\psi(\mu)-\psi(\nu)\rVert\qquad\text{whenever $\mu$ and $\nu$ have equal mass.}
\end{equation}
For a compactly supported probability law $\nu$, we define $\Pi_{\mathrm M}^{\Theta(s)}\nu$ by the same MMD minimization as in Section~\ref{sec:mtd-new2}. Its mean embedding is the metric projection of $\psi(\nu)$ onto the finite-dimensional affine subspace $\psi(\{\text{mass-}1\text{ signed laws on }\Theta(s)\})$ of $\mathcal H_\kappa$, and the minimizer is unique by Proposition~\ref{prop:mtd-isometry} and the injectivity discussed after it. Since the metric projection onto a closed affine subspace is $1$-Lipschitz, the extended projection is nonexpansive in $\mathrm{MMD}_\kappa$. The scaling and mixture arguments in the proof of Proposition~\ref{prop:mtd-contraction-app} apply unchanged to compactly supported laws. Let $B_\gamma:=[0,(1-\gamma)^{-1}]^q$, which contains the support of every return law $\eta^\pi(s)$, and define
\begin{equation}
\bar\kappa:=\sup_{x,y\in[0,1]^q}d_\kappa(x,y),\qquad h_\Theta:=\max_{s\in\X}\sup_{z\in B_\gamma}\min_{i\le d}d_\kappa(z,\theta_i(s)).
\end{equation}
By homogeneity and shift invariance,
\begin{equation}
\sup_{x,y\in B_\gamma}d_\kappa(x,y)=\frac{\bar\kappa}{(1-\gamma)^c}.
\end{equation}
The quantity $h_\Theta$ is the squared covering radius of the atoms in $\mathrm{MMD}_\kappa$. For instance, a grid of spacing $\Delta$ in each coordinate that covers $B_\gamma$ gives $h_\Theta\le\bar\kappa\Delta^c$.

\begin{lemma}[MMD total variance]
\label{lem:mtd-tv}
Consider MTD with rewards in $[0,1]^q$. Then Assumption~\ref{ass:sharp} holds with $\beta=\gamma^{c/2}$, and
\begin{equation}
b\le\bar B_{\mathrm M}:=\frac{\sqrt{\bar\kappa}}{(1-\gamma)^{c/2}}+2\sqrt{\frac{h_\Theta}{1-\gamma^c}},\qquad \mathrm{TV}(s)\le\frac{\bar\kappa}{(1-\gamma)^c}+\frac{2h_\Theta}{(1-\gamma^c)^2}\qquad\text{for every }s\in\X.
\end{equation}
In particular, $B_{\mathrm M}\le\bar B_{\mathrm M}$. Moreover, the projection bias satisfies
\begin{equation}
\ell_{\mathrm M,\infty}(\eta^*,\eta^\pi)\le\sqrt{\frac{h_\Theta}{1-\gamma^c}}.
\end{equation}
\end{lemma}

\begin{proof}
We first verify the structural conditions of Assumption~\ref{ass:sharp}. For $y=(s,a,s')$, the map $\mu\mapsto J_s((f_{R(s,a),\gamma})_\#\mu)$ is affine on mass-$1$ signed laws on $\Theta(s')$ by Proposition~\ref{prop:mtd-proj-affine}, and $I_{\mathrm M,s'}$ is an affine bijection onto the $s'$th block of $\mathcal D_{\mathrm M}$. Every element of $E_0(s')$ is a difference of two points of that block, so the linear part $L_y$ of the composition is well defined, and Proposition~\ref{prop:mtd-proj-affine} with the scaling identity gives $\lVert L_y\rVert\le\gamma^{c/2}$. The fixed-point noise has conditional mean zero by Proposition~\ref{prop:mtd-unbiased} and the fixed-point property.

Next, we bound the projection bias. Fix $u$. The mean embedding of $\eta^*(u)=\Pi_{\mathrm M}^{\Theta(u)}(T^\pi\eta^*)(u)$ minus that of $\Pi_{\mathrm M}^{\Theta(u)}\eta^\pi(u)$ lies in the direction space of the affine subspace, while $\psi(\eta^\pi(u))-\psi(\Pi_{\mathrm M}^{\Theta(u)}\eta^\pi(u))$ is orthogonal to it. Since $T^\pi\eta^\pi=\eta^\pi$ and the projection is nonexpansive, Pythagoras' theorem gives
\begin{equation}
\mathrm{MMD}_\kappa\bigl(\eta^*(u),\eta^\pi(u)\bigr)^2\le\gamma^c\ell_{\mathrm M,\infty}(\eta^*,\eta^\pi)^2+\epsilon_\Theta^2,\qquad\epsilon_\Theta:=\max_u\mathrm{MMD}_\kappa\bigl(\Pi_{\mathrm M}^{\Theta(u)}\eta^\pi(u),\eta^\pi(u)\bigr),
\end{equation}
and taking the maximum over $u$ gives
\begin{equation}
\ell_{\mathrm M,\infty}(\eta^*,\eta^\pi)\le\frac{\epsilon_\Theta}{\sqrt{1-\gamma^c}}.
\end{equation}
Let $Q_u(z)$ be a nearest atom of $\Theta(u)$ to $z$ in $d_\kappa$, chosen measurably. Then $(Q_u)_\#\eta^\pi(u)$ is a probability law on $\Theta(u)$, so, with $Z\sim\eta^\pi(u)$,
\begin{equation}
\epsilon_\Theta\le\max_u\bigl\lVert\mathbb E\bigl[\psi(\delta_Z)-\psi(\delta_{Q_u(Z)})\bigr]\bigr\rVert\le\max_u\mathbb E\Bigl[\sqrt{d_\kappa(Z,Q_u(Z))}\Bigr]\le\sqrt{h_\Theta}.
\end{equation}

We now bound the noise. Write
\begin{equation}
\nu_y:=(f_{R(s,a),\gamma})_\#\eta^*(s'),\qquad\nu^\pi_y:=(f_{R(s,a),\gamma})_\#\eta^\pi(s').
\end{equation}
By Proposition~\ref{prop:mtd-proj-affine},
\begin{equation}
\xi(y)=J_s(\nu_y)-\mathbb E\bigl[J_s(\nu_Y)\mid S=s\bigr],\qquad\lVert\xi(y)\rVert_2\le\sup_{y'}\mathrm{MMD}_\kappa(\nu_y,\nu_{y'}),
\end{equation}
the supremum being over samples $y'$ from state $s$. The same bound holds for every reward in $[0,1]^q$, which also covers $B_{\mathrm M}$. By the triangle inequality,
\begin{equation}
\mathrm{MMD}_\kappa(\nu_y,\nu_{y'})\le2\gamma^{c/2}\ell_{\mathrm M,\infty}(\eta^*,\eta^\pi)+\mathrm{MMD}_\kappa(\nu^\pi_y,\nu^\pi_{y'}),
\end{equation}
and, for any coupling of $X\sim\nu^\pi_y$ and $X'\sim\nu^\pi_{y'}$, the last term is at most
\begin{equation}
\mathbb E\Bigl[\sqrt{d_\kappa(X,X')}\Bigr]\le\frac{\sqrt{\bar\kappa}}{(1-\gamma)^{c/2}},
\end{equation}
because both laws live on $B_\gamma$. This gives $b\le\bar B_{\mathrm M}$.

We now bound the variance. Since $J_s$ is affine and $1$-Lipschitz from $\mathrm{MMD}_\kappa$ to $\lVert\cdot\rVert_2$ on finitely supported laws, and $\mathbb E[\nu_Y\mid S=s]=(T^\pi\eta^*)(s)$ is finitely supported,
\begin{equation}
\sigma^2(s)\le\mathbb E\Bigl[\bigl\lVert\psi(\nu_Y)-\mathbb E[\psi(\nu_Y)\mid S=s]\bigr\rVert^2\Bigm| S=s\Bigr].
\end{equation}
Let
\begin{equation}
\tilde v(s):=\mathbb E\bigl[\lVert\psi(\nu^\pi_Y)-\psi(\eta^\pi(s))\rVert^2\mid S=s\bigr],
\end{equation}
and note that $\mathbb E[\psi(\nu^\pi_Y)\mid S=s]=\psi((T^\pi\eta^\pi)(s))=\psi(\eta^\pi(s))$. Minkowski's inequality in $L^2(\mathcal H_\kappa)$ and the scaling identity give
\begin{equation}
\sigma(s)\le\sqrt{\tilde v(s)}+\gamma^{c/2}\ell_{\mathrm M,\infty}(\eta^*,\eta^\pi),\qquad\text{hence}\qquad\sigma^2\le2\tilde v+\frac{2\gamma^ch_\Theta}{1-\gamma^c}.
\end{equation}

We use a martingale of mean embeddings. Run the chain from $S_0=s$ as in the proof of Lemma~\ref{lem:ctd-tv}, with $G_k$, $\mathcal H_k$, and $G_\infty:=\lim_kG_k$, and let
\begin{equation}
\tilde Y_k:=\psi\bigl((g_k)_\#\eta^\pi(S_k)\bigr),\qquad g_k(z):=G_k+\gamma^kz.
\end{equation}
Given $\mathcal H_k$, the law of $G_\infty$ is $(g_k)_\#\eta^\pi(S_k)$, so
\begin{equation}
\tilde Y_k=\mathbb E\bigl[\psi(\delta_{G_\infty})\mid\mathcal H_k\bigr]
\end{equation}
is a square-integrable martingale in $\mathcal H_\kappa$. Because $g_{k+1}=g_k\circ f_{R_k,\gamma}$ and $(g_k)_\#$ scales MMD between probability laws by $\gamma^{kc/2}$, its increments satisfy
\begin{equation}
\mathbb E\bigl[\lVert\tilde Y_{k+1}-\tilde Y_k\rVert^2\mid\mathcal H_k\bigr]=\gamma^{kc}\tilde v(S_k).
\end{equation}
Orthogonality of the increments and the contraction property of conditional expectation give, for every $K$,
\begin{equation}
\sum_{k<K}\gamma^{kc}\mathbb E[\tilde v(S_k)]=\mathbb E\bigl[\lVert\tilde Y_K-\tilde Y_0\rVert^2\bigr]\le\mathbb E\bigl[\lVert\psi(\delta_{G_\infty})-\psi(\eta^\pi(s))\rVert^2\bigr]=\tfrac12\mathbb E\bigl[d_\kappa(G_\infty,G'_\infty)\bigr]\le\frac{\bar\kappa}{2(1-\gamma)^c},
\end{equation}
where $G'_\infty$ is an independent copy with the same law $\eta^\pi(s)$. Since $\beta^2=\gamma^c$ and $((P^\pi)^k\tilde v)(s)=\mathbb E[\tilde v(S_k)]$, letting $K\to\infty$ gives
\begin{equation}
\mathrm{TV}(s)\le2\sum_{k\ge0}\gamma^{kc}\bigl((P^\pi)^k\tilde v\bigr)(s)+\frac{2\gamma^ch_\Theta}{(1-\gamma^c)^2}\le\frac{\bar\kappa}{(1-\gamma)^c}+\frac{2h_\Theta}{(1-\gamma^c)^2}.
\end{equation}
\end{proof}

\subsection{Consequences for CTD and MTD}

\begin{proof}[Proof of Corollary~\ref{cor:sharp-ctd}]
We first collect the parameters. By Lemma~\ref{lem:ctd-tv}, Assumption~\ref{ass:sharp} holds with
\begin{equation}
\beta=\sqrt\gamma,\qquad1-\beta^2=1-\gamma,\qquad1-\beta\ge\frac{1-\gamma}2,\qquad b=B_{\mathrm C},\qquad\mathrm{TV}_{\max}\le\frac{\bar\tau}{1-\gamma},
\end{equation}
where $\bar\tau:=2+8\Delta_{\max}/(1-\gamma)$. Moreover, by the proof of Proposition~\ref{prop:ctd-block-bound}, the $j$th coordinates of $U(s)$ and $U'(s)$ both lie in $[0,\sqrt{\Delta_j(s)}]$ for all $U,U'\in\mathcal D_{\mathrm C}$, so
\begin{equation}
\lVert U(s)-U'(s)\rVert_2^2\le\sum_j\Delta_j(s)\le B_{\mathrm C}^2.
\end{equation}
In particular, $\lVert W_t\rVert_{2,\infty}\le B_{\mathrm C}$ for all $t$. Let $\varepsilon\in(0,1]$ and $\delta\in(0,1)$, so that $\varepsilon\le B_{\mathrm C}$ and $\varepsilon^2<\bar\tau$.

In the i.i.d.\ regime, we choose the step size as follows. Set $D_0:=B_{\mathrm C}$ and $K:=\lceil\log_2(2B_{\mathrm C}/\varepsilon)\rceil$, and for $\ell\ge1$ let
\begin{equation}
\alpha(\ell):=\frac{(1-\gamma)\varepsilon^2}{2^{11}C_F\bar\tau\ell}.
\end{equation}
Since $\varepsilon^2<\bar\tau$ and $1-\beta^2=1-\gamma$, we have $\alpha(\ell)\le(1-\beta^2)/(2^{11}C_F\ell)$, as Theorem~\ref{thm:sharp-app} requires. Using $\mathrm{TV}_{\max}\le\bar\tau/(1-\gamma)$, the two summands of $N$ satisfy
\begin{equation}
\sqrt{6C_F\alpha\,\mathrm{TV}_{\max}\ell}\le\sqrt{\frac6{2^{11}}}\,\varepsilon\le\frac\varepsilon{16},\qquad2C_F\alpha B_{\mathrm C}\ell=\frac{(1-\gamma)^{1/2}\varepsilon^2}{2^{10}\bar\tau}\le\frac\varepsilon{16}.
\end{equation}
Hence, $N\le\varepsilon/8$, $D_0\ge\varepsilon\ge2N$, and $2^{-K}D_0\le\varepsilon/2$.

We next choose the confidence level and count the iterations in the i.i.d.\ regime. Let $T(\ell):=KL(\ell)$ with $L(\ell)$ as in Theorem~\ref{thm:sharp-app} for $\alpha=\alpha(\ell)$. Since $1/\alpha(\ell)$ is linear in $\ell$, $T(\ell)$ is at most affine in $\ell$, and there exists $\ell\ge1$ with
\begin{equation}
\ell\ge\log\Bigl(\frac{16T(\ell)\lvert\X\rvert(\log_2T(\ell)+3)}{\delta}\Bigr)
\end{equation}
that is logarithmic in $\lvert\X\rvert$, $\delta^{-1}$, $\varepsilon^{-1}$, $\rho_{\min}^{-1}$, and $(1-\gamma)^{-1}$. For this $\ell$ and $T=T(\ell)$, Theorem~\ref{thm:sharp-app} gives, with probability at least $1-\delta$,
\begin{equation}
\lVert W_T\rVert_{2,\infty}\le2N+\frac\varepsilon2<\varepsilon.
\end{equation}
Since $\alpha\rho_{\min}(1-\beta)\le1$ and $1-\beta\ge(1-\gamma)/2$,
\begin{equation}
L\le\frac{\log8+1}{\alpha\rho_{\min}(1-\beta)}\le\frac{3.1}{\alpha\rho_{\min}(1-\beta)},\qquad T\le\frac{6.2K}{\alpha\rho_{\min}(1-\gamma)},
\end{equation}
and inserting $1/\alpha(\ell)$ gives $T_{\mathrm{iid}}$.

In the Markovian regime, let
\begin{equation}
\ell:=\log\frac{8\lvert\X\rvert}{\delta\mu_{\min}}\ge1,\qquad\varepsilon_0:=\frac{\varepsilon}{2B_{\mathrm C}},\qquad\alpha:=\min\Bigl\{\frac{(1-\gamma)\varepsilon^2}{2^9C_F\bar\tau\ell^2},\ \frac{(1-\gamma)\varepsilon}{2^5B_{\mathrm C}}\Bigr\}.
\end{equation}
Then $\alpha\le1-\beta$. Since $\bar\tau\ge2$ and $B_{\mathrm C}^2=(1-\gamma)^{-1}$, the constant of Proposition~\ref{prop:rev-energy} satisfies
\begin{equation}
\mathcal E\le\frac{4\bar\tau+3}{1-\gamma}\le\frac{6\bar\tau}{1-\gamma},
\end{equation}
so the three summands of $N^{\mathrm{mk}}$ in Proposition~\ref{prop:rev-noise} are at most
\begin{equation}
\ell\sqrt{\frac{12C_F\alpha\bar\tau}{1-\gamma}}\le\frac{\varepsilon}{6},\qquad2C_F\alpha B_{\mathrm C}\ell\le\frac{(1-\gamma)^{1/2}\varepsilon^2}{2^8\bar\tau\ell}\le\frac\varepsilon8,\qquad\frac{2\alpha B_{\mathrm C}}{1-\beta}\le\frac{4\alpha B_{\mathrm C}}{1-\gamma}\le\frac\varepsilon8,
\end{equation}
so $N^{\mathrm{mk}}<\varepsilon/2$. Theorem~\ref{thm:sharp-rev} with $\lVert W_0\rVert_{2,\infty}\le B_{\mathrm C}$ gives, at $T:=K_{\mathrm{mk}}L_{\mathrm{mk}}$ and with probability at least $1-\delta/2-4\lvert\X\rvert e^{-\ell}/\mu_{\min}=1-\delta$,
\begin{equation}
\lVert W_T\rVert_{2,\infty}\le\varepsilon_0B_{\mathrm C}+N^{\mathrm{mk}}<\varepsilon.
\end{equation}
Here,
\begin{equation}
K_{\mathrm{mk}}=\Bigl\lceil\log_2\frac{4\lvert\X\rvert B_{\mathrm C}}{\delta\varepsilon\mu_{\min}}\Bigr\rceil,
\end{equation}
and, since $1/(\alpha\mu_{\min}(1-\beta))\ge1$,
\begin{equation}
L_{\mathrm{mk}}\le\frac{9}{\alpha\mu_{\min}(1-\beta)}+8h_Q\le\frac{18}{\alpha\mu_{\min}(1-\gamma)}+8h_Q.
\end{equation}
Inserting the two terms of $1/\alpha$ gives $T_{\mathrm{mk}}$.

We now turn to the bound in expectation. Since $\lVert W_T\rVert_{2,\infty}\le B_{\mathrm C}$ holds surely, running the argument with $\min\{\delta,\varepsilon^2/B_{\mathrm C}^2\}$ in place of $\delta$ gives both the probability bound and
\begin{equation}
\mathbb E\bigl[\lVert W_T\rVert_{2,\infty}^2\bigr]\le\varepsilon^2+\varepsilon^2=2\varepsilon^2.
\end{equation}
This changes only the logarithmic factor.

Finally, we pass to the true return law. The grid condition gives
\begin{equation}
\Delta_{\max}\le\frac{(1-\gamma)^{-1}}{d-1}\le\frac{\varepsilon^2(1-\gamma)}4,
\end{equation}
hence $\bar\tau\le2+2\varepsilon^2\le4$ and, by Lemma~\ref{lem:ctd-tv}, $\ell_{\mathrm C,\infty}(\eta^*,\eta^\pi)\le\varepsilon/2$. Running the argument above with $\varepsilon/2$ in place of $\varepsilon$ and using the triangle inequality gives $\ell_{\mathrm C,\infty}(\eta_T,\eta^\pi)\le\varepsilon$ with probability at least $1-\delta$ and
\begin{equation}
\mathbb E\bigl[\ell_{\mathrm C,\infty}(\eta_T,\eta^\pi)^2\bigr]\le2\cdot2(\varepsilon/2)^2+2(\varepsilon/2)^2\le2\varepsilon^2,
\end{equation}
which changes only constants. The $\bar W_1$ statement follows in the same way from $\bar W_1\le(1-\gamma)^{-1/2}\ell_{\mathrm C,\infty}$.
\end{proof}

\begin{proof}[Proof of Corollary~\ref{cor:sharp-mtd}]
We first collect the parameters. By Lemma~\ref{lem:mtd-tv}, Assumption~\ref{ass:sharp} holds with
\begin{equation}
\beta=\gamma^{c/2},\qquad1-\beta^2=1-\gamma^c,\qquad b\le\bar B_{\mathrm M},\qquad\mathrm{TV}_{\max}\le\bar\tau_{\mathrm M}(1-\gamma)^{-c}.
\end{equation}
Set $D_0:=\max\{e_{\mathrm M,0},\bar B_{\mathrm M}\}\ge\varepsilon$. We use
\begin{equation}
1-\gamma^{c/2}\ge\frac c2(1-\gamma),\qquad1-\gamma^c\ge\min\{1,c\}(1-\gamma),
\end{equation}
and, since $1-\gamma^c\le1$,
\begin{equation}
\bar B_{\mathrm M}^2\le2\bar\kappa(1-\gamma)^{-c}+\frac{8h_\Theta}{1-\gamma^c}\le4\Bigl(\bar\kappa(1-\gamma)^{-c}+\frac{2h_\Theta}{(1-\gamma^c)^2}\Bigr)=4\bar\tau_{\mathrm M}(1-\gamma)^{-c}.
\end{equation}
In particular, $\varepsilon\bar B_{\mathrm M}\le\bar B_{\mathrm M}^2\le4\bar\tau_{\mathrm M}(1-\gamma)^{-c}$.

In the i.i.d.\ regime, let $K:=\lceil\log_2(2D_0/\varepsilon)\rceil$, and let
\begin{equation}
\alpha(\ell):=\min\Bigl\{\frac{1-\gamma^c}{2^{11}C_F\ell},\ \frac{(1-\gamma)^c\varepsilon^2}{2^{11}C_F\bar\tau_{\mathrm M}\ell}\Bigr\}.
\end{equation}
This choice satisfies the conditions of Theorem~\ref{thm:sharp-app}, and the two summands of $N$ satisfy
\begin{equation}
\sqrt{6C_F\alpha\,\mathrm{TV}_{\max}\ell}\le\sqrt{\frac6{2^{11}}}\,\varepsilon\le\frac\varepsilon{16},\qquad2C_F\alpha b\ell\le\frac{(1-\gamma)^c\varepsilon^2\bar B_{\mathrm M}}{2^{10}\bar\tau_{\mathrm M}}\le\frac\varepsilon{2^8},
\end{equation}
so $N\le\varepsilon/8$ and $D_0\ge2N$. The choice of $\ell$ and the probability bound follow as in the proof of Corollary~\ref{cor:sharp-ctd}, and
\begin{equation}
T=KL\le\frac{3.1K}{\alpha\rho_{\min}(1-\gamma^{c/2})},
\end{equation}
which gives $T_{\mathrm{iid}}$ with constants depending only on $c$.

In the Markovian regime, the bound on $\bar B_{\mathrm M}^2$ gives $\mathcal E\le16\bar\tau_{\mathrm M}(1-\gamma)^{-c}$. Let
\begin{equation}
\ell:=\log\frac{8\lvert\X\rvert}{\delta\mu_{\min}},\qquad\varepsilon_0:=\frac{\varepsilon}{2D_0},\qquad\alpha:=\min\Bigl\{\frac{(1-\gamma)^c\varepsilon^2}{2^{11}C_F\bar\tau_{\mathrm M}\ell^2},\ \frac{(1-\gamma^{c/2})\varepsilon}{2^4\bar B_{\mathrm M}}\Bigr\}.
\end{equation}
Then $\alpha\le1-\gamma^{c/2}$ because $\varepsilon\le\bar B_{\mathrm M}$, and each summand of $N^{\mathrm{mk}}$ is at most $\varepsilon/8$. For the second summand, this follows from
\begin{equation}
2C_F\alpha\bar B_{\mathrm M}\ell\le\frac{(1-\gamma)^c\varepsilon^2\bar B_{\mathrm M}}{2^{10}\bar\tau_{\mathrm M}\ell}\le\frac\varepsilon{2^8}.
\end{equation}
Hence, $\varepsilon_0D_0+N^{\mathrm{mk}}<\varepsilon$, and Theorem~\ref{thm:sharp-rev} gives $e_{\mathrm M,T}\le\varepsilon$ with probability at least $1-\delta$ at $T:=K_{\mathrm{mk}}L_{\mathrm{mk}}$. As in the proof of Corollary~\ref{cor:sharp-ctd},
\begin{equation}
L_{\mathrm{mk}}\le\frac9{\alpha\mu_{\min}(1-\gamma^{c/2})}+8h_Q,
\end{equation}
and inserting the two terms of $1/\alpha$ gives $T_{\mathrm{mk}}$.

In both regimes, the constants are of order at most $c^{-5/2}$ as $c\to0$.

For the bound in expectation, the recursion gives
\begin{equation}
\lVert W_{t+1}(S_t)\rVert_2\le(1-\alpha)\lVert W_t\rVert_{2,\infty}+\alpha\bigl(\beta\lVert W_t\rVert_{2,\infty}+b\bigr),
\end{equation}
so, by induction, for all $t$,
\begin{equation}
\lVert W_t\rVert_{2,\infty}\le B_W:=\max\Bigl\{D_0,\frac{\bar B_{\mathrm M}}{1-\beta}\Bigr\}.
\end{equation}
Running the argument with $\min\{\delta,\varepsilon^2/B_W^2\}$ in place of $\delta$ gives $\mathbb E[e_{\mathrm M,T}^2]\le2\varepsilon^2$ at the cost of a logarithmic factor.

Finally, for the true return law, Lemma~\ref{lem:mtd-tv} gives $\ell_{\mathrm M,\infty}(\eta^*,\eta^\pi)\le\sqrt{h_\Theta/(1-\gamma^c)}\le\varepsilon/2$, and running the argument with $\varepsilon/2$ in place of $\varepsilon$ gives the claim as for CTD.
\end{proof}

\section{Analysis of the fixed-horizon recursions}
\label{app:fh-sharp}

This appendix proves Theorem~\ref{thm:undisc} and Corollaries~\ref{cor:undisc} and~\ref{cor:undisc-mtd}. The argument is the one of Appendix~\ref{app:sharp}, applied to the blocks indexed by horizon-state pairs. We first state the fixed-horizon versions of the ingredients that change: the structure of the sampled target, the propagation lemma, whose contraction now comes from the nilpotent horizon shift rather than from a discount factor, and the Poisson solution for the episodic chain. We then list the changes to the proof of Theorem~\ref{thm:sharp-app} and bound the fixed-horizon total variance for CTD and MTD.

\subsection{Structural assumptions}

Let $\mathcal S_H=\{1,\dots,H\}\times\X$ be the set of blocks, and let
\begin{equation}
V_H:=\prod_{(h,s)\in\mathcal S_H}E(h,s),\qquad\lVert Y\rVert_{2,\infty}:=\max_{(h,s)\in\mathcal S_H}\lVert Y(h,s)\rVert_2,
\end{equation}
be a finite product of Euclidean blocks and its block-supremum norm. For $u\in\X$, we write $P^H_u$ for the coordinate projector onto the blocks $(h,u)$ with $1\le h\le H$. Let $\mathcal D_H\subseteq V_H$ be the invariant domain of the method, $U^*\in\mathcal D_H$ the embedded fixed point, and $E_0(h,s)$ the linear span of differences of points of the $(h,s)$th block of $\mathcal D_H$, with $E_0:=\prod_{(h,s)}E_0(h,s)$. A sample is $y=(s,a,s')$, with reward $R(s,a)$, and $\widehat T_H(U;y)=(\widehat T_H^h(U;y))_{h=1}^H$ is the stacked target at state $s$.

\begin{assumption}[Structure of the fixed-horizon target]
\label{ass:fh}
\begin{enumerate}
\item The first layer $\widehat T_H^1(U;y)$ does not depend on $U$. For every sample $y=(s,a,s')$ and $2\le h\le H$ there is a linear map $L_y^h:E_0(h-1,s')\to E_0(h,s)$ with operator norm at most $1$ such that
\begin{equation}
\widehat T_H^h(U;y)-\widehat T_H^h(U';y)=L_y^h\bigl(U(h-1,s')-U'(h-1,s')\bigr)\qquad\text{for all }U,U'\in\mathcal D_H.
\end{equation}
\item The fixed-point noise $\xi^h(y):=\widehat T_H^h(U^*;y)-U^*(h,s)$ satisfies $\mathbb E[\xi^h(Y)\mid S=s]=0$ when $A\sim\pi(\cdot\mid s)$ and $S'\sim P(\cdot\mid s,A)$, and $\lVert\xi^h(y)\rVert_2\le b$ for every sample and every $h$.
\end{enumerate}
We write $\sigma^2(h,s):=\mathbb E[\lVert\xi^h(Y)\rVert_2^2\mid S=s]$. Let $N_H\in\mathbb R^{\mathcal S_H\times\mathcal S_H}$ be the horizon shift,
\begin{equation}
(N_Hq)(h,s):=\sum_{s'}P^\pi(s,s')q(h-1,s')\quad(h\ge2),\qquad(N_Hq)(1,s):=0.
\end{equation}
Since $N_H^H=0$, we define the fixed-horizon total variance
\begin{equation}
\mathrm{TV}_H:=(I-N_H)^{-1}\sigma^2=\sum_{j=0}^{H-1}N_H^j\sigma^2,\qquad \mathrm{TV}_H(h,s)=\sum_{j=0}^{h-1}\bigl((P^\pi)^j\sigma^2(h-j,\cdot)\bigr)(s),
\end{equation}
and $\mathrm{TV}_{H,\max}:=\max_{(h,s)}\mathrm{TV}_H(h,s)$.
\end{assumption}

Define the linear map $\mathcal P:E_0\to E_0$ by
\begin{equation}
(\mathcal PY)(h,s):=\mathbb E\bigl[L_Y^hY(h-1,S')\mid S=s\bigr]\quad(h\ge2),\qquad(\mathcal PY)(1,s):=0.
\end{equation}
Taking conditional expectations in the first item gives $\mathcal O_HU-\mathcal O_HU'=\mathcal P(U-U')$ for $U,U'\in\mathcal D_H$, and Jensen's inequality then yields, with $q_Y(h,s):=\lVert Y(h,s)\rVert_2^2$ and $r_Y(h,s):=\lVert Y(h,s)\rVert_2$,
\begin{equation}\label{eq:fh-sq-mixture}
\lVert(\mathcal PY)(h,s)\rVert_2^2\le(N_Hq_Y)(h,s),\qquad\lVert(\mathcal PY)(h,s)\rVert_2\le(N_Hr_Y)(h,s).
\end{equation}

\noindent\textbf{Sampling and recursion.} In the i.i.d.\ regime, the samples $Y_k=(S_k,A_k,S'_k)$ are i.i.d.\ with $S_k\sim\rho$, and we set $\nu:=\rho$ and $\mathcal F_k:=\sigma(Y_0,\dots,Y_{k-1})$. In the episodic regime, let $\phi_k:=k\bmod H$ be the phase and
\begin{equation}
Z_k:=(\phi_k,S_k)\in\mathcal Z:=\{0,\dots,H-1\}\times\X.
\end{equation}
Episodes start from $S_0\sim\nu_0$. If $\phi_k<H-1$, then $S_{k+1}=S'_k$, and if $\phi_k=H-1$, then $S_{k+1}\sim\nu_0$ independently of everything else. We set
\begin{equation}
\mathcal F_k:=\sigma(S_0,Y_0,S_1,\dots,Y_{k-1},S_k),
\end{equation}
which contains $Z_k$. Given $\mathcal F_k$, the pair $(A_k,S'_k)$ has law $\pi(\cdot\mid S_k)P(\cdot\mid S_k,\cdot)$, and $\Pr(Z_{k+1}=\zeta\mid\mathcal F_k)=Q_H(Z_k,\zeta)$, where
\begin{equation}
Q_H\bigl((t,s),(t+1,s')\bigr):=P^\pi(s,s')\quad(t<H-1),\qquad Q_H\bigl((H-1,s),(0,s')\bigr):=\nu_0(s'),
\end{equation}
and all other entries zero. The distribution $\mu_{\mathcal Z}(t,s):=\rho_t(s)/H$ is stationary for $Q_H$ because $\rho_{t+1}=\rho_tP^\pi$ and $\rho_0=\nu_0$, and its marginal on $\X$ is $\bar\rho_H$. We set $\nu:=\bar\rho_H$. In both regimes, $U_0\in\mathcal D_H$ is deterministic, $\alpha\in(0,1]$, and
\begin{equation}
U_{k+1}=U_k+\alpha P^H_{S_k}\bigl(\widehat T_H(U_k;Y_k)-U_k(S_k)\bigr),\qquad W_k:=U_k-U^*\in E_0.
\end{equation}
We write $\nu_{\min}:=\min_s\nu(s)$, and let $D_\nu$ act on $E_0$ and on $\mathbb R^{\mathcal S_H}$ by multiplying block $(h,s)$ by $\nu(s)$.

\subsection{Deterministic propagation}

Define the averaged operator and the propagation matrix
\begin{equation}
\bar A:=I-\alpha D_\nu(I-\mathcal P)\ \text{ on }E_0,\qquad \mathcal R_H:=I-\alpha D_\nu(I-N_H)\in\mathbb R^{\mathcal S_H\times\mathcal S_H},
\end{equation}
and the epoch length
\begin{equation}
L_H:=\Bigl\lceil\frac{2(H+8)}{\alpha\nu_{\min}}\Bigr\rceil.
\end{equation}

\begin{lemma}[Fixed-horizon propagation]
\label{lem:fh-propagation}
For every $Y\in E_0$, $n\ge0$, and $(h,s)\in\mathcal S_H$:
\begin{enumerate}
\item $\lVert(\bar A^nY)(h,s)\rVert_2\le(\mathcal R_H^nr_Y)(h,s)$ and $\lVert(\bar A^nY)(h,s)\rVert_2^2\le(\mathcal R_H^nq_Y)(h,s)$;
\item $\mathcal R_H$ is entrywise nonnegative with row sums at most $1$, $\sum_{n\ge0}\mathcal R_H^nD_\nu=\alpha^{-1}(I-N_H)^{-1}$, and $((I-N_H)^{-1}\bm1)(h,s)=h\le H$;
\item $(\mathcal R_H^n\bm1)(h,s)\le\Pr(\mathrm{Bin}(n,\alpha\nu_{\min})\le h-1)$, so in particular $\lVert\bar A^nY\rVert_{2,\infty}\le\lVert Y\rVert_{2,\infty}$ for all $n$, and $\lVert\bar A^{L_H}Y\rVert_{2,\infty}\le\lVert Y\rVert_{2,\infty}/8$;
\item entrywise $\lvert\mathcal R_H^n-\mathcal R_H^{n+1}\rvert\le\alpha\mathcal R_H^nD_\nu(I+N_H)$, and
\begin{equation}
\sum_{n\ge0}\sum_{(h',x)\in\mathcal S_H}\bigl\lvert\mathcal R_H^n-\mathcal R_H^{n+1}\bigr\rvert_{(h,s),(h',x)}\le2H.
\end{equation}
\end{enumerate}
\end{lemma}

\begin{proof}
To prove the first statement, note that for each block,
\begin{equation}
(\bar AY)(h,u)=(1-\alpha\nu(u))Y(h,u)+\alpha\nu(u)(\mathcal PY)(h,u)
\end{equation}
is a convex combination because $\alpha\nu(u)\le1$. The triangle inequality, convexity of $\lVert\cdot\rVert_2^2$, and \eqref{eq:fh-sq-mixture} give, entrywise,
\begin{equation}
r_{\bar AY}\le\mathcal R_Hr_Y,\qquad q_{\bar AY}\le\mathcal R_Hq_Y,
\end{equation}
and the claim follows by induction because $\mathcal R_H$ is entrywise nonnegative.

For the second statement, the diagonal entries of $\mathcal R_H$ are $1-\alpha\nu(u)\ge0$, its other entries are those of $\alpha D_\nu N_H\ge0$, and the row sum of row $(h,u)$ is $1-\alpha\nu(u)\mathds{1}\{h=1\}\le1$. If the blocks are ordered by $h$, then $\mathcal R_H$ is block lower triangular with diagonal blocks $\operatorname{diag}(1-\alpha\nu(u))_u$, so its spectral radius is $1-\alpha\nu_{\min}<1$, the Neumann series converges, and
\begin{equation}
\sum_{n\ge0}\mathcal R_H^n=\bigl(\alpha D_\nu(I-N_H)\bigr)^{-1}=\alpha^{-1}(I-N_H)^{-1}D_\nu^{-1}.
\end{equation}
Finally, $(N_H^j\bm1)(h,s)=\mathds{1}\{h>j\}$, so $((I-N_H)^{-1}\bm1)(h,s)=h$.

For the third statement, let $p:=\alpha\nu_{\min}$, $B_n\sim\mathrm{Bin}(n,p)$, and
\begin{equation}
f_n(h):=\Pr(B_n\le h-1)\quad(h\ge1),\qquad f_n(0):=0.
\end{equation}
We show $(\mathcal R_H^n\bm1)(h,s)\le f_n(h)$ by induction on $n$. The case $n=0$ is $1\le1$. With $q:=\alpha\nu(u)\ge p$ and $h\ge1$, we have
\begin{equation}
(\mathcal R_H^{n+1}\bm1)(h,u)=(1-q)(\mathcal R_H^n\bm1)(h,u)+q\,\mathds{1}\{h\ge2\}\sum_{s'}P^\pi(u,s')(\mathcal R_H^n\bm1)(h-1,s')\le(1-q)f_n(h)+qf_n(h-1).
\end{equation}
Since $f_n(h-1)\le f_n(h)$, the right-hand side is nonincreasing in $q$, so it is at most
\begin{equation}
(1-p)f_n(h)+pf_n(h-1)=f_{n+1}(h),
\end{equation}
where the identity follows by conditioning on the last Bernoulli trial. Together with the first statement, this gives
\begin{equation}
\lVert\bar A^nY\rVert_{2,\infty}\le f_n(H)\lVert Y\rVert_{2,\infty}\le\lVert Y\rVert_{2,\infty}.
\end{equation}
For $n=L_H$, we have $np\ge2(H+8)$, so $H-1<np/2$, and the Chernoff bound $\Pr(B_n\le np/2)\le e^{-np/8}$ gives
\begin{equation}
f_{L_H}(H)\le e^{-(H+8)/4}\le e^{-9/4}<\frac18.
\end{equation}

For the fourth statement, we have
\begin{equation}
\mathcal R_H^n-\mathcal R_H^{n+1}=\alpha\mathcal R_H^nD_\nu(I-N_H),
\end{equation}
whose entries are bounded in absolute value by those of $\alpha\mathcal R_H^nD_\nu(I+N_H)$ because $\mathcal R_H^n,D_\nu,N_H\ge0$. Since $(I+N_H)\bm1\le2\cdot\bm1$, summing row $(h,s)$ over all blocks and over $n$, and using the second statement, gives at most $2((I-N_H)^{-1}\bm1)(h,s)\le2H$.
\end{proof}

Thus, Lemma~\ref{lem:propagation}(2)--(5) holds with $\mathcal V$ and $\mathcal R$ both replaced by $\mathcal R_H$, the resolvents $(I-\beta^2P^\pi)^{-1}$ and $(I-\beta P^\pi)^{-1}$ both replaced by $(I-N_H)^{-1}$, and $1/(1-\beta^2)$ and $1/(1-\beta)$ both replaced by $H$. The geometric contraction of Lemma~\ref{lem:propagation}(1) is replaced by the third statement.

\subsection{Poisson solution for the episodic chain}

\begin{lemma}
\label{lem:fh-poisson}
In the episodic regime, let $C_H^\star:=2H$. Then the following hold.
\begin{enumerate}
\item For $g:\mathcal Z\to\mathbb R$, let $\tilde g:=g-\mu_{\mathcal Z}(g)$ and
\begin{equation}
\varphi_g(t,s):=\sum_{u=t}^{H-1}\bigl((P^\pi)^{u-t}\tilde g(u,\cdot)\bigr)(s).
\end{equation}
Then $g\mapsto\varphi_g$ is linear, and
\begin{equation}
\varphi_g-Q_H\varphi_g=\tilde g,\qquad\lVert\varphi_g\rVert_\infty\le H\lVert\tilde g\rVert_\infty.
\end{equation}
\item For $x\in\X$, let $\bar z(\cdot,x):=\varphi_g$ for $g(t,s)=\mathds{1}\{s=x\}$, and
\begin{equation}
\bar v_x(\zeta):=\bigl(Q_H\bar z(\cdot,x)^2\bigr)(\zeta)-\bigl(Q_H\bar z(\cdot,x)\bigr)(\zeta)^2.
\end{equation}
Writing $s(\zeta)$ for the state component of $\zeta$, we have
\begin{equation}
\bar z(\zeta,x)-\bigl(Q_H\bar z(\cdot,x)\bigr)(\zeta)=\mathds{1}\{s(\zeta)=x\}-\bar\rho_H(x),\qquad\lvert\bar z\rvert\le H,
\end{equation}
and
\begin{equation}
\sum_\zeta\mu_{\mathcal Z}(\zeta)\lvert\bar z(\zeta,x)\rvert\le2H\bar\rho_H(x),\qquad\sum_\zeta\mu_{\mathcal Z}(\zeta)\bar v_x(\zeta)\le6H\bar\rho_H(x),\qquad\sum_\zeta\mu_{\mathcal Z}(\zeta)\bar z(\zeta,x)^2\le2H^2\bar\rho_H(x).
\end{equation}
In particular, the inequalities of Lemma~\ref{lem:poisson-facts}(3) and~(4) hold for $\bar z$, $\bar v_x$, $\mu_{\mathcal Z}$, and $\bar\rho_H$ with $C_Q^\star$ replaced by $C_H^\star$.
\item Lemma~\ref{lem:markov-bernstein} holds for the chain $(Z_k)$ and functions $h_j:\mathcal Z\to[0,\Lambda]$, with $\mathcal F_{t_0}$ as above, $\mu_{\X}$ replaced by $\mu_{\mathcal Z}$, and $C_Q^{\mathrm{Pois}}$ replaced by $H$.
\end{enumerate}
\end{lemma}

\begin{proof}
To prove the first statement, note that linearity is clear. For $t<H-1$, we have
\begin{equation}
(Q_H\varphi_g)(t,s)=\sum_{s'}P^\pi(s,s')\varphi_g(t+1,s')=\sum_{u=t+1}^{H-1}\bigl((P^\pi)^{u-t}\tilde g(u,\cdot)\bigr)(s),
\end{equation}
so $\varphi_g(t,s)-(Q_H\varphi_g)(t,s)=\tilde g(t,s)$. For $t=H-1$, we have $\varphi_g(H-1,s)=\tilde g(H-1,s)$ and
\begin{equation}
(Q_H\varphi_g)(H-1,s)=\sum_{u=0}^{H-1}\nu_0(P^\pi)^u\tilde g(u,\cdot)=\sum_{u=0}^{H-1}\rho_u\bigl(\tilde g(u,\cdot)\bigr)=H\mu_{\mathcal Z}(\tilde g)=0.
\end{equation}
The sum defining $\varphi_g$ has at most $H$ terms, each bounded by $\lVert\tilde g\rVert_\infty$.

For the first moments in the second statement, note that $\mu_{\mathcal Z}(g)=\bar\rho_H(x)$, so the Poisson identity is the first statement, and
\begin{equation}
\bar z\bigl((t,s),x\bigr)=\sum_{u=t}^{H-1}(P^\pi)^{u-t}(s,x)-(H-t)\bar\rho_H(x)
\end{equation}
is a difference of two numbers in $[0,H-t]$, so $\lvert\bar z\rvert\le H$. Since $\sum_s\rho_t(s)(P^\pi)^{u-t}(s,x)=\rho_u(x)$ and $\sum_{u\ge t}\rho_u(x)\le H\bar\rho_H(x)$,
\begin{equation}
\sum_\zeta\mu_{\mathcal Z}(\zeta)\lvert\bar z(\zeta,x)\rvert\le\frac1H\sum_{t=0}^{H-1}\Bigl(\sum_{u=t}^{H-1}\rho_u(x)+(H-t)\bar\rho_H(x)\Bigr)\le2H\bar\rho_H(x).
\end{equation}

For the second moments in the second statement, write $F:=\bar z(\cdot,x)$ and $G(\zeta):=\mathds{1}\{s(\zeta)=x\}-\bar\rho_H(x)$, so that $Q_HF=F-G$. Stationarity gives $\mu_{\mathcal Z}(Q_HF^2)=\mu_{\mathcal Z}(F^2)$, hence
\begin{equation}
\sum_\zeta\mu_{\mathcal Z}(\zeta)\bar v_x(\zeta)=2\mu_{\mathcal Z}(FG)-\mu_{\mathcal Z}(G^2)\le2\mu_{\mathcal Z}\bigl(\lvert F\rvert\lvert G\rvert\bigr).
\end{equation}
Since $\lvert G\rvert\le\mathds{1}\{s(\cdot)=x\}+\bar\rho_H(x)$, $\lvert F\rvert\le H$, and $\mu_{\mathcal Z}(s(\cdot)=x)=\bar\rho_H(x)$, the right-hand side is at most
\begin{equation}
2\bigl(H\bar\rho_H(x)+2H\bar\rho_H(x)^2\bigr)\le6H\bar\rho_H(x).
\end{equation}
Finally,
\begin{equation}
\sum_\zeta\mu_{\mathcal Z}(\zeta)F(\zeta)^2\le H\sum_\zeta\mu_{\mathcal Z}(\zeta)\lvert F(\zeta)\rvert\le2H^2\bar\rho_H(x).
\end{equation}
With $C_H^\star=2H$, these bounds imply
\begin{equation}
\lvert\bar z\rvert\le2C_H^\star,\qquad\sum\mu_{\mathcal Z}\lvert\bar z\rvert\le2C_H^\star\bar\rho_H(x),\qquad\sum\mu_{\mathcal Z}\bar v_x\le4C_H^\star\bar\rho_H(x),\qquad\sum\mu_{\mathcal Z}\bar z^2\le4(C_H^\star)^2\bar\rho_H(x).
\end{equation}

For the third statement, note that the proof of Lemma~\ref{lem:markov-bernstein} uses the chain only through the Markov property $\Pr(Z_{t_0+j+1}=\cdot\mid\mathcal F_{t_0+j})=Q_H(Z_{t_0+j},\cdot)$, and uses $Z_Q$ only through Lemma~\ref{lem:poisson-facts}(2), namely, through functions $\varphi_j$ that solve $\varphi_j-Q\varphi_j=h_j-\mu(h_j)$ and satisfy
\begin{equation}
\lVert\varphi_j\rVert_\infty\le C\lVert h_j-\mu(h_j)\rVert_\infty,\qquad\lVert\varphi_j-\varphi_{j-1}\rVert_\infty\le2C\lVert h_j-h_{j-1}\rVert_\infty.
\end{equation}
The functions $\varphi_j:=\varphi_{h_j}$ of the first statement have these properties with $C=H$, by linearity and because $\lVert\tilde g\rVert_\infty\le2\lVert g\rVert_\infty$.
\end{proof}

\subsection{Main result}

\begin{theorem}[Fixed-horizon finite-iteration bound]
\label{thm:fh-sharp-app}
Suppose Assumption~\ref{ass:fh} holds, let $\ell\ge1$, and let
\begin{equation}
\alpha\le\frac1{2^{11}C_FH\ell},\qquad\text{and, in the episodic regime,}\qquad\alpha\le\frac{2^{-21}}{C_FH^2\ell}.
\end{equation}
Define
\begin{equation}
N^H:=\sqrt{6C_F\alpha\,\mathrm{TV}_{H,\max}\ell}+2C_F\alpha b\ell+\mathds{1}_{\mathrm{ep}}\,20\alpha H^2b,
\end{equation}
where $\mathds{1}_{\mathrm{ep}}$ is $1$ in the episodic regime and $0$ otherwise. Let $D_0\ge\max\{\lVert W_0\rVert_{2,\infty},2N^H\}$ and $K\in\mathbb N$. Then, with probability at least $1-16KL_HH\lvert\X\rvert(\log_2L_H+3)e^{-\ell}$,
\begin{equation}
\lVert W_{KL_H}\rVert_{2,\infty}\le2N^H+2^{-K}D_0.
\end{equation}
\end{theorem}

\begin{proof}
We follow the proof of Theorem~\ref{thm:sharp-app} line by line, with the following replacements:
\begin{itemize}
\item the index $s\in\X$ of a block by $(h,s)\in\mathcal S_H$, the operator $P_{S_t}$ by $P^H_{S_t}$, and $L$ by $L_H$;
\item $\mu_{\X}$ by $\bar\rho_H$ where it is the diagonal of $D_\nu$, and by $\mu_{\mathcal Z}$ where it is the stationary law of the chain;
\item $\mathds{1}_{\mathrm{mk}}$ by $\mathds{1}_{\mathrm{ep}}$.
\end{itemize}
For $c\in\mathbb R^{\X}$, we write $(c\odot Y)(h,x):=c(x)Y(h,x)$. We list the ingredients that change and check that each estimate survives.

We first decompose the error. Let $y_t:=\widehat T_H(U_t;Y_t)-U_t(S_t)$, whose $h$th component is
\begin{equation}
y_t(h)=L_{Y_t}^hW_t(h-1,S'_t)+\xi^h(Y_t)-W_t(h,S_t),
\end{equation}
the first term being absent for $h=1$, and let $g_t:=(\mathcal P-I)W_t$. The conditional mean of $P^H_{S_t}y_t$ is $\rho\odot g_t$ in the i.i.d.\ regime and $\bm1_{S_t}\odot g_t$ in the episodic regime, because $(A_t,S'_t)$ given $\mathcal F_t$ has law $\pi(\cdot\mid S_t)P(\cdot\mid S_t,\cdot)$. In the episodic regime, we set
\begin{equation}
\zeta_t:=P^H_{S_t}\bigl(y_t+W_t(S_t)-(\mathcal PW_t)(S_t)\bigr),\qquad e_t:=(\bm1_{S_t}-\bar\rho_H)\odot g_t.
\end{equation}
Note that $\zeta_t$ uses the recorded successor $S'_t$ and not $S_{t+1}$. Then \eqref{eq:sharp-decomp} holds with $\mathbb E[\zeta_t\mid\mathcal F_t]=0$.

The operator bounds change as follows. By \eqref{eq:fh-sq-mixture}, on $\{j<\tau\}$, we have $\lVert(\mathcal PW_j)(h,x)\rVert_2\le2D$, so
\begin{equation}
\lVert g_j\rVert_{2,\infty}\le4D,\qquad\lVert y_j(h)\rVert_2\le b+4D,
\end{equation}
and
\begin{equation}
\mathbb E\bigl[\lVert y_j(h)\rVert_2^2\mid S_j=u\bigr]\le2\sigma^2(h,u)+32D^2,\qquad\mathbb E\bigl[\lVert\zeta_j(h,S_j)\rVert_2^2\mid\mathcal F_{t_k+j}\bigr]\le2\sigma^2(h,S_j)+8D^2.
\end{equation}
These are the bounds of the discounted proof with $\beta=1$. For $u\in\X$ and a stack $y=(y(h))_h$ at state $u$, we have
\begin{equation}
(\mathcal PP^H_uy)(h,x)=\mathbb E\bigl[\mathds{1}\{S'=u\}L_Y^hy(h-1)\mid S=x\bigr],
\end{equation}
so
\begin{equation}
\bigl\lVert\bigl((\mathcal P-I)P^H_uy\bigr)(h,x)\bigr\rVert_2\le\bigl(P^\pi(x,u)+\mathds{1}\{x=u\}\bigr)\max_{h'}\lVert y(h')\rVert_2,
\end{equation}
which replaces the bound with $\beta Q(x,u)+\mathds{1}\{x=u\}$ in the discounted proof.

For propagation, every use of Lemma~\ref{lem:propagation}(2)--(5) is replaced by Lemma~\ref{lem:fh-propagation}(1), (2), and~(4), so every factor $1/(1-\beta^2)$ or $1/(1-\beta)$ that arises from a resolvent or from Lemma~\ref{lem:propagation}(5) becomes $H$. The two uses of Lemma~\ref{lem:propagation}(1), in the bounds on the increments $d_j$ and in the closing step, require only
\begin{equation}
\lVert\bar A^nY\rVert_{2,\infty}\le\lVert Y\rVert_{2,\infty},\qquad\lVert\bar A^{L_H}Y\rVert_{2,\infty}\le\frac{\lVert Y\rVert_{2,\infty}}8,
\end{equation}
which is Lemma~\ref{lem:fh-propagation}(3). The union bound runs over $n\le L_H$ and the $H\lvert\X\rvert$ blocks. The functions to which Lemma~\ref{lem:markov-bernstein} is applied are now
\begin{equation}
h_j(\zeta):=\sum_{(h',x)\in\mathcal S_H}(\mathcal R_H^{n-1-j})_{(h,s),(h',x)}F(h',x,\zeta)
\end{equation}
for the corresponding weight $F$ of the discounted proof, in which $f(u)$ becomes $f(h',s(\zeta))\mathds{1}\{x=s(\zeta)\}$ and $\beta Q(x,u)+\mathds{1}\{x=u\}$ becomes $P^\pi(x,s(\zeta))+\mathds{1}\{x=s(\zeta)\}$. They take values in the same ranges because the rows of $\mathcal R_H^m$ sum to at most $1$, and their $\mu_{\mathcal Z}$-means are bounded as in the discounted proof because $\mu_{\mathcal Z}$ is stationary for $Q_H$ with state marginal $\bar\rho_H$, which is the diagonal of $D_\nu$.

For the drift fluctuation in the episodic regime, $z(S_j,\cdot)$, $(Qz)(S_j,\cdot)$, and $v_x(S_j)$ are replaced by $\bar z(Z_j,\cdot)$, $(Q_H\bar z)(Z_j,\cdot)$, and $\bar v_x(Z_j)$, and
\begin{equation}
m_j:=\bar z(Z_{j+1},\cdot)-(Q_H\bar z)(Z_j,\cdot),
\end{equation}
which has conditional mean zero given $\mathcal F_{t_k+j}$. By Lemma~\ref{lem:fh-poisson}(2), $e_j=\Gamma_j-\Gamma_j^++m_j\odot g_j$ and \eqref{eq:mk-parts} hold, and every property of $z$ and $v_x$ used in the discounted proof holds with $C_Q^\star$ replaced by $C_H^\star$. Lemma~\ref{lem:markov-bernstein} is replaced by Lemma~\ref{lem:fh-poisson}(3), whose constant $H$ is at most $C_H^\star$.

For the step-size conditions, note that the discounted proof uses its step-size assumptions only through the inequalities
\begin{equation}
\frac{C_F\alpha\ell}{1-\beta^2}\le2^{-11},\qquad\frac{C_FC_Q^\star\alpha\ell}{1-\beta^2}\le\frac{C_FC_Q^\star\alpha\ell}{1-\beta}\le2^{-20},\qquad\frac{C_F^2(C_Q^\star)^3\alpha^2\ell^2}{1-\beta}\le2^{-24},
\end{equation}
together with $1-\beta^2\le1$ and $C_F\ell\ge1$. After the replacements above, these read
\begin{equation}
C_F\alpha\ell H\le2^{-11},\qquad2C_F\alpha\ell H^2\le2^{-20},\qquad8C_F^2\alpha^2\ell^2H^4\le2^{-24}.
\end{equation}
The first is the assumption on $\alpha$, the second is the episodic assumption, and the third follows from the second because $8\cdot2^{-42}\le2^{-24}$ and $C_F\ge1$. The Markovian noise term $10\alpha C_Q^\star b/(1-\beta)$ of $N$ becomes $10\alpha C_H^\star bH=20\alpha H^2b$.

With these replacements, every inequality of the proof of Theorem~\ref{thm:sharp-app} holds with the same numerical constants, and the epoch argument gives the claim.
\end{proof}

\begin{proof}[Proof of Theorem~\ref{thm:undisc}]
Fix $\star\in\{\mathrm C,\mathrm M\}$. The parts of the proofs of Lemmas~\ref{lem:fh-ctd-tv} and~\ref{lem:fh-mtd-tv} that verify the structural conditions use neither the endpoint assumption on the supports nor the covering radius, so Assumption~\ref{ass:fh} holds with $b=B_{H,\star}$, where $B_{H,\mathrm C}$ bounds the fixed-point noise by the proof of Proposition~\ref{prop:ctd-block-bound} applied horizonwise. Its total variance is $\mathrm{TV}^H_{\star,\max}$, and the horizonwise isometries give $e^H_{\star,k}=\lVert W_k\rVert_{2,\infty}$. Theorem~\ref{thm:fh-sharp-app} with these constants is Theorem~\ref{thm:undisc}, since $N^H$ equals $N^H_\star$ under i.i.d.\ sampling and $N^H_\star+20\alpha H^2B_{H,\star}$ under episodic sampling.
\end{proof}

\subsection{Fixed-horizon total variance for CTD}

\begin{lemma}[Fixed-horizon Cram\'er total variance]
\label{lem:fh-ctd-tv}
Consider fixed-horizon CTD with rewards in $[0,1]$ and supports with $\theta_{h,1}(s)=0$ and $\theta_{h,d}(s)=h$ for every $(h,s)$, and let
\begin{equation}
\Delta_{\max}:=\max_{h,s,j}\bigl(\theta_{h,j+1}(s)-\theta_{h,j}(s)\bigr).
\end{equation}
Then Assumption~\ref{ass:fh} holds with $b=B_{H,\mathrm C}=\sqrt H$, and
\begin{equation}
\mathrm{TV}_H(h,s)\le2h+2h^2\Delta_{\max}\le\bar\tau_HH,\qquad\bar\tau_H:=2+2H\Delta_{\max},\qquad\text{for every }(h,s)\in\mathcal S_H.
\end{equation}
Moreover, for every $h$,
\begin{equation}
\max_s\ell_{\mathrm C}\bigl(\eta^{*,h}(s),\eta^{\pi,h}(s)\bigr)\le\sqrt{h\Delta_{\max}}.
\end{equation}
\end{lemma}

\begin{proof}
We first verify the structural conditions of Assumption~\ref{ass:fh}. The first layer of the target is
\begin{equation}
I_{H,\mathrm C,1,s}\bigl(\Pi_{\mathrm C}^{\Theta_1(s)}\delta_{R(s,a)}\bigr),
\end{equation}
which does not depend on $U$. For $h\ge2$, the map
\begin{equation}
\mu\mapsto I_{H,\mathrm C,h,s}\bigl(\Pi_{\mathrm C}^{\Theta_h(s)}((f_{r,1})_\#\mu)\bigr)
\end{equation}
is affine on probability laws on $\Theta_{h-1}(s')$, and $I_{H,\mathrm C,h-1,s'}$ is an affine bijection onto the corresponding block of $\mathcal D_{H,\mathrm C}$, so the composition is affine with linear part $L^h_y$ of norm at most $1$ by Lemma~\ref{lem:ctd-projection} with $\gamma'=1$. The fixed-point noise has conditional mean zero because
\begin{equation}
U^*(h,s)=I_{H,\mathrm C,h,s}\bigl(\Pi_{\mathrm C}^{\Theta_h(s)}(T_H^\pi\eta^*)^h(s)\bigr)
\end{equation}
is the conditional mean of $\widehat T_H^h(U^*;Y)$. By the proof of Proposition~\ref{prop:ctd-block-bound}, the coordinates of two embedded laws at $(h,s)$ lie in $[0,\sqrt{\Delta_j}]$, so $\lVert\xi^h\rVert_2^2\le h\le H$.

Next, we bound the variance. Fix $(h,u)$ and let
\begin{equation}
\nu_\omega:=(f_{R(u,A),1})_\#\eta^{*,h-1}(S'),\qquad(A,S')\sim\pi(\cdot\mid u)P(\cdot\mid u,A),\qquad\eta^{*,0}:=\delta_0.
\end{equation}
As in the proof of Lemma~\ref{lem:ctd-tv}, Jensen's inequality on each cell gives
\begin{equation}
\sigma^2(h,u)\le v(h,u):=\int_{\mathbb R}\operatorname{Var}\bigl(F_{\nu_\omega}(z)\bigr)\,\mathrm dz.
\end{equation}

We use a martingale of conditional distribution functions. Fix $(h,s)$ and run the chain from $S_0=s$ with $A_k\sim\pi(\cdot\mid S_k)$ and $S_{k+1}\sim P(\cdot\mid S_k,A_k)$. Let
\begin{equation}
G_k:=\sum_{i<k}R(S_i,A_i),\qquad\mathcal H_k:=\sigma(S_0,A_0,\dots,S_k),\qquad Y_k(z):=F_{\eta^{*,h-k}(S_k)}(z-G_k)\quad(0\le k\le h,\ z\in\mathbb R).
\end{equation}
Since $F_{(f_{r,1})_\#\eta}(z)=F_\eta(z-r)$, we have, for $k<h$,
\begin{equation}
\mathbb E[Y_{k+1}(z)\mid\mathcal H_k]=F_{(T_H^\pi\eta^*)^{h-k}(S_k)}(z-G_k).
\end{equation}
Write $Y_{k+1}-Y_k=D_{k+1}+A_k$ with
\begin{equation}
D_{k+1}:=Y_{k+1}-\mathbb E[Y_{k+1}\mid\mathcal H_k],\qquad A_k(z):=\bigl(F_{(T_H^\pi\eta^*)^{h-k}(S_k)}-F_{\eta^{*,h-k}(S_k)}\bigr)(z-G_k).
\end{equation}
The substitution $z=G_k+w$ gives
\begin{equation}
\int\mathbb E\bigl[D_{k+1}(z)^2\mid\mathcal H_k\bigr]\,\mathrm dz=v(h-k,S_k),
\end{equation}
and the increments $D_{k+1}(z)$ are orthogonal martingale differences, so
\begin{equation}
\sum_{k<h}\mathbb E[v(h-k,S_k)]=\int\mathbb E\Bigl[\Bigl(Y_h-Y_0-\sum_{k<h}A_k\Bigr)^2\Bigr]\mathrm dz\le2\int\mathbb E\bigl[(Y_h-Y_0)^2\bigr]\,\mathrm dz+2\,\mathbb E\Bigl[\int\Bigl(\sum_{k<h}A_k\Bigr)^2\mathrm dz\Bigr].
\end{equation}

We now bound the two terms. Here, $Y_h(z)=\mathds{1}\{G_h\le z\}$ and $Y_0=F_{\eta^{*,h}(s)}$ take values in $[0,1]$, vanish for $z<0$, and equal $1$ for $z\ge h$, because $G_h\in[0,h]$ and $\eta^{*,h}(s)$ is supported on $[0,h]$, so the first integral is at most $h$. For the second, note that $(T_H^\pi\eta^*)^{h'}(u)$ is supported in $[0,h']=[\theta_{h',1}(u),\theta_{h',d}(u)]$ and that $\eta^{*,h'}(u)$ is its projection, so, as in the proof of Lemma~\ref{lem:ctd-tv},
\begin{equation}
\int\bigl(F_{(T_H^\pi\eta^*)^{h'}(u)}-F_{\eta^{*,h'}(u)}\bigr)^2\,\mathrm dz\le\Delta_{\max}.
\end{equation}
Thus, pathwise, $\lVert A_k\rVert_{L^2(\mathrm dz)}\le\sqrt{\Delta_{\max}}$, and Minkowski's inequality gives
\begin{equation}
\int\Bigl(\sum_{k<h}A_k\Bigr)^2\mathrm dz\le h^2\Delta_{\max}.
\end{equation}
Since
\begin{equation}
\bigl((P^\pi)^j\sigma^2(h-j,\cdot)\bigr)(s)=\mathbb E\bigl[\sigma^2(h-j,S_j)\bigr]\le\mathbb E\bigl[v(h-j,S_j)\bigr],
\end{equation}
we obtain $\mathrm{TV}_H(h,s)\le2h+2h^2\Delta_{\max}$.

We now bound the projection bias. Let
\begin{equation}
e_h:=\max_u\ell_{\mathrm C}\bigl(\eta^{*,h}(u),\eta^{\pi,h}(u)\bigr),
\end{equation}
so that $e_0=0$. The bias argument of Lemma~\ref{lem:ctd-tv}, with the pushforward $f_{r,1}$, which preserves $\ell_{\mathrm C}$, and the supports $[0,h]$ of $(T_H^\pi\eta^*)^h(u)$ and $\eta^{\pi,h}(u)$, gives
\begin{equation}
e_h^2\le e_{h-1}^2+\Delta_{\max},\qquad\text{hence}\qquad e_h^2\le h\Delta_{\max}.
\end{equation}
\end{proof}

\subsection{Fixed-horizon total variance for MTD}

We use the notation $d_\kappa$, $\bar\kappa$, $\psi$ of Appendix~\ref{app:sharp}, and set $\mathcal B_h:=[0,h]^q$, which contains the support of every $h$-step return law $\eta^{\pi,h}(s)$, and
\begin{equation}
h_{\Theta,H}:=\max_{(h,s)\in\mathcal S_H}\sup_{z\in\mathcal B_h}\min_{i\le d}d_\kappa(z,\theta_{h,i}(s)).
\end{equation}
By homogeneity and shift invariance,
\begin{equation}
\sup_{x,y\in\mathcal B_h}d_\kappa(x,y)=\bar\kappa h^c.
\end{equation}

\begin{lemma}[Fixed-horizon MMD total variance]
\label{lem:fh-mtd-tv}
Consider fixed-horizon MTD with rewards in $[0,1]^q$. Then Assumption~\ref{ass:fh} holds, and
\begin{equation}
b\le\bar B_{H,\mathrm M}:=\sqrt{\bar\kappa}H^{c/2}+2\sqrt{H\,h_{\Theta,H}},\qquad\mathrm{TV}_H(h,s)\le\bar\kappa h^c+h^2h_{\Theta,H}\le\bar\tau_{H,\mathrm M}H^c,
\end{equation}
where $\bar\tau_{H,\mathrm M}:=\bar\kappa+H^{2-c}h_{\Theta,H}$. Moreover, for every $h$, the projection bias satisfies
\begin{equation}
\max_s\mathrm{MMD}_\kappa\bigl(\eta^{*,h}(s),\eta^{\pi,h}(s)\bigr)\le\sqrt{h\,h_{\Theta,H}}.
\end{equation}
\end{lemma}

\begin{proof}
We first verify the structural conditions of Assumption~\ref{ass:fh} as in the proof of Lemma~\ref{lem:mtd-tv}, with $f_{r,1}$ in place of $f_{r,\gamma}$. The first layer does not depend on $U$. For $h\ge2$, shift invariance gives
\begin{equation}
\bigl\lVert\psi((f_{r,1})_\#\mu)-\psi((f_{r,1})_\#\nu)\bigr\rVert=\lVert\psi(\mu)-\psi(\nu)\rVert\qquad\text{for laws of equal mass,}
\end{equation}
so Proposition~\ref{prop:mtd-proj-affine} gives an affine target with linear part of norm at most $1$. The noise has conditional mean zero by Proposition~\ref{prop:mtd-unbiased} applied horizonwise and the fixed-point property.

Next, we bound the projection bias. Let
\begin{equation}
e_h:=\max_u\mathrm{MMD}_\kappa\bigl(\eta^{*,h}(u),\eta^{\pi,h}(u)\bigr),
\end{equation}
so that $e_0=0$, since both layers equal $\delta_0$. As in the proof of Lemma~\ref{lem:mtd-tv}, Pythagoras' theorem, nonexpansiveness of the projection, the mixture bound, and shift invariance give
\begin{equation}
e_h^2\le e_{h-1}^2+h_{\Theta,H},\qquad\text{hence}\qquad e_h^2\le h\,h_{\Theta,H},
\end{equation}
because $\eta^{\pi,h}(u)$ is supported on $\mathcal B_h$ and the nearest-atom map onto $\Theta_h(u)$ moves each point of $\mathcal B_h$ by at most $\sqrt{h_{\Theta,H}}$ in $\sqrt{d_\kappa}$.

We now bound the noise. Write
\begin{equation}
\nu_y:=(f_{R(s,a),1})_\#\eta^{*,h-1}(s'),\qquad\nu^\pi_y:=(f_{R(s,a),1})_\#\eta^{\pi,h-1}(s').
\end{equation}
The argument of Lemma~\ref{lem:mtd-tv} gives
\begin{equation}
\lVert\xi^h(y)\rVert_2\le\sup_{y'}\mathrm{MMD}_\kappa(\nu_y,\nu_{y'})\le2e_{h-1}+\sqrt{\bar\kappa}h^{c/2},
\end{equation}
because $\nu^\pi_y$ and $\nu^\pi_{y'}$ both live on $\mathcal B_h$. This gives $b\le\bar B_{H,\mathrm M}$.

Next, we bound the variance. Let
\begin{equation}
\tilde v(h,s):=\mathbb E\bigl[\lVert\psi(\nu^\pi_Y)-\psi(\eta^{\pi,h}(s))\rVert^2\mid S=s\bigr].
\end{equation}
As in the proof of Lemma~\ref{lem:mtd-tv},
\begin{equation}
\sigma(h,s)\le\sqrt{\tilde v(h,s)}+e_{h-1},\qquad\text{so}\qquad\sigma^2(h,s)\le2\tilde v(h,s)+2(h-1)h_{\Theta,H}.
\end{equation}
We run the chain from $S_0=s$ as in the proof of Lemma~\ref{lem:fh-ctd-tv}, and, for $0\le k\le h$, let
\begin{equation}
\tilde Y_k:=\psi\bigl((g_k)_\#\eta^{\pi,h-k}(S_k)\bigr),\qquad g_k(z):=G_k+z.
\end{equation}
Given $\mathcal H_k$, the law of $G_h$ is $(g_k)_\#\eta^{\pi,h-k}(S_k)$, so $\tilde Y_k=\mathbb E[\psi(\delta_{G_h})\mid\mathcal H_k]$ is a square-integrable martingale in $\mathcal H_\kappa$, and shift invariance gives
\begin{equation}
\mathbb E\bigl[\lVert\tilde Y_{k+1}-\tilde Y_k\rVert^2\mid\mathcal H_k\bigr]=\tilde v(h-k,S_k).
\end{equation}
Orthogonality of the increments gives
\begin{equation}
\sum_{k<h}\mathbb E\bigl[\tilde v(h-k,S_k)\bigr]=\mathbb E\bigl[\lVert\psi(\delta_{G_h})-\psi(\eta^{\pi,h}(s))\rVert^2\bigr]=\tfrac12\mathbb E\bigl[d_\kappa(G_h,G'_h)\bigr]\le\tfrac12\bar\kappa h^c,
\end{equation}
where $G'_h$ is an independent copy with law $\eta^{\pi,h}(s)$. Therefore,
\begin{equation}
\mathrm{TV}_H(h,s)=\sum_{k<h}\mathbb E\bigl[\sigma^2(h-k,S_k)\bigr]\le\bar\kappa h^c+2h_{\Theta,H}\sum_{k<h}(h-k-1)\le\bar\kappa h^c+h^2h_{\Theta,H}.
\end{equation}
\end{proof}

\subsection{Consequences for CTD and MTD}

\begin{proof}[Proof of Corollaries~\ref{cor:undisc} and~\ref{cor:undisc-mtd}]
We first treat CTD. By Lemma~\ref{lem:fh-ctd-tv}, Assumption~\ref{ass:fh} holds with $b=B_{H,\mathrm C}=\sqrt H$ and $\mathrm{TV}_{H,\max}\le\bar\tau_HH$, and, as in the proof of Corollary~\ref{cor:sharp-ctd}, $\lVert W_t\rVert_{2,\infty}\le B_{H,\mathrm C}$ for all $t$. Let $\varepsilon\in(0,1]$ and $\delta\in(0,1)$, so that $\varepsilon\le B_{H,\mathrm C}$ and $\varepsilon^2<\bar\tau_H$. Set $D_0:=B_{H,\mathrm C}$ and $K:=\lceil\log_2(2B_{H,\mathrm C}/\varepsilon)\rceil$, and let
\begin{equation}
\alpha(\ell):=\frac{\varepsilon^2}{2^{11}C_F\bar\tau_HH\ell}
\end{equation}
in the i.i.d.\ regime and
\begin{equation}
\alpha(\ell):=\min\Bigl\{\frac{\varepsilon^2}{2^{11}C_F\bar\tau_HH\ell},\ \frac{2^{-21}}{C_FH^2\ell},\ \frac{\varepsilon}{320H^2B_{H,\mathrm C}}\Bigr\}
\end{equation}
in the episodic regime. These choices satisfy the conditions of Theorem~\ref{thm:fh-sharp-app}, and each of the at most three summands of $N^H$ is at most $\varepsilon/16$. For the first two summands, this follows from
\begin{equation}
6C_F\alpha\,\bar\tau_HH\ell\le\frac{6\varepsilon^2}{2^{11}},\qquad2C_F\alpha B_{H,\mathrm C}\ell\le\frac{\varepsilon^2}{2^{10}\bar\tau_H\sqrt H},
\end{equation}
and for the third from $20\alpha H^2B_{H,\mathrm C}\le\varepsilon/16$. Hence, $N^H\le3\varepsilon/16$, $D_0\ge\varepsilon\ge2N^H$, and $2^{-K}D_0\le\varepsilon/2$, so, at $T=KL_H$ and with probability at least $1-16TH\lvert\X\rvert(\log_2T+3)e^{-\ell}$,
\begin{equation}
\lVert W_T\rVert_{2,\infty}\le\frac{7\varepsilon}8.
\end{equation} The choice of $\ell$ and the expectation bound follow as in the proof of Corollary~\ref{cor:sharp-ctd}. Since $\alpha\nu_{\min}\le1$,
\begin{equation}
L_H\le\frac{2H+17}{\alpha\nu_{\min}}\le\frac{19H}{\alpha\nu_{\min}},\qquad T\le\frac{19HK}{\alpha\nu_{\min}},
\end{equation}
and inserting the terms of $1/\alpha(\ell)$ gives the stated orders.

For MTD, by Lemma~\ref{lem:fh-mtd-tv}, Assumption~\ref{ass:fh} holds with $b\le\bar B_{H,\mathrm M}$ and $\mathrm{TV}_{H,\max}\le\bar\tau_{H,\mathrm M}H^c$. Set $D_0:=\max\{e^H_{\mathrm M,0},\bar B_{H,\mathrm M}\}$, $K:=\lceil\log_2(2D_0/\varepsilon)\rceil$, and choose $\alpha(\ell)$ as the minimum of $1/(2^{11}C_FH\ell)$ and the CTD choice with $\bar\tau_HH$ replaced by $\bar\tau_{H,\mathrm M}H^c$ and $B_{H,\mathrm C}$ by $\bar B_{H,\mathrm M}$. Since $\varepsilon\le\bar B_{H,\mathrm M}$ and
\begin{equation}
\bar B_{H,\mathrm M}^2\le2\bar\kappa H^c+8Hh_{\Theta,H}\le8\bar\tau_{H,\mathrm M}H^c,
\end{equation}
the second summand of $N^H$ satisfies
\begin{equation}
2C_F\alpha\bar B_{H,\mathrm M}\ell\le\frac{\varepsilon^2\bar B_{H,\mathrm M}}{2^{10}\bar\tau_{H,\mathrm M}H^c}\le\frac\varepsilon{2^7},
\end{equation}
and the probability bound follows as for CTD. For the expectation bound, let $M_h(t):=\max_s\lVert W_t(h,s)\rVert_2$ and $M_0:\equiv0$. The recursion gives
\begin{equation}
\lVert W_{t+1}(h,S_t)\rVert_2\le(1-\alpha)M_h(t)+\alpha\Bigl(\sup_{t'}M_{h-1}(t')+b\Bigr),\qquad\text{so}\qquad\sup_tM_h(t)\le\max\Bigl\{M_h(0),\sup_tM_{h-1}(t)+b\Bigr\},
\end{equation}
and, by induction on $h$, $\lVert W_t\rVert_{2,\infty}\le B_W:=D_0+H\bar B_{H,\mathrm M}$ for all $t$. Running the argument with $\min\{\delta,\varepsilon^2/B_W^2\}$ in place of $\delta$ gives $\mathbb E[(e^H_{\mathrm M,T})^2]\le2\varepsilon^2$ at the cost of a logarithmic factor.

It remains to pass to the true return laws. For CTD, the grid condition gives
\begin{equation}
\Delta_{\max}\le\frac H{d-1}\le\frac{\varepsilon^2}{4H},
\end{equation}
hence $\bar\tau_H\le2+\varepsilon^2/2\le4$ and, by Lemma~\ref{lem:fh-ctd-tv}, a bias of at most $\sqrt{H\Delta_{\max}}\le\varepsilon/2$. For MTD, Lemma~\ref{lem:fh-mtd-tv} gives a bias of at most $\sqrt{Hh_{\Theta,H}}\le\varepsilon/2$. Running the arguments above with $\varepsilon/2$ in place of $\varepsilon$ and using the triangle inequality gives the claims, as in the proof of Corollary~\ref{cor:sharp-ctd}. The $\bar W_1$ statement follows from $\bar W_1\le\sqrt h\,\ell_{\mathrm C}$ for laws on $[0,h]$.
\end{proof}

\section{Minimax lower bounds}\label{app:lower}

This appendix proves Theorem~\ref{thm:lower}. We use $\bar H$, $\kappa_1$, $e_1$, and $D$ as in Section~\ref{sec:lower}. An estimator is any measurable function of the observed transitions and of independent randomness, with values in the space on which the distance is defined, and $\mathsf P_i$ denotes the law of the observed transitions under the instance with index $i$.

We first relate the three distances to the difference of the means. Fix a smooth compactly supported function $\mathsf f:\mathbb R\to\mathbb R$ with $\mathsf f(u)=u$ for $u\in[0,1]$, write $\hat{\mathsf f}(\lambda):=\int\mathsf f(t)e^{-\mathrm i\lambda t}\,\mathrm dt$, and set
\begin{equation}\label{eq:lb-consts}
B_c:=\int_{\mathbb R}(1-\cos u)\lvert u\rvert^{-1-c}\,\mathrm du,\qquad J_c:=\int_{\mathbb R}\lvert\hat{\mathsf f}(\lambda)\rvert^2\lvert\lambda\rvert^{1+c}\,\mathrm d\lambda,\qquad k_c:=\min\Bigl\{\frac{2\pi^2}{B_cJ_c},1\Bigr\}\quad(c\ne1),\qquad k_1:=1.
\end{equation}
Both integrals are finite and positive for $c\in(0,2)$, since $1-\cos u\le\min\{u^2/2,2\}$ and $\hat{\mathsf f}$ decays faster than any power, so $k_c\in(0,1]$ depends only on $c$ and the fixed choice of $\mathsf f$.

\begin{lemma}[Distances from means]\label{lem:lb-sep}
Let $\mu$ and $\mu'$ be probability laws on $[0,\bar H]$ with means $g$ and $g'$, and let $\tilde\mu$ and $\tilde\mu'$ be their pushforwards under $t\mapsto te_1$. Then the following hold.
\begin{enumerate}[leftmargin=*,itemsep=1pt]
\item The Wasserstein distance satisfies
\begin{equation}
\lvert g-g'\rvert\le W_1(\mu,\mu')\le\bar H.
\end{equation}
\item The Cram\'er distance satisfies
\begin{equation}
\bar H^{-1/2}\lvert g-g'\rvert\le\ell_{\mathrm C}(\mu,\mu')\le\bar H^{1/2}.
\end{equation}
\item The MMD satisfies
\begin{equation}
(k_c\kappa_1)^{1/2}\bar H^{c/2-1}\lvert g-g'\rvert\le\mathrm{MMD}_\kappa(\tilde\mu,\tilde\mu')\le\kappa_1^{1/2}\bar H^{c/2},
\end{equation}
and $\mathrm{MMD}_\kappa(\tilde\mu,\tilde\mu')=\kappa_1^{1/2}\ell_{\mathrm C}(\mu,\mu')$ for $c=1$.
\end{enumerate}
\end{lemma}

\begin{proof}
To prove the first statement, note that the identity is $1$-Lipschitz, and both laws live on an interval of length $\bar H$.

For the second statement, note that
\begin{equation}
g-g'=\int_0^{\bar H}(F_{\mu'}-F_\mu)(t)\,\mathrm dt,\qquad\ell_{\mathrm C}(\mu,\mu')^2=\int_0^{\bar H}(F_\mu-F_{\mu'})^2(t)\,\mathrm dt\le\bar H,
\end{equation}
and the Cauchy--Schwarz inequality applied to the first identity gives the lower bound.

For the third statement, let $\sigma:=\mu-\mu'$ and $\tilde\sigma:=\tilde\mu-\tilde\mu'$, which are signed measures of total mass zero. For every kernel,
\begin{equation}
\iint d_\kappa\,\mathrm d\tilde\sigma\,\mathrm d\tilde\sigma=2\tilde\sigma(\mathbb R^q)\int\kappa(y,y)\,\mathrm d\tilde\sigma(y)-2\iint\kappa\,\mathrm d\tilde\sigma\,\mathrm d\tilde\sigma,
\end{equation}
and the first term on the right vanishes. By shift invariance, $c$-homogeneity, and symmetry of $d_\kappa$ (Appendix~\ref{app:sharp}),
\begin{equation}
d_\kappa(te_1,se_1)=d_\kappa((t-s)e_1,0)=\kappa_1\lvert t-s\rvert^c.
\end{equation}
Hence
\begin{equation}\label{eq:lb-mmd-line}
\mathrm{MMD}_\kappa(\tilde\mu,\tilde\mu')^2=\iint\kappa\,\mathrm d\tilde\sigma\,\mathrm d\tilde\sigma=-\frac12\iint d_\kappa\,\mathrm d\tilde\sigma\,\mathrm d\tilde\sigma=-\frac{\kappa_1}2\iint\lvert t-s\rvert^c\,\mathrm d\sigma(t)\,\mathrm d\sigma(s).
\end{equation}
With $X,X'\sim\mu$ and $Y,Y'\sim\mu'$ independent, the right-hand side equals
\begin{equation}
\frac{\kappa_1}{2}\Bigl(2\mathbb E\lvert X-Y\rvert^c-\mathbb E\lvert X-X'\rvert^c-\mathbb E\lvert Y-Y'\rvert^c\Bigr)\le\kappa_1\bar H^c,
\end{equation}
which is the upper bound. For $c=1$, the energy-distance identity gives
\begin{equation}
\iint\lvert t-s\rvert\,\mathrm d\sigma(t)\,\mathrm d\sigma(s)=-2\int(F_\mu-F_{\mu'})^2\,\mathrm dt,
\end{equation}
so $\mathrm{MMD}_\kappa(\tilde\mu,\tilde\mu')=\kappa_1^{1/2}\ell_{\mathrm C}(\mu,\mu')$, and the second statement gives the lower bound with $k_1=1$.

For general $c$, the substitution $u=\lambda y$ gives
\begin{equation}
\int_{\mathbb R}(1-\cos\lambda y)\lvert\lambda\rvert^{-1-c}\,\mathrm d\lambda=B_c\lvert y\rvert^c.
\end{equation}
Since
\begin{equation}
\iiint\lvert1-\cos\lambda(t-s)\rvert\lvert\lambda\rvert^{-1-c}\,\mathrm d\lambda\,\mathrm d\lvert\sigma\rvert(t)\,\mathrm d\lvert\sigma\rvert(s)=B_c\iint\lvert t-s\rvert^c\,\mathrm d\lvert\sigma\rvert(t)\,\mathrm d\lvert\sigma\rvert(s)<\infty,
\end{equation}
Fubini's theorem applies. Let $\hat\sigma(\lambda):=\int e^{-\mathrm i\lambda t}\,\mathrm d\sigma(t)$. Then
\begin{equation}
\iint\cos\lambda(t-s)\,\mathrm d\sigma(t)\,\mathrm d\sigma(s)=\lvert\hat\sigma(\lambda)\rvert^2,
\end{equation}
and, since $\sigma(\mathbb R)=0$,
\begin{equation}\label{eq:lb-fourier}
-\iint\lvert t-s\rvert^c\,\mathrm d\sigma(t)\,\mathrm d\sigma(s)=\frac1{B_c}\int_{\mathbb R}\lvert\hat\sigma(\lambda)\rvert^2\lvert\lambda\rvert^{-1-c}\,\mathrm d\lambda.
\end{equation}
Let $\mathsf f_{\bar H}(t):=\mathsf f(t/\bar H)$, so that
\begin{equation}
\hat{\mathsf f}_{\bar H}(\lambda)=\bar H\hat{\mathsf f}(\bar H\lambda),\qquad\int_{\mathbb R}\lvert\hat{\mathsf f}_{\bar H}(\lambda)\rvert^2\lvert\lambda\rvert^{1+c}\,\mathrm d\lambda=\bar H^{-c}J_c.
\end{equation}
By Fourier inversion,
\begin{equation}
\int\mathsf f_{\bar H}\,\mathrm d\sigma=\frac1{2\pi}\int_{\mathbb R}\hat{\mathsf f}_{\bar H}(\lambda)\overline{\hat\sigma(\lambda)}\,\mathrm d\lambda,
\end{equation}
so the Cauchy--Schwarz inequality with the weights $\lvert\lambda\rvert^{\pm(1+c)/2}$ gives
\begin{equation}\label{eq:lb-cs}
\Bigl\lvert\int\mathsf f_{\bar H}\,\mathrm d\sigma\Bigr\rvert^2\le\frac{\bar H^{-c}J_c}{4\pi^2}\int_{\mathbb R}\lvert\hat\sigma(\lambda)\rvert^2\lvert\lambda\rvert^{-1-c}\,\mathrm d\lambda.
\end{equation}
Since $\sigma$ lives on $[0,\bar H]$, where $\mathsf f_{\bar H}(t)=t/\bar H$, we have $\int\mathsf f_{\bar H}\,\mathrm d\sigma=(g-g')/\bar H$. Combining this with \eqref{eq:lb-mmd-line}, \eqref{eq:lb-fourier}, and \eqref{eq:lb-cs}, we obtain
\begin{equation}
\mathrm{MMD}_\kappa(\tilde\mu,\tilde\mu')^2\ge\frac{\kappa_1}{2B_c}\cdot\frac{4\pi^2\bar H^c}{J_c}\cdot\frac{(g-g')^2}{\bar H^2}\ge k_c\kappa_1\bar H^{c-2}(g-g')^2.
\end{equation}
\end{proof}

We now construct the two-state instances. For $p\in(0,1)$ and $r\in[0,1)$, let $\mathcal C(p,r)$ be the instance with states $\X=\{x,z\}$, a single action, rewards $R(x)=1$ and $R(z)=0$, and transitions
\begin{equation}
P(x\mid x)=p,\qquad P(z\mid x)=1-p,\qquad P(x\mid z)=r,\qquad P(z\mid z)=1-r.
\end{equation} For MTD the rewards are $R(x)=e_1$ and $R(z)=0$, so every return lies on the line $\mathbb Re_1$, and we identify a return law on that line with the law of its first coordinate. For $\Delta_p\in(0,1/(2\bar H)]$, let
\begin{equation}
p_0:=1-\frac1{\bar H},\qquad p_1:=p_0+\Delta_p,\qquad\mathcal M_i:=\mathcal C(p_i,r)\quad(i\in\{0,1\}),
\end{equation}
with $r:=0$ in the fixed-horizon setting and $r:=\underline\nu(1-p_0)/(1-\underline\nu)$ in the discounted setting. Let $\eta_i$ be the scalar return law at $x$ under $\mathcal M_i$, at horizon $H$ in the fixed-horizon setting. It lives on $[0,\bar H]$, and we write $\bar G_i$ for its mean.

\begin{lemma}[Mean separation]\label{lem:lb-mean}
The means of $\eta_0$ and $\eta_1$ satisfy
\begin{equation}
\bar G_1-\bar G_0\ge\frac{\bar H^2\Delta_p}{16}.
\end{equation}
\end{lemma}

\begin{proof}
In the discounted setting, let $g(p)$ and $v(p)$ be the values at $x$ and $z$ of $\mathcal C(p,r)$. The Bellman equations
\begin{equation}
v=\gamma rg+\gamma(1-r)v,\qquad g=1+\gamma pg+\gamma(1-p)v
\end{equation}
give
\begin{equation}
\frac1{g(p)}=1-\gamma p-\frac{\gamma^2r(1-p)}{1-\gamma+\gamma r},\qquad \frac{\mathrm d}{\mathrm dp}\frac1{g(p)}=-\frac{\gamma(1-\gamma)}{1-\gamma+\gamma r},\qquad\text{so}\qquad g'(p)=\frac{\gamma(1-\gamma)}{1-\gamma+\gamma r}g(p)^2.
\end{equation}
Since $\underline\nu\le1/8$, we have $r\le(1-\gamma)/7$, so $1-\gamma+\gamma r\le\frac43(1-\gamma)$. For $p\ge p_0=\gamma$, we have $1/g(p)\le1-\gamma p\le1-\gamma^2\le2(1-\gamma)$, so $g(p)\ge\bar H/2$, and
\begin{equation}
g'(p)\ge\frac34\gamma\cdot\frac{\bar H^2}4\ge\frac{\bar H^2}{16}
\end{equation}
because $\gamma\ge1/2$. In the fixed-horizon setting, we have $r=0$, and the chain started at $x$ stays there for $N\ge1$ steps, with $\Pr(N\ge k)=p^{k-1}$. The return at horizon $H$ is $\min\{N,H\}$, with mean $g(p):=\sum_{k=1}^Hp^{k-1}$. Then
\begin{equation}
g'(p)=\sum_{k=2}^H(k-1)p^{k-2}\ge p^H\frac{H(H-1)}2,
\end{equation}
and, for $p\ge p_0=1-1/H$, we have $p^H\ge(1-1/H)^H\ge1/4$, so $g'(p)\ge H(H-1)/8\ge H^2/16$ because $H\ge2$. In both settings, integrating $g'$ over $[p_0,p_1]$ gives the claim.
\end{proof}

\begin{lemma}[Coverage and information]\label{lem:lb-kl}
Let $\mathsf P_i$ be the law of $T$ sampled transitions under $\mathcal M_i$, and let $\mathrm{kl}:=\mathrm{KL}(\mathrm{Ber}(p_0)\Vert\mathrm{Ber}(p_1))$. Then $\mathrm{kl}\le3\bar H\Delta_p^2$, and the following hold.
\begin{enumerate}[leftmargin=*,itemsep=1pt]
\item Under i.i.d.\ sampling with $\rho(x)=\underline\nu$, we have $\rho_{\min}=\underline\nu$ and
\begin{equation}
\mathrm{KL}(\mathsf P_0\Vert\mathsf P_1)\le T\underline\nu\,\mathrm{kl}.
\end{equation}
\item Under Markovian sampling in the discounted setting, with $S_0$ drawn from the stationary law of $\mathcal M_0$ under both instances, both chains are irreducible and aperiodic, $\mu_{\min}\in[\underline\nu,2\underline\nu]$, $h_Q\le\bar H/\underline\nu$, and
\begin{equation}
\mathrm{KL}(\mathsf P_0\Vert\mathsf P_1)=T\underline\nu\,\mathrm{kl}.
\end{equation}
\item Under episodic sampling in the fixed-horizon setting with $\nu_0(x)=4\underline\nu$, we have $\rho_{H,\min}\in[\underline\nu,4\underline\nu]$ under both instances, and
\begin{equation}
\mathrm{KL}(\mathsf P_0\Vert\mathsf P_1)\le4T\underline\nu\,\mathrm{kl}.
\end{equation}
\end{enumerate}
\end{lemma}

\begin{proof}
Since $p_1\in[1/2,1-1/(2\bar H)]$ and $\bar H\ge2$, we have $p_1(1-p_1)\ge(1-1/(2\bar H))/(2\bar H)\ge3/(8\bar H)$, and bounding the Kullback--Leibler divergence by the $\chi^2$-divergence gives
\begin{equation}
\mathrm{kl}\le\frac{\Delta_p^2}{p_1(1-p_1)}\le3\bar H\Delta_p^2.
\end{equation} The instances differ only in the successor law at $x$, the rewards are deterministic, and the law of the first state and of the resets is the same under both. By the chain rule, $\mathrm{KL}(\mathsf P_0\Vert\mathsf P_1)$ is therefore $\mathrm{kl}$ times the expected number of transitions from $x$ under $\mathcal M_0$.

Under i.i.d.\ sampling, this number is $T\underline\nu$, and $\rho(z)=1-\underline\nu\ge\underline\nu$.

Under Markovian sampling, since $r>0$ and $p_i\in(0,1)$, both chains are irreducible and aperiodic. Balancing the flow between the two states gives $\mu_i(x)=r/(r+1-p_i)$, so $\mu_0(x)=\underline\nu$ by the choice of $r$, and the expected number is $T\underline\nu$. Since $(1-p_0)/2\le1-p_1\le1-p_0$, we have
\begin{equation}
\underline\nu\le\mu_1(x)\le\frac{2r}{2r+1-p_0}\le2\underline\nu\le\frac14,
\end{equation}
so $\mu_{\min}=\mu_i(x)\in[\underline\nu,2\underline\nu]$. Every chain on two states is reversible, so the time-reversed chain is the chain itself, and
\begin{equation}
h_Q=\max\bigl\{\mathbb E_z[\tau_x],\mathbb E_x[\tau_z]\bigr\}=\max\Bigl\{\frac1r,\frac1{1-p_i}\Bigr\}.
\end{equation}
Here, $1/r=\bar H(1-\underline\nu)/\underline\nu\ge7\bar H\ge1/(1-p_i)$, so $h_Q=1/r\le\bar H/\underline\nu$.

Under episodic sampling, a transition under $\mathcal M_i$ at phase $t$ starts from $x$ with probability $\rho_t(x)=4\underline\nu p_i^t$. Under $\mathcal M_0$, this is at most $4\underline\nu$, which bounds the expected number by $4T\underline\nu$. Moreover,
\begin{equation}
\bar\rho_H(x)=\frac{4\underline\nu}H\sum_{t<H}p_i^t\in[\underline\nu,4\underline\nu]
\end{equation}
because $p_i^t\ge p_0^H\ge1/4$, and $\bar\rho_H(z)\ge1-4\underline\nu\ge1/2$, so $\rho_{H,\min}=\bar\rho_H(x)$.
\end{proof}

\begin{proof}[Proof of Theorem~\ref{thm:lower}]
Suppose first that \eqref{eq:lb-T} holds. The instances $\mathcal M_0$ and $\mathcal M_1$ belong to the class of the theorem by Lemma~\ref{lem:lb-kl}, since $h_Q\le\bar H/\underline\nu\le \hat h$ under Markovian sampling. Set $\Delta_p:=48\varepsilon/(\bar HD)$, which is at most $1/(2\bar H)$ by the condition on $\varepsilon$. Lemmas~\ref{lem:lb-sep} and~\ref{lem:lb-mean} give, for each of the three distances $d$,
\begin{equation}
d(\eta_0,\eta_1)\ge\frac{\bar HD\Delta_p}{16}=3\varepsilon.
\end{equation}
Let $\mathcal U:=\{d(\hat\eta,\eta_0)>\varepsilon\}$. On the complement of $\mathcal U$, the triangle inequality gives $d(\hat\eta,\eta_1)\ge d(\eta_0,\eta_1)-d(\hat\eta,\eta_0)>\varepsilon$. In all three sampling models, Lemma~\ref{lem:lb-kl} gives
\begin{equation}
\mathrm{KL}(\mathsf P_0\Vert\mathsf P_1)\le12T\underline\nu\bar H\Delta_p^2=\frac{27648\,T\underline\nu\varepsilon^2}{\bar HD^2},
\end{equation}
and the Bretagnolle--Huber inequality \citep{tsybakov2009introduction} gives
\begin{equation}
\mathsf P_0(d(\hat\eta,\eta_0)>\varepsilon)+\mathsf P_1(d(\hat\eta,\eta_1)>\varepsilon)\ge\mathsf P_0(\mathcal U)+\mathsf P_1(\mathcal U^{\mathsf c})\ge\frac12e^{-\mathrm{KL}(\mathsf P_0\Vert\mathsf P_1)}\ge\frac12\exp\Bigl(-\frac{27648\,T\underline\nu\varepsilon^2}{\bar HD^2}\Bigr),
\end{equation}
which exceeds $2\delta$ under \eqref{eq:lb-T}.

Suppose now that $T\le \hat h/2$. We adapt the coverage obstruction of \citet[Proposition~9.1]{peng2026sharp} to our distances. Let $\mathcal N_i:=\mathcal C(1-1/\hat h,1/\hat h)$, with the reward at $x$ replaced by $0$ in $\mathcal N_0$, and let $S_0:=z$. Both chains have $\mu=(1/2,1/2)$ and $h_Q=\max\{1/r,1/(1-p)\}=\hat h$, so both belong to the class of the theorem, and $\hat h\ge\bar H$. Under $\mathcal N_0$, every return is $0$. Under $\mathcal N_1$, the return at $x$ is at least the discounted reward collected before the first move to $z$, whose mean is
\begin{equation}
\frac1{1-\gamma(1-1/\hat h)}\ge\frac{\bar H}2,
\end{equation}
since $1-\gamma+\gamma/\hat h\le2(1-\gamma)$ for $\hat h\ge\bar H$. Lemma~\ref{lem:lb-sep} gives $d(\eta_0,\eta_1)\ge D/2>2\varepsilon$. Let $\mathcal A:=\{S_k=z\text{ for all }k<T\}$. Then, for $T\le\hat h/2$,
\begin{equation}
\mathsf P_0(\mathcal A)=\mathsf P_1(\mathcal A)=\Bigl(1-\frac1{\hat h}\Bigr)^{T-1}\ge1-\frac{T-1}{\hat h}\ge\frac12,
\end{equation}
and the observed transitions have the same law on $\mathcal A$ under both instances, because the chains agree at $z$. Let $\mathsf Q$ be this common conditional law. Since $d(\eta_0,\eta_1)>2\varepsilon$, the events $\{d(\hat\eta,\eta_0)>\varepsilon\}$ and $\{d(\hat\eta,\eta_1)>\varepsilon\}$ cover the sample space, so
\begin{equation}
\mathsf P_0(d(\hat\eta,\eta_0)>\varepsilon)+\mathsf P_1(d(\hat\eta,\eta_1)>\varepsilon)\ge\mathsf P_0(\mathcal A)\bigl[\mathsf Q(d(\hat\eta,\eta_0)>\varepsilon)+\mathsf Q(d(\hat\eta,\eta_1)>\varepsilon)\bigr]\ge\frac12>2\delta.
\end{equation}
In both cases, one of the two probabilities exceeds $\delta$.
\end{proof}

We now compare Theorem~\ref{thm:lower} with the upper bounds. The error at one state and one horizon is at most the supremum over states and horizons, so Theorem~\ref{thm:lower} applies to $\ell_{\mathrm C,\infty}$, $\bar W_1$, $\ell_{\mathrm M,\infty}$, and their fixed-horizon counterparts. For CTD, the grid conditions of Corollaries~\ref{cor:sharp-ctd} and~\ref{cor:undisc} give $\bar\tau,\bar\tau_H\le4$. For MTD, $1-\gamma^c\ge c(1-\gamma)$ for $c\le1$ by concavity, and $1-\gamma^c\ge1-\gamma\ge(1-\gamma)^c$ for $c\ge1$. With the bound on $\bar\tau_{\mathrm M}$ after Corollary~\ref{cor:sharp-mtd}, this gives
\begin{equation}
\bar\tau_{\mathrm M}(1-\gamma)^{-1-c}\varepsilon^{-2}\le\bar\kappa(1-\gamma)^{-1-c}\varepsilon^{-2}+\frac{(1-\gamma)^{-2}}{2\min\{c,1\}},
\end{equation}
and the condition on $h_{\Theta,H}$ in Corollary~\ref{cor:undisc-mtd} gives
\begin{equation}
\bar\tau_{H,\mathrm M}H^{1+c}\varepsilon^{-2}\le\bar\kappa H^{1+c}\varepsilon^{-2}+\frac{H^2}4.
\end{equation}
The second terms are $\varepsilon$-free and of the order of terms already present in \eqref{eq:mtd-rates} and \eqref{eq:fh-mtd-rates}. Since $\kappa_1\le\bar\kappa$, the remaining factor $\bar\kappa/(k_c\kappa_1)$ depends only on $c$, $q$, and the kernel. For the kernel induced by $\lVert x-y\rVert_2^c$, normalized so that $d_\kappa(x,y)=\lVert x-y\rVert_2^c$, we have $\kappa_1=1$ and $\bar\kappa=q^{c/2}$.

Finally, we discuss the scope of the conditions of Theorem~\ref{thm:lower}. By Lemma~\ref{lem:lb-sep}, no two laws on $[0,\bar H]$ are farther apart than $\bar H$ in $W_1$, $\bar H^{1/2}$ in Cram\'er, and $\kappa_1^{1/2}\bar H^{c/2}=k_c^{-1/2}D$ in MMD along $\mathbb Re_1$. The condition $\varepsilon\le D/96$ therefore excludes only accuracies within a factor depending on $c$ of the trivial one. The conditions $\gamma\ge1/2$ and $H\ge2$ cannot be removed entirely. For $\gamma=0$ or $H=1$, the return law at a state is the point mass at its reward, which one transition from that state reveals, so no bound of order $\varepsilon^{-2}$ holds. What the theorem leaves open are the lower-order terms. Since a smaller $\varepsilon$ only makes estimation harder, condition \eqref{eq:lb-T} at $\varepsilon=D/96$ gives the $\varepsilon$-free bound $\Omega(\underline\nu^{-1}\bar H\log(1/\delta))$, below the $\varepsilon$-free terms $\rho_{\min}^{-1}(1-\gamma)^{-2}$ and $\rho_{H,\min}^{-1}H^3$ of the upper bounds, and the terms of order $\varepsilon^{-1}$ are not matched either.

\section{Experimental details}\label{app:more-exp}
All experiments are tabular simulations written in Python with NumPy and run on a single CPU, without GPUs. Each panel of Figure~\ref{fig:sweeps} isolates one prediction of Theorem~\ref{thm:sharp}. We run constant-step CTD on a five-state ring that stays put with probability $p$ and otherwise moves to the next state, with two equiprobable actions and state-dependent rewards in $[0,1]$. We measure the error $e_{\mathrm C,k}$ to the projected fixed point, which excludes the projection bias, so the grids need not depend on $\varepsilon$. They are uniform with $d-1=\lceil(1-\gamma)^{-2}\rceil$ atoms, so that $\bar\tau\le10$, and by default $\gamma=0.9$, $p=0.5$, and queries are uniform. We estimate the noise floor by the mean of $e_{\mathrm C,k}^2$ over the second half of long runs. The transient length, the first time the mean error falls below twice that floor, measures how long the initial error takes to be forgotten. (a) The floor should be linear in $\alpha$, and it is for $\alpha\in[0.002,0.2]$, with Markovian and i.i.d.\ floors within $12\%$. (b) The floor should scale with the total variance rather than the worst-case bound, which is what improves the rate from $(1-\gamma)^{-3}$ to $(1-\gamma)^{-2}$. With $\alpha=0.1(1-\gamma)$, $\mathbb E[e_{\mathrm C}^2]/\alpha$ stays between $0.20$ and $0.23$ for $\gamma\in[0.8,0.975]$, below the exact $\mathrm{TV}_{\mathrm C,\max}$, while the worst-case bound $(1-\gamma)^{-2}$ grows $64$-fold. (c) Trajectory dependence should not enlarge the floor. As $p$ moves $h_Q$ from $4$ to $200$, with $p=0$ periodic, the Markovian floor exceeds the i.i.d.\ one, for the same transition matrix and $\rho=\mu_{\X}$, by at most $5\%$. (d) Coverage should enter only through the time scale $1/(\alpha\nu_{\min}(1-\beta))$. With one state queried with probability $\rho_{\min}\in[0.01,0.2]$, the floor changes by at most $12\%$, while the transient length grows as $\rho_{\min}^{-1}$.

The bottom row of Figure~\ref{fig:sweeps} repeats these experiments for MTD with bivariate rewards, the energy-distance kernel ($c=1$), and a $9\times9$ grid, with the same conclusions.

Figure~\ref{fig:fh} tests Theorem~\ref{thm:undisc} for fixed-horizon CTD and MTD with episodic resets, with the same conclusions. The transients end near rescaled time $\alpha\rho_{H,\min}k/H=1$, as the binomial contraction predicts.

In Figure~\ref{fig:sweeps}, the noise floors average the squared error over the second half of runs of at least $8/(\alpha\nu_{\min}(1-\gamma))$ updates and over $128$ runs ($48$ in (b), $512$ in (c)). In (a), the floor stays linear in $\alpha$ even far above the step sizes of Theorem~\ref{thm:sharp}. In (b), the exact $\mathrm{TV}_{\mathrm C,\max}$ grows from $0.43$ to $1.01$ over $\gamma\in[0.8,0.975]$. In (c), $p\in\{0,0.5,\dots,0.98\}$. For MTD, the rewards are bivariate in $[0,1]^2$ and the grid is a $9\times9$ grid on $[0,(1-\gamma)^{-1}]^2$. In (f), $\mathbb E[e_{\mathrm M}^2]/\alpha$ stays between $0.28$ and $0.54$, below $\mathrm{TV}_{\mathrm M,\max}$ ($1.10$ to $1.32$), while the support-diameter part $\bar\kappa(1-\gamma)^{-c}/(1-\gamma^c)$ of the worst-case bound, with $\bar\kappa=2\sqrt2$, grows $64$-fold. The Markovian floor is within $12\%$ of the i.i.d.\ one in (e) and exceeds it by at most $4\%$ as $h_Q$ grows in (g), and in (h) the floor changes by at most $4\%$ while the transient length grows as $\rho_{\min}^{-1}$.

Figure~\ref{fig:convergence} shows the constant-step trajectories behind Figure~\ref{fig:sweeps}, on the same ring and grids, for CTD (top) and MTD (bottom). Time is rescaled by $\alpha\nu_{\min}(1-\beta_{\mathrm C})$, the contraction rate of the averaged operator in Lemma~\ref{lem:propagation}, so Theorem~\ref{thm:sharp} predicts that the transients collapse and that each run then fluctuates in a neighborhood of squared radius $O(\alpha\,\mathrm{TV}_{\mathrm C,\max}\ell)$ under i.i.d.\ sampling and $O(\alpha\mathcal E_{\mathrm C}\ell^2)$ under Markovian sampling, up to lower-order terms. Solid lines are means over $128$ runs, bands are interquartile ranges, and dotted lines mark the noise floors.

Panels (a) and (b) use $\gamma=0.9$, $p=0.5$, uniform queries or the chain itself, and $\alpha\in\{0.01,0.03,0.1\}$. The transients coincide in rescaled time in both regimes, and the noise floors scale linearly in $\alpha$ and agree across regimes, as in Figure~\ref{fig:sweeps}(a). Panel (c) runs the Markovian recursion with $\alpha=0.03$ on chains with $p\in\{0,0.5,0.9,0.98\}$. The noise floors differ only because the total variance does, while the stickiest chains take longer to reach them, which is the effect of the additive transient $h_Q$ in Corollary~\ref{cor:sharp-ctd}. Panel (d) uses i.i.d.\ queries with one state of probability $\rho_{\min}$ and $\alpha=0.02$. Once the rare state dominates, the curves collapse in time rescaled by $\rho_{\min}$, and all runs reach the same noise floor. The MTD trajectories behave in the same way in all four panels.

\begin{figure}[ht]
\centering
\includegraphics[width=\linewidth]{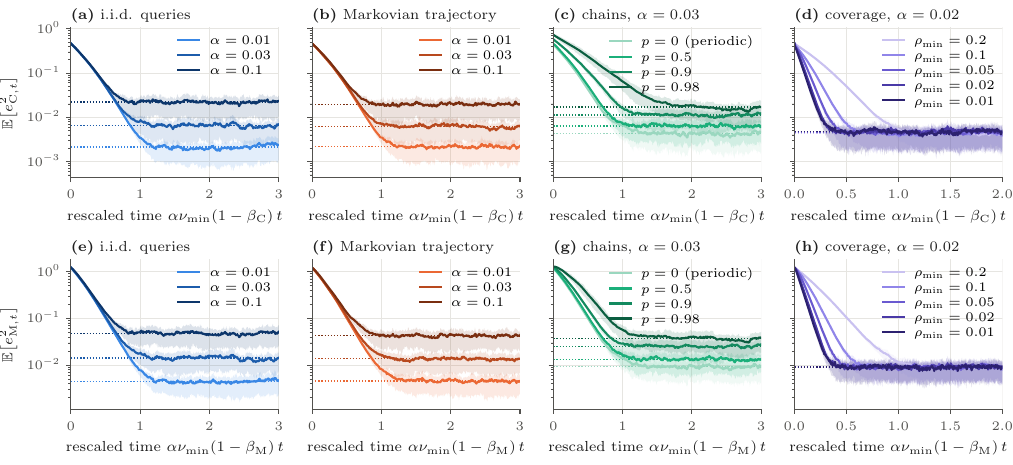}
\caption{Constant-step CTD (top) and MTD (bottom) trajectories: mean squared supremum error against time rescaled by $\alpha\nu_{\min}(1-\beta_\star)$, with interquartile bands over $128$ runs and dotted noise floors. (a, e) i.i.d.\ queries and (b, f) a Markovian trajectory for three step sizes. (c, g) Markovian trajectories on chains of increasing stickiness, including a periodic chain. (d, h) i.i.d.\ queries with one rarely queried state.}
\label{fig:convergence}
\end{figure}

\begin{figure}[ht]
\centering
\includegraphics[width=\linewidth]{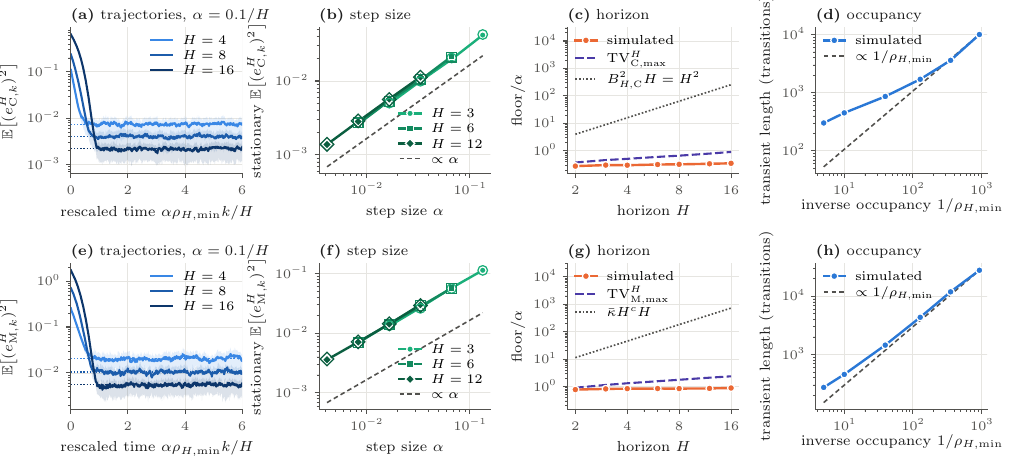}
\caption{Constant-step fixed-horizon CTD (top) and MTD (bottom), supremum error over horizon-state pairs. (a, e) Episodic trajectories against rescaled time, means with interquartile bands over $128$ runs ($64$ for MTD) and dotted noise floors. (b, f) Noise floor against $\alpha$, episodic (filled) and i.i.d.\ from $\bar\rho_H$ (hollow). (c, g) Floor divided by $\alpha$ against $H$, with the exact total variance and the worst-case bound. (d, h) Transient length, the first time the mean error falls below twice its noise floor, against the inverse occupancy.}
\label{fig:fh}
\end{figure}

Figure~\ref{fig:fh} tests Theorem~\ref{thm:undisc} and Corollary~\ref{cor:undisc} with constant-step fixed-horizon CTD \eqref{eq:undisc-recur} on the same ring with $p=0.5$. Layer $h$ has $d-1=\max\{H^2,16\}$ uniform atoms on $[0,h]$, so that $\bar\tau_H\le4$, and the reset distribution puts mass $0.6$ on one state and $0.1$ on each of the others. The error $e^H_{\mathrm C,k}$ is the supremum Cram\'er error over horizon-state pairs, and noise floors average $(e^H_{\mathrm C,k})^2$ over the second half of runs of $6H/(\alpha\rho_{H,\min})$ transitions and over $128$ runs ($96$ in (b)). Panel (a) uses episodic sampling with $\alpha=0.1/H$ for $H\in\{4,8,16\}$ and rescales time by $\alpha\rho_{H,\min}/H$. The mean error first comes within $20\%$ of its floor at rescaled times $0.49$, $0.68$, and $0.89$, approaching $1$ as $H$ grows, as the binomial contraction of Lemma~\ref{lem:fh-propagation} predicts: an error is removed once it has been passed down $H$ layers, each passage occurring at rate at least $\alpha\rho_{H,\min}$, and $\Pr(\mathrm{Bin}(n,\alpha\rho_{H,\min})\le H-1)$ drops sharply at $n\alpha\rho_{H,\min}=H$ when $H$ is large. In (b), the floor is linear in $\alpha$ for $H\in\{3,6,12\}$, with floor$/\alpha$ between $0.29$ and $0.34$, and i.i.d.\ sampling from $\rho=\bar\rho_H$ (hollow markers) gives the same floor as episodic sampling within $5\%$, as the common leading term of Corollary~\ref{cor:undisc} predicts. In (c), with $\alpha=0.1/H$ and $H\in\{2,\dots,16\}$, floor$/\alpha$ grows from $0.28$ to $0.35$, below the exact $\mathrm{TV}^H_{\mathrm C,\max}$, which grows from $0.38$ to $0.91$, while the worst-case bound $B_{H,\mathrm C}^2H=H^2$ grows $64$-fold. In (d), the chain is stickier ($p=0.8$), $H=5$, $\alpha=0.05$, and the reset puts mass $q\in\{0.2,0.6,0.9,0.97,0.99,0.997\}$ on one state, which moves $\rho_{H,\min}$ from $0.2$ to $0.0011$. The floor stays between $0.0191$ and $0.0201$, while the transient length grows with $\rho_{H,\min}^{-1}$ and becomes proportional to it once the occupancy is small, with a log-log slope of $1.05$ between the two rarest cases, as the factor $\nu_{\min}^{-1}$ in $L_H$ predicts.

For fixed-horizon MTD, the bottom row of Figure~\ref{fig:fh} uses the same chains and schedules with bivariate rewards $R(s,a)\in[0,1]^2$, whose second coordinate is the first coordinate of a shifted state with the actions swapped, the energy-distance kernel ($c=1$), a $7\times7$ grid on $[0,h]^2$ at layer $h$, and $64$ runs. In (e), the mean error first comes within $20\%$ of its floor at rescaled times $0.64$, $0.74$, and $0.87$ for $H\in\{4,8,16\}$. In (f), floor$/\alpha$ lies between $0.79$ and $0.90$, and i.i.d.\ and episodic floors agree within $4\%$. In (g), floor$/\alpha$ grows from $0.78$ to $0.89$ for $H\in\{2,\dots,16\}$, below $\mathrm{TV}^H_{\mathrm M,\max}$, which grows from $0.91$ to $2.37$, while the support-diameter worst-case bound $\bar\kappa H^cH=2\sqrt2H^2$, the analogue of $B_{H,\mathrm C}^2H=H^2$, grows $64$-fold. In (h), the floor stays between $0.050$ and $0.053$, and the log-log slope of the transient length between the two rarest cases is $0.88$.

\section{Related work}\label{app:related}
Our arguments build on existing contractive SA theory. \citet{chen2020finite,chen2022finite,chen2024lyapunov} develop smooth-envelope analyses for arbitrary contraction norms under i.i.d.\ and Markovian noise, extending the broader SA and TD literature \citep{robbins1951stochastic,ljung2003analysis,borkar2008stochastic,kushner2003stochastic,bhandari2018finite,qu2020finite}. For the episodic results, we use a Poisson-equation decomposition as in \citet{haque2024stochastic}, which avoids the explicit mixing window appearing in earlier Markovian SA bounds. For the discounted Markovian results, we build on the reverse-time analysis of \citet{peng2026sharp}, who proves the rate $\tilde O(\mu_{\min}^{-1}(1-\gamma)^{-3}\varepsilon^{-2}+\mu_{\min}^{-1}t_{\mathrm{mix}})$, with a matching lower bound, for scalar TD on one trajectory, and extend it to operator-valued CTD and MTD (Appendix~\ref{app:sharp-rev}).

For the distributional RL machinery, \citet{rowland2018analysis,bellemare2023distributional} provide the Cram\'er geometry, categorical projection, and asymptotic CTD theory. \citet{wiltzer2024foundations} develop the signed-measure MMD geometry and almost-sure asynchronous MTD convergence. \citet{de2020fixed} introduce fixed-horizon TD, which keeps a separate value function for each horizon and updates all of them at every step. Our fixed-horizon CTD and MTD are its undiscounted categorical counterparts, which, to our knowledge, had no finite-iteration analysis. 

Other finite-iteration RL works answer neighboring questions under different sampling access or estimators. The analysis of \citet{peng2024statistical} gives high-probability rates for synchronous distributional TD and CTD with a generative model. The expanded analysis in \citep{peng2025statisticalefficiencydistributionaltemporal} treats distributional TD as infinite-dimensional SA with a Hilbert-space Freedman inequality, and covers independent one-state sampling and Markovian data with data dropping or variance reduction. \citet{boeck2022speedy} analyze an accelerated categorical method by extending the speedy Q-learning analysis, while \citet{zhang2023estimation,rowland2024near} study model-based and direct estimators. \citet{wu2023distributional} consider offline evaluation, \citet{peng2025finite} treat linear approximation, \citet{kastner2025categorical} study KL-based categorical dynamics asymptotically, and \citet{cheng2026finite} give finite-sample rates for synchronous quantile TD. Our i.i.d.\ analysis uses the same Freedman inequality, but we propagate statewise second moments and bound the Cram\'er total variance, which gives the sharper Cram\'er rate. Under Markovian data, an energy bound in terms of the same total variance gives the same exponent for standard CTD and MTD, which update on every transition, without data dropping or variance reduction (Section~\ref{sec:sharp-main}). The same analysis also covers fixed-horizon CTD and MTD (Section~\ref{sec:undisc}).

\end{document}